\documentclass[11pt]{article}

\usepackage[final]{acl}

\usepackage{times}
\usepackage{latexsym}
\usepackage[T1]{fontenc}
\usepackage[utf8]{inputenc}
\usepackage{microtype}
\usepackage{inconsolata}
\usepackage{graphicx}
\usepackage{amsmath}
\usepackage{amsthm}
\usepackage{amssymb}
\usepackage{booktabs}
\usepackage{url}
\usepackage{hyperref}
\usepackage{subcaption}
\usepackage{multirow}
\usepackage{algorithm}
\usepackage{algorithmic}
\usepackage{float}
\usepackage{tcolorbox}
\usepackage{colortbl}
\usepackage{xcolor}
\usepackage{enumitem}
\usepackage[section]{placeins}
\usepackage{balance}
\usepackage[capitalize,noabbrev]{cleveref}

\definecolor{darkblue}{rgb}{0, 0, 0.5}
\hypersetup{
  colorlinks=true,
  linkcolor=darkblue,
  citecolor=darkblue,
  urlcolor=darkblue,
}

\theoremstyle{remark}

\newtheorem{proposition}{Observation}
\theoremstyle{definition}
\newtheorem{assumption}{Assumption}

\begin{document}
\title{ST-Lite: Training-Free KV Cache Compression with \\ Spatio-Trajectory Guidance for Long-Horizon GUI Agents}
\setlength\titlebox{12\baselineskip}

\author{
  \textbf{Bowen Zhou}$^{1}$\thanks{Equal contribution.} \quad
  \textbf{Zhou Xu}$^{1}$\footnotemark[1] \quad
  \textbf{Wanli Li}$^{2}$\footnotemark[1] \quad
  \textbf{Jingyu Xiao}$^{3}$ \\
  \textbf{Pingan Gan}$^{1}$ \quad \textbf{Haoqian Wang}$^{1}$\thanks{Corresponding author.} \\
  $^{1}$Tsinghua University, Shenzhen, China \\
  $^{2}$Zhejiang University, Hangzhou, China \\
  $^{3}$The Chinese University of Hong Kong, Hong Kong, China \\
  \texttt{zhoubw25@mails.tsinghua.edu.cn} \\
  \texttt{xu-z25@mails.tsinghua.edu.cn} \\
  \texttt{wanli\_li@zju.edu.cn}
}

\maketitle

\begin{abstract}
Training-free KV cache compression is essential for deploying vision-language GUI agents under memory and latency constraints, yet existing methods are designed for generic language workloads and ignore the distinctive structure of GUI interaction traces. We characterize three GUI-specific workload properties---high inter-frame visual redundancy, extremely small UI-element spatial footprints, and near-uniform cross-layer attention sparsity---that cause existing schemes to retain as few as $39\%$ of oracle-important KV pairs at the standard $20\%$ budget. To address this, we propose \textbf{ST-Lite}, a training-free compression scheme whose three components each target one property: \textbf{Trajectory-aware Semantic Gating (TSG)} filters redundant historical frames, \textbf{Component-centric Spatial Saliency (CSS)} preserves fine-grained element boundaries, and a flat per-layer budget avoids hierarchical misallocation. Across seven GUI benchmarks and two backbones in the deployment-relevant $\beta \in [10\%, 40\%]$ window, ST-Lite consistently outperforms all existing compression baselines---matching or exceeding Full Cache task accuracy on the primary backbone at the $20\%$ budget---while delivering up to $2.35\times$ decoding speedup at fivefold compression. The implementation is available at \url{https://github.com/94wen94/ST-Lite}.
\end{abstract}

\section{Introduction}
\label{sec:intro}


Vision-language models have catalyzed autonomous Graphical User Interface (GUI) agents that navigate digital environments and execute multi-step workflows~\cite{chae2024web,christianos2023pangu,patel2024large,putta2024agent,qi2024webrl,wen2024autodroid,zhang2025appagent}. For example, WeaveBench evaluates long-horizon hybrid-interface workflows that require GUI control and CLI/code operations within a single trajectory~\citep{li2026weavebench}. A central obstacle to their real-time deployment is the Key-Value (KV) cache. High-resolution screenshots combined with long action histories cause the cache to grow linearly with sequence length, routinely reaching multi-gigabyte scale within fifty steps and dominating both GPU memory and decoding latency on memory-constrained accelerators.

\begin{figure*}[t]
    \centering 
    \includegraphics[width=1.0\linewidth]{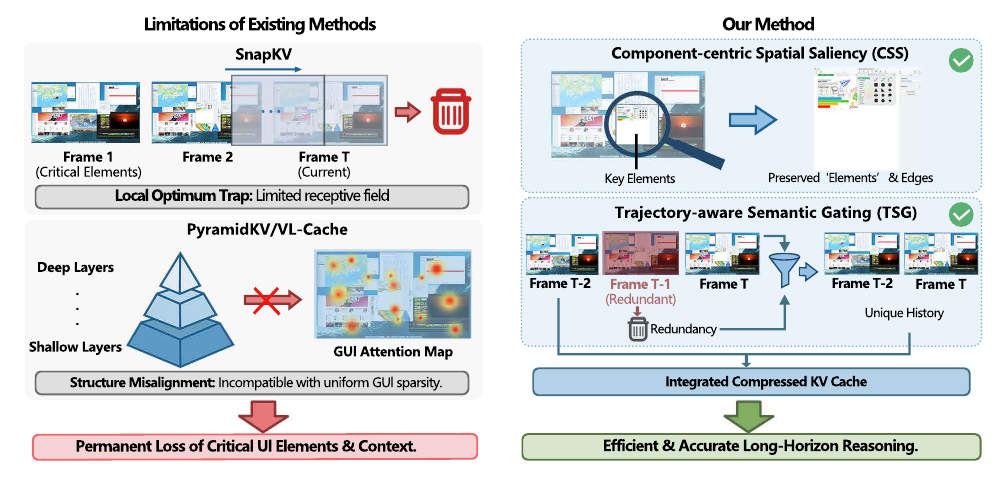}
    
    \caption{Conceptual illustration of ST-Lite compared with existing methods. The left side depicts the limitations of current strategies: window-based greedy methods (e.g., SnapKV) suffer from \textbf{local optimum traps}, while hierarchical allocation methods (e.g., PyramidKV) lead to \textbf{structure misalignment}. The right side demonstrates our ST-Lite, which integrates \textbf{Component-centric Spatial Saliency (CSS)} to preserve key interactive elements and \textbf{Trajectory-aware Semantic Gating (TSG)} to eliminate history redundancy, achieving efficient and accurate long-horizon reasoning.}
    
    \label{fig:intro_concept}
\end{figure*}

Existing training-free KV cache compression methods inherit assumptions from text and general multimodal settings, and these assumptions do not match the GUI agent regime. We make this concrete by adopting the Cache Hit Rate (CHR) of VL-Cache~\cite{tu2024vl} (Definition~3.1), which quantifies the fraction of oracle-important KV pairs that a compression scheme retains (\S\ref{app:chr}). On AITW~\cite{rawles2023androidinthewild} under UI-TARS-1.5-7B~\cite{qin2025ui} at the canonical $20\%$ budget, SnapKV~\cite{li2024snapkv}, PyramidKV~\cite{cai2024pyramidkv} and VL-Cache~\cite{tu2024vl} attain CHR of only $53.1\%$, $40.4\%$ and $39.0\%$ respectively (\autoref{tab:cache_hit_rate_app}, \autoref{fig:cache_hit_rate_app}). Two failure modes account for the gap (\autoref{fig:intro_concept}): a local optimum trap, in which window-based scoring favors recent attention and discards long-history dependencies required for the next action; and a structure misalignment, in which hierarchical allocation concentrates the cache budget on shallow layers while GUI workloads exhibit near-uniform sparsity across all layers, so deep-layer decision-relevant tokens are systematically lost.

We therefore characterize the statistical properties of GUI agent KV caches directly. Three properties jointly distinguish this regime from the settings that motivated prior compression methods. First, GUI trajectories exhibit substantial temporal redundancy: on $311$ consecutive frame pairs from AgentNetBench~\cite{wang2025opencua}, the median pairwise pixel overlap is $82.7\%$, and more than three quarters of pairs exceed $50\%$ overlap (\autoref{fig:time_redundancy}, \S\ref{app:frame_redundancy}). Second, decision-relevant content is spatially localized: on ScreenSpot Pro~\cite{li2025screenspot}, the median UI element occupies only $1.7\% \times 1.6\%$ of the screen (\S\ref{app:ui_element_size}). Third, attention sparsity is nearly uniform across layers: the coefficient of variation across the $28$ transformer layers is at most $1.4\%$ on both AITW and AgentNetBench (measured on OpenCUA-7B~\cite{wang2025opencua}, \autoref{tab:opencua_sparsity_app}; visually consistent on UI-TARS-1.5-7B, \autoref{fig:layer_sparsity}, \S\ref{app:sparsity_def}), in clear contrast to the pyramid-shaped (top-heavy) profiles that motivate hierarchical allocation in text and general vision-language settings.

Guided by these three properties, we propose \textbf{ST-Lite} (Spatio-Trajectory Lite), a training-free KV cache compression scheme whose three design choices are each motivated by one of the three measured properties and directly target the two failure modes above (\autoref{fig:intro_concept}). Trajectory-aware Semantic Gating (TSG) exploits temporal redundancy by filtering historical KV tokens already covered by the current frame, addressing the local optimum trap. Component-centric Spatial Saliency (CSS) exploits spatial localization by selecting tokens that form structural boundaries of small UI elements within a $3 \times 3$ Moore neighborhood. A uniform per-layer budget replaces hierarchical allocation and removes the structure-misalignment source.

We evaluate ST-Lite on seven GUI benchmarks on UI-TARS-1.5-7B and OpenCUA-7B. At the $20\%$ budget on ScreenSpot Pro~\cite{li2025screenspot}, ST-Lite reaches $42.2\%$ step accuracy and matches the uncompressed Full Cache baseline ($42.3\%$). On long-horizon AITW~\cite{rawles2023androidinthewild}, it improves Full Cache from $18.2\%$ to $20.1\%$ at the same budget (\autoref{tab:key_benchmarks_detail}), with up to $2.35\times$ decoding speedup at fivefold compression (\S\ref{sec:efficiency}). Under the CHR metric, ST-Lite consistently leads all baselines across budgets (\autoref{tab:cache_hit_rate_app}).

Our contributions are threefold. First, we identify two failure modes (recency-bias eviction and hierarchical misallocation) grounded in three measured workload properties (P1--P3) that jointly distinguish GUI agent caches from prior settings. Second, we show that the minimal operators satisfying the constraints imposed by P1--P3---TSG, CSS, and flat allocation---tightly recover performance; wider kernels, alternative metrics, and hierarchical budgets all degrade (\S\ref{app:neighborhood_size}, \S\ref{app:sim_metric}, \autoref{tab:budget_allocation_ablation}). Third, on seven GUI benchmarks ST-Lite delivers up to $2.35\times$ decoding speedup at fivefold compression while consistently outperforming all existing compression baselines---matching or exceeding Full Cache on the primary backbone and leading under the VL-Cache CHR metric at every tested budget.

\section{Background}
\label{sec:preliminary}

\begin{figure}[t]
    \centering
    \begin{subfigure}[b]{0.48\linewidth}
        \centering
        \includegraphics[width=\linewidth]{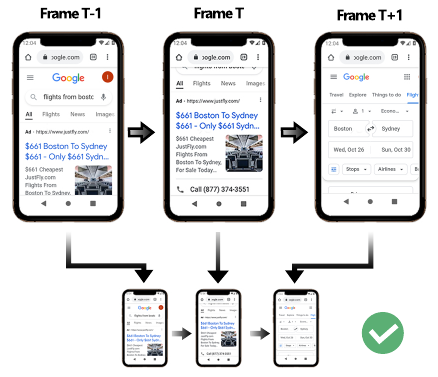}
        \caption{Original Sequence $\mathcal{O}$}
        \label{fig:seq_original}
    \end{subfigure}
    \hfill
    \begin{subfigure}[b]{0.48\linewidth}
        \centering
        \includegraphics[width=\linewidth]{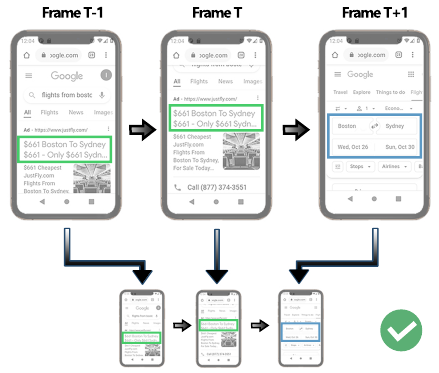}
        \caption{Pruned Sequence $\mathcal{O}_r$}
        \label{fig:seq_pruned}
    \end{subfigure}
    \caption{Task execution flows under different context settings. (a) The original sequence $\mathcal{O}$ with full historical frames leads to task success. (b) The pruned sequence $\mathcal{O}_r$, with trajectory-wise redundant frames removed, also succeeds, confirming the high historical redundancy in GUI trajectories.}
    \label{fig:time_redundancy}
\end{figure}

\begin{figure}[t]
    \centering
    \begin{subfigure}[b]{0.48\linewidth}
        \centering
        \includegraphics[width=\linewidth]{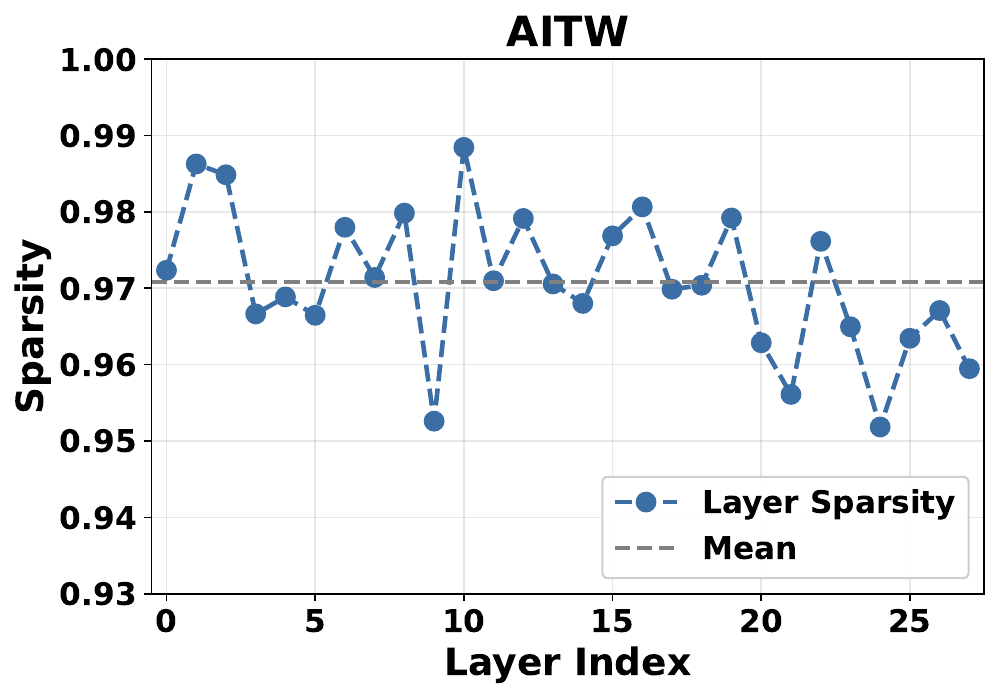}
        \caption{}
        \label{fig:sparsity_aitw}
    \end{subfigure}
    \hfill
    \begin{subfigure}[b]{0.48\linewidth}
        \centering
        \includegraphics[width=\linewidth]{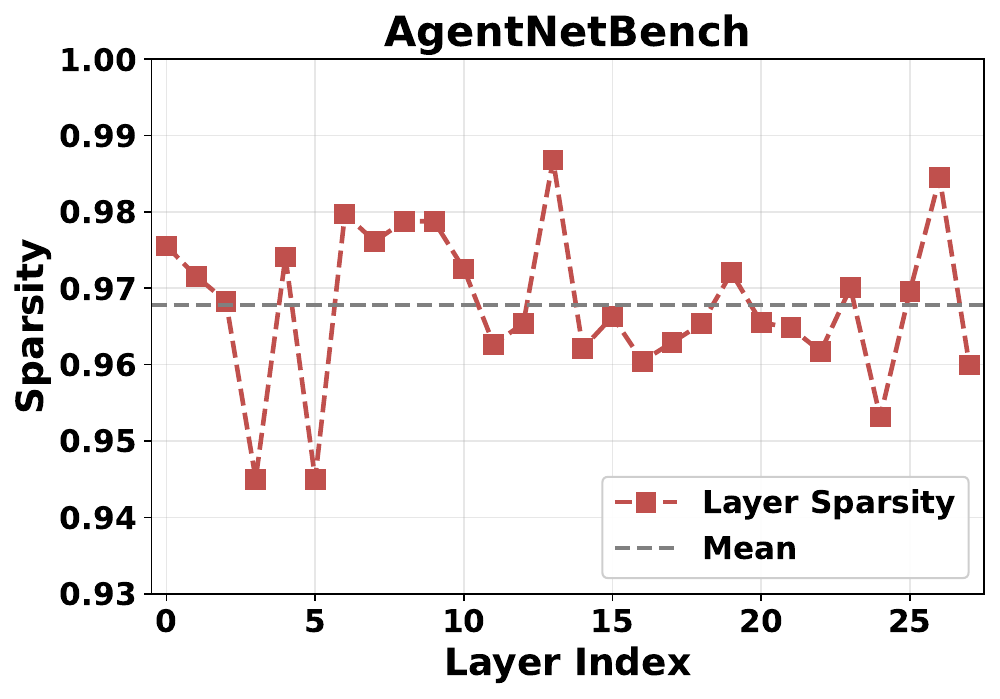}
        \caption{}
        \label{fig:sparsity_agentnet}
    \end{subfigure}
    \caption{\textbf{Layer-wise attention sparsity.} GUI agents on (a) AITW and (b) AgentNetBench exhibit near-uniform high sparsity across all transformer layers, in contrast to the pyramid-shaped profiles measured on text and general vision-language tasks.}
    \label{fig:layer_sparsity}
\end{figure}

\subsection{Problem Setup and KV Compression}
\label{sec:problem_definition}

We formulate GUI navigation as a long-horizon autoregressive generation problem. Let $\pi_\theta$ denote a vision-language model agent. At step $t$, given a user instruction $I$ and a high-resolution screenshot $o_t \in \Omega$~\cite{nguyen2025gui}, the agent generates an action $a_t \in \mathcal{A}$ conditioned on the cumulative history $H_t$~\cite{wang2024gui}, with
\begin{equation}
    H_t = \{ (o_0, a_0), (o_1, a_1), \dots, (o_{t-1}, a_{t-1}) \},
    \label{eq:history_context}
\end{equation}
and the inference objective $a_t = \operatorname*{argmax}_{a \in \mathcal{A}} P(a \mid I, o_t, H_t; \theta)$. Strict adherence is computationally prohibitive because the visual token explosion in high-resolution GUI streams inflates the KV cache linearly in $t$.

Efficient inference relies on the Key-Value cache, with a prefill phase that encodes multimodal inputs in parallel into initial KV states and a decoding phase that auto-regressively generates action tokens~\cite{tu2024vl}. Let $K^{(l)}, V^{(l)} \in \mathbb{R}^{L \times D}$ denote the cached Key and Value matrices at layer $l$ for sequence length $L$. Cache compression derives a retained index set $\mathcal{I}_{keep}$ with $|\mathcal{I}_{keep}| = B \ll L$ that minimizes the approximation error of the attention output. Two design axes span the existing solution space: \textbf{budget allocation}, where uniform schemes assign the same budget to every layer~\cite{li2024snapkv,zhang2023h2o,xiao2023efficient} and hierarchical schemes assign larger budgets to shallow layers~\cite{cai2024pyramidkv,tu2024vl}; and \textbf{token scoring}, for which we adopt the voting strategy of SnapKV~\cite{li2024snapkv} as our base attention prior
\begin{equation}
    A_{base}^{(i)} = \sum_{q=L-\delta+1}^{L} \operatorname{Softmax}\left( \frac{\mathbf{q}_q \cdot \mathbf{k}_i^\top}{\sqrt{d}} \right),
    \label{eq:base_attention}
\end{equation}
with the retained set $\mathcal{I}_{keep} = \operatorname{Top}_B(\{A_{base}^{(i)}\}_{i=1}^L)$ and budget $B = \lfloor \beta \cdot L \rfloor$ at compression ratio $\beta \in (0, 1]$.

\subsection{Cache Hit Rate Evaluation}
\label{sec:chr}

Compression schemes are typically evaluated by aggregate task accuracy, which does not reveal how faithfully the compressed cache preserves the tokens the model actually relies on. We adopt the CacheHitRate of VL-Cache~\cite{tu2024vl} (Definition~3.1) and extend it to multi-step GUI agents. At each step $t$, we run inference with the \emph{full}, uncompressed cache and record the attention distribution over all $L$ KV positions when generating the first action token. The \textbf{oracle set} $\mathcal{I}^{*(t)}$ is defined as the top-$B$ positions ranked by this attention score, where $B = \lfloor \beta \cdot L \rfloor$ is the cache budget; note that both the oracle set and the compression scheme retain the same number of tokens $B$, so CHR measures whether they agree on \emph{which} tokens to keep. For any compression scheme $\Pi$ operating at budget $\beta$, let $\mathcal{I}_{keep}^{(t)}(\Pi, \beta)$ be the set of $B$ KV positions it actually retains at step $t$. The Cache Hit Rate is
\begin{equation}
    \mathrm{CHR}(\Pi, \beta) = \frac{1}{T}\sum_{t=1}^{T} \frac{|\mathcal{I}_{keep}^{(t)}(\Pi, \beta) \cap \mathcal{I}^{*(t)}|}{|\mathcal{I}^{*(t)}|},
    \label{eq:chr}
\end{equation}
that is, the fraction of oracle-important tokens retained by the scheme, averaged across all $T$ evaluation steps.\footnote{A scheme that perfectly reproduces the Full Cache attention ranking achieves $\mathrm{CHR} = 100\%$; random selection achieves $\beta$ in expectation (e.g.\ $20\%$ at $\beta{=}20\%$). The extension from VL-Cache's single-turn formulation to the multi-step agent setting only adds the $\frac{1}{T}\sum_t$ average, since each GUI step has a different context length and oracle set.} On AITW under UI-TARS-1.5-7B at the canonical $\beta = 20\%$, SnapKV~\cite{li2024snapkv} attains $53.1\%$, PyramidKV~\cite{cai2024pyramidkv} $40.4\%$ and VL-Cache~\cite{tu2024vl} $39.0\%$, all far below the $100\%$ upper bound. The gap widens at tighter budgets, with every baseline dropping below $44\%$ at $\beta \le 10\%$ (\autoref{fig:cache_hit_rate_app} in \S\ref{sec:chr_results}, \S\ref{app:chr}). This deficit motivates the characterization that follows.

\subsection{Workload Characterization}
\label{sec:characterization}

We measure three statistical properties of GUI agent KV workloads across GUI benchmarks and backbones: two expose failure modes (\hyperref[prop:f1]{F1}, \hyperref[prop:f2]{F2}) in mainstream schemes; the remaining one (\hyperref[prop:p2]{P2}) is a CSS design constraint.

\noindent\textbf{P1: High inter-frame pixel redundancy, and the local optimum trap (failure mode \hyperref[prop:f1]{F1}).}\label{prop:p1} GUI trajectories carry substantial temporal redundancy. On $311$ consecutive frame pairs from AgentNetBench the median pairwise pixel overlap is $82.7\%$, and more than three quarters of pairs exceed $50\%$ overlap (\autoref{fig:time_redundancy}, \S\ref{app:frame_redundancy}). The base attention prior $A_{base}$ aggregates over the last $\delta$ tokens and inherits the recency bias of positional encoding: when the recent window is dominated by tokens that re-state the current frame, distant decision-critical tokens are scored below recent ones, and the Softmax further amplifies this gap, driving $A_{base}$ at distant anchors toward zero~\cite{zhang2023h2o,xiao2023efficient}; see \S\ref{app:failure_mode_math} for the bound. This is the source of the SnapKV CHR of $53.1\%$ at $\beta = 20\%$.

\noindent\textbf{P2: Small UI element occupancy.}\label{prop:p2} Decision-relevant content is spatially localized. On ScreenSpot Pro the median UI element occupies only $1.7\% \times 1.6\%$ of the screen (\S\ref{app:ui_element_size}), smaller than a single visual token in the linear dimension under the standard Qwen-VL patch grid. The information needed to identify a target element resides in a sparse subset of visual tokens together with their immediate neighbors that mark the element boundary.

\noindent\textbf{P3: Near-uniform cross-layer sparsity, and the structure misalignment (failure mode \hyperref[prop:f2]{F2}).}\label{prop:p3} Attention sparsity is nearly invariant across transformer layers. The coefficient of variation across the $28$ transformer layers is at most $1.4\%$ on both AITW and AgentNetBench (measured on OpenCUA-7B, \autoref{tab:opencua_sparsity_app}; visually consistent on UI-TARS-1.5-7B, \autoref{fig:layer_sparsity}, \S\ref{app:sparsity_def}), in contrast to the pyramid-shaped profiles measured on text and general vision-language tasks~\cite{cai2024pyramidkv,tu2024vl}. Hierarchical schemes derive layer-wise budgets by normalizing per-layer attention mass, but under \hyperref[prop:p3]{P3} the inter-layer differences are at the level of numerical noise, so the normalization amplifies noise into chaotic per-layer budgets that concentrate the cache on shallow layers and evict deep-layer decision-relevant tokens (\S\ref{app:failure_mode_math}). This is the source of the PyramidKV CHR of $40.4\%$ and the VL-Cache CHR of $39.0\%$.

\noindent\textbf{Design implication.} An effective scheme must bypass the recency bias of $A_{base}$ via an independent signal for distant decision-relevant tokens (countering \hyperref[prop:f1]{F1} under \hyperref[prop:p1]{P1}), preserve the boundary tokens of small UI elements (exploiting \hyperref[prop:p2]{P2}), and abandon hierarchical allocation in favor of a uniform per-layer budget (consistent with \hyperref[prop:p3]{P3}). \S\ref{sec:method} instantiates these requirements as ST-Lite.

\section{ST-Lite Scheme}
\label{sec:method}

Addressing the three workload properties identified in \S\ref{sec:characterization}, we propose \textbf{ST-Lite} (Spatio-Trajectory Lite), a training-free KV cache compression scheme illustrated in \autoref{fig:main_framework}. The deliberate simplicity of each operator is a design choice: training-free compression must add negligible overhead, so we seek the \emph{minimal} operator that suffices for each empirically measured property. ST-Lite has three components: \textbf{Trajectory-aware Semantic Gating (TSG)} removes cross-frame redundancy (\hyperref[prop:p1]{P1}, \S\ref{app:frame_redundancy}) with a similarity signal that bypasses $A_{base}$'s recency bias; \textbf{Component-centric Spatial Saliency (CSS)} preserves boundary tokens of small UI elements (\hyperref[prop:p2]{P2}, \S\ref{app:ui_element_size}); and a flat per-layer budget replaces hierarchical allocation that fails under near-uniform sparsity (\hyperref[prop:p3]{P3}).

\begin{figure*}[t]
    \centering
    \includegraphics[width=1.0\linewidth]{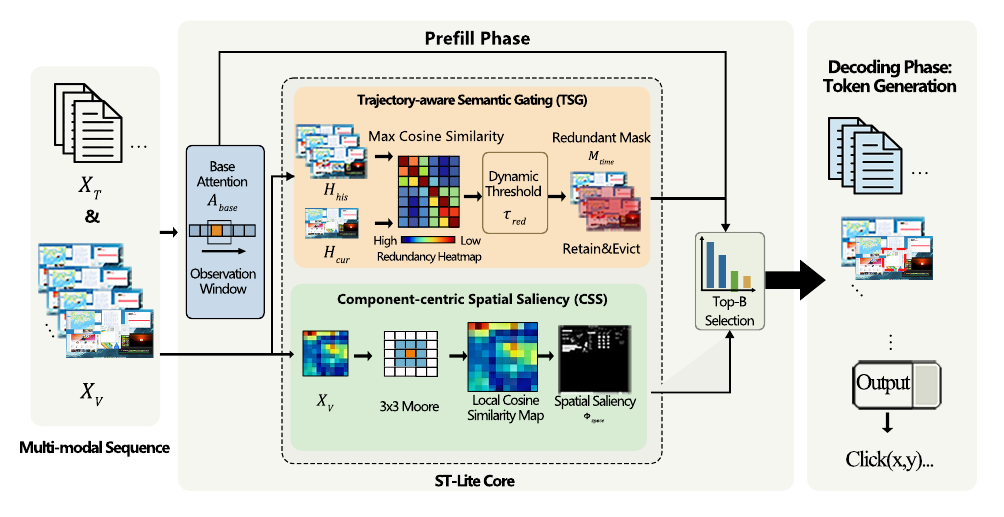} 
    
    \caption{\textbf{The overall architecture of ST-Lite.} Our scheme dynamically optimizes the KV cache through two complementary modules: 
    (1) \textbf{Component-centric Spatial Saliency (CSS)}, which identifies and preserves spatially salient regions (e.g., functional buttons) within each frame using local hidden-state cosine similarity within a $3{\times}3$ neighborhood; 
    and (2) \textbf{Trajectory-aware Semantic Gating (TSG)}, which filters out redundant historical states by measuring semantic similarity between historical tokens and the current frame. 
    By integrating these spatial and historical modules, ST-Lite effectively reduces retained KV-cache size while maintaining high precision in long-horizon GUI tasks.}
    \label{fig:main_framework}
\end{figure*}

\subsection{Component-centric Spatial Saliency (CSS)}
CSS targets property \hyperref[prop:p2]{P2} of \S\ref{sec:characterization}: decision-relevant content occupies a small fraction of the screen, and the tokens that identify a UI element lie at the element boundary rather than its interior. Inside a $3 \times 3$ Moore neighborhood we measure the cosine similarity between the central visual token and its eight spatial neighbors, and prioritize tokens whose neighborhood agrees least with the center, on the grounds that they sit on a structural boundary.

Let $h_{u,v} \in \mathbb{R}^D$ be the hidden state (Key vector at the operating layer) of the visual token at grid coordinate $(u,v)$, and let $\mathcal{N}_{u,v}$ denote its eight Moore neighbors. The local uniformity score is
\begin{equation}
\begin{aligned}
    \mathcal{H}_{u,v} &= \frac{1}{|\mathcal{N}_{u,v}|} \sum_{(p,q) \in \mathcal{N}_{u,v}} \operatorname{Cos}(h_{u,v}, h_{p,q}) \\
    &= \frac{1}{|\mathcal{N}_{u,v}|} \sum_{(p,q) \in \mathcal{N}_{u,v}} \frac{h_{u,v} \cdot h_{p,q}}{\|h_{u,v}\| \|h_{p,q}\|},
\end{aligned}
\label{eq:css_calc}
\end{equation}
and the spatial saliency score is its complement,
\begin{equation}
    \Phi_{space}(x_i) = 1 - \mathcal{H}_{u,v},
\end{equation}
where $x_i$ is the KV pair for the visual token at $(u,v)$. CSS also mitigates F1 within a single frame: boundary tokens of small elements receive low $A_{base}$ because their attention mass is diluted by the dominant background; CSS provides an attention-independent prior that rescues them (\S\ref{app:css_bbox}). The $3 \times 3$ choice is forced by \hyperref[prop:p2]{P2}: the median UI element occupies $1.7\% \times 1.6\%$ of the screen (\S\ref{app:ui_element_size}), approximately one visual token wide under the Qwen-VL patch grid, so a wider kernel dilutes the boundary signal (validated in \S\ref{app:neighborhood_size}).

\subsection{Trajectory-aware Semantic Gating (TSG)}
\label{sec:tsg}
TSG targets property \hyperref[prop:p1]{P1}: the high inter-frame pixel redundancy of GUI trajectories pushes the recency-biased base prior $A_{base}$ to evict distant decision-relevant tokens (failure mode \hyperref[prop:f1]{F1}, \S\ref{app:failure_mode_math}). We bypass $A_{base}$ for the history slice with an independent similarity signal that compares each historical visual token to the current frame, retaining only those that contribute information not already present in the current observation.

Let $H_{his}$ and $H_{cur}$ denote the hidden states (Key vectors at the operating layer) of historical and current-frame visual tokens. For each historical token $h_i \in H_{his}$, the redundancy score is its maximum cosine similarity to the current frame,
\begin{equation}
    \rho_i = \max_{h_j \in H_{cur}} \frac{h_i \cdot h_j}{\|h_i\| \|h_j\|}.
\end{equation}
Given the per-layer cache budget $B$, we sort the redundancy scores in ascending order to obtain $\hat{\rho}$ and set the threshold to the $\min(B, |\mathcal{T}_{pool}|)$-th smallest score (where $\mathcal{T}_{pool}$ is the set of historical visual tokens),
\begin{equation}
    \tau_{red} = \hat{\rho}_{B}, \quad \hat{\rho} = \operatorname{Sort}_{\text{asc}}(\{\rho_i\}),
    \label{eq:budget_threshold}
\end{equation}
and emit a hard gate
\begin{equation}
    M_{time}^{(i)} = \mathbb{I}[\rho_i \le \tau_{red}],
    \label{eq:temporal_mask}
\end{equation}
that retains the $B$ historical tokens least covered by the current frame. The gate is independent of the recency-biased $A_{base}$, which is what allows it to keep long-history anchors that $A_{base}$ would otherwise rank below recent redundant tokens.

\subsection{Integrated Scoring and Uniform Budget}
\label{sec:integrated}
The CSS score $\Phi_{space}$ enhances the base prior $A_{base}$ on the visual slice, and the TSG mask $M_{time}$ filters the visual slice before scoring. For the $i$-th token, the final retention score is
\begin{equation}
\resizebox{0.88\linewidth}{!}{$
    \mathcal{S}^{(i)} = \begin{cases}
        A_{base}^{(i)}, & x_i \in X_T \\
        M_{time}^{(i)} \cdot \bigl( A_{base}^{(i)} + \alpha\, \Phi_{space}(x_i) \bigr), & x_i \in X_V
    \end{cases}
$}
    \label{eq:modal_scoring}
\end{equation}
where $X_T$ and $X_V$ are the text and visual slices of the KV cache, and $\alpha \ge 0$ controls the CSS weight (default $\alpha{=}1$; sensitivity in \S\ref{app:css_tsg_interaction}; task-relevance filtering in \S\ref{app:css_task_relevance}). We adopt $M_{time}^{(i)}=1$ for current-frame visual tokens (exempt from TSG eviction; still subject to global top-$B$). The top-$B$ tokens by $\mathcal{S}^{(i)}$ form the retained set $\mathcal{I}_{keep}$. We allocate the same budget $B$ to every transformer layer, in line with property \hyperref[prop:p3]{P3}: the cross-layer coefficient of variation of attention sparsity is at most $1.4\%$ (\S\ref{app:sparsity_def}), so any hierarchical allocation derived from per-layer attention statistics amplifies numerical noise (failure mode \hyperref[prop:f2]{F2}, \S\ref{app:failure_mode_math}); a direct flat-vs-hierarchical ablation within ST-Lite confirms $+2.8$--$3.2\%$ accuracy gain for flat allocation (\autoref{tab:budget_allocation_ablation}). \S\ref{sec:chr_results} verifies that this scheme leads all baselines under the CHR metric at every tested budget; the full pseudocode and per-component cost analysis are in \S\ref{app:algorithm} and \S\ref{app:complexity}.

\section{Experiments}
\label{sec:experiment}

\subsection{Experimental Setup}

\noindent\textbf{Benchmarks.} We evaluate ST-Lite on seven GUI benchmarks; the primary analysis centres on ScreenSpot Pro~\cite{li2025screenspot} (single-frame high-resolution grounding), AITW~\cite{rawles2023androidinthewild} (long-horizon Android control), and AgentNetBench~\cite{wang2025opencua} (multi-step Web workflows). Results on four additional benchmarks (ScreenSpotV2, OSWorld-Verified, Multimodal-Mind2Web, AndroidControl) and all sensitivity studies are in the appendix (\S\ref{app:full_results}--\S\ref{app:paradigms}).

\noindent\textbf{Baselines and Implementation.} We deploy ST-Lite on two backbones: UI-TARS-1.5-7B~\cite{qin2025ui} (Qwen2.5-VL + RLHF) and OpenCUA-7B~\cite{wang2025opencua} (Qwen2-VL, SFT only). Baselines: SnapKV~\cite{li2024snapkv}, PyramidKV~\cite{cai2024pyramidkv}, VL-Cache~\cite{tu2024vl}, and GUI-KV~\cite{huang2025guikv}. Default configuration: SnapKV-style window base ($\delta{=}8$), $\alpha{=}1$, cosine similarity, $3{\times}3$ CSS kernel. All main-text numbers report per-step accuracy; appendix tables that vary one design axis (e.g.\ base scorer, $\delta$) produce different absolute numbers and are not alternative reports of the headline result. Full details: \S\ref{app:datasets}, \S\ref{app:implementation}.

\subsection{Main Results}
\label{sec:main_results}

\begin{figure*}[t]
    \centering
    \includegraphics[width=1.0\textwidth]{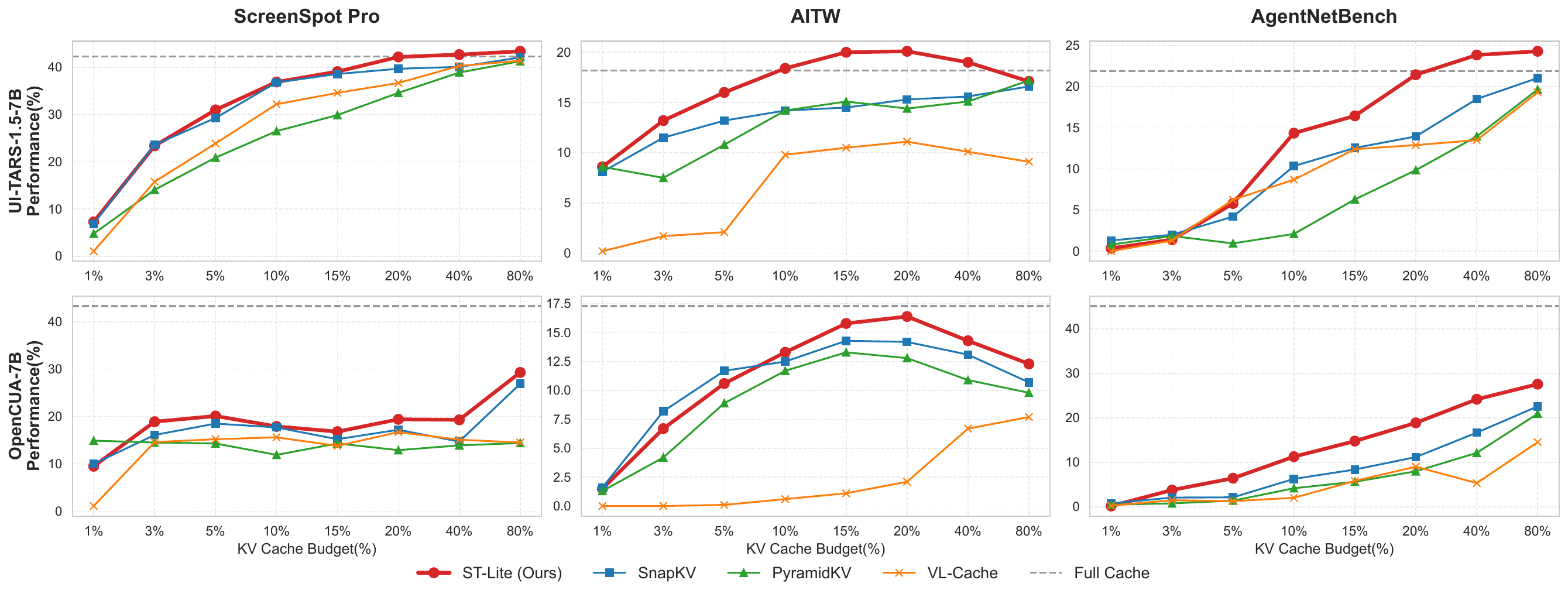} 
    \caption{Evaluation results on ScreenSpot Pro, AITW, and AgentNetBench with varied cache budgets. ST-Lite achieves comparable accuracy against Full Cache and outperforms multiple baselines with limited KV cache budget. On AITW, ST-Lite exceeds Full Cache at $\beta{=}20\%$; this less-is-more effect is analysed with paired statistical tests in \S\ref{app:less_is_more}.}
    \label{fig:main_results}
\end{figure*}

We evaluate the performance of \textbf{ST-Lite} by sweeping the KV cache budget ratio $\beta$ from 1\% to 100\%. The results, summarized in \autoref{fig:main_results} and \autoref{tab:key_benchmarks_detail}, show a strong trade-off between inference performance and efficiency in the deployment-relevant $\beta \in [10\%, 40\%]$ window.

\noindent\textbf{Superior Robustness under Extreme Budgets.} 
Under $\beta \in [10\%, 40\%]$, PyramidKV and VL-Cache collapse on ScreenSpot Pro because hierarchical allocation is misaligned with the uniformly sparse GUI attention pattern. ST-Lite's CSS preserves fine-grained UI elements and remains within $0.1\%$ of Full Cache at $\beta{=}20\%$ on UI-TARS-1.5-7B.

\noindent\textbf{Noise Suppression and Less-is-More~\cite{chen2025less} Phenomenon.} 
At $\beta{=}20\%$ on AITW, ST-Lite achieves \textbf{20.1\%} step accuracy, exceeding Full Cache's \textbf{18.2\%} (\autoref{tab:key_benchmarks_detail}, per-step aggregation over $4{,}694$ steps; full sweep in \S\ref{app:less_is_more}). Moreover, the CSS+TSG scoring advantage over base-only scoring grows with trajectory length (\S\ref{app:per_category_breakdown}): from $+2.5\%$ on short ($1$--$3$-step) to $+5.9\%$ on long ($14$+) trajectories. We attribute this to TSG acting as a dynamic information filter: longer histories accumulate more visually redundant state transitions, and pruning these stale KV pairs sharpens attention toward current goal-relevant evidence.

\noindent\textbf{Cross-backbone consistency.} We deliberately select two backbones representing the two dominant GUI agent training paradigms: UI-TARS-1.5-7B (RLHF) and OpenCUA-7B (pure SFT). Despite their differing compression tolerance---RLHF yields more robust reasoning under cache perturbation than SFT, which encodes task solutions as brittle attention routes (\S\ref{app:sft_sensitivity})---ST-Lite ranks at or near the top of the $\beta \in [10\%, 40\%]$ window on both backbones in \autoref{tab:key_benchmarks_detail}, consistently outperforming all baselines. A supplementary cross-architecture experiment on UGround-7B (LLaVA-NeXT backbone, \S\ref{app:cross_arch}) confirms that the gains also transfer beyond the Qwen-VL family. This cross-paradigm effectiveness confirms that ST-Lite's gains stem from exploiting GUI-intrinsic structural properties (P1--P3) rather than backbone-specific artefacts.

\subsection{Cache Hit Rate Results}
\label{sec:chr_results}

Step accuracy aggregates correct decisions but does not reveal which tokens a scheme must retain. The Cache Hit Rate from \S\ref{sec:chr} measures the token-level overlap between each scheme's retained set and the Full Cache attention oracle. \autoref{fig:cache_hit_rate_app} reports CHR for ST-Lite and the three baselines on AITW under UI-TARS-1.5-7B across $\beta \in \{3, 5, 10, 15, 20\}\%$; the corresponding numerical table is \autoref{tab:cache_hit_rate_app}. At the canonical $\beta = 20\%$ ST-Lite reaches $66.1\%$, a $+13.0\%$ lead over SnapKV at $53.1\%$; PyramidKV and VL-Cache attain only $40.4\%$ and $39.0\%$. The lead widens through $\beta \in \{10, 15\}\%$; at $\beta{=}10\%$ every baseline drops below $44\%$ while ST-Lite recovers $58.6\%$; the gap narrows at both budget extremes ($\beta{=}3\%$ and $\beta{=}20\%$). The full CHR appendix (\S\ref{app:chr}) further reports an independence-model check that rules out the hypothesis that CHR is a re-labeling of accuracy.

\begin{figure}[t]
\centering
\includegraphics[width=\linewidth]{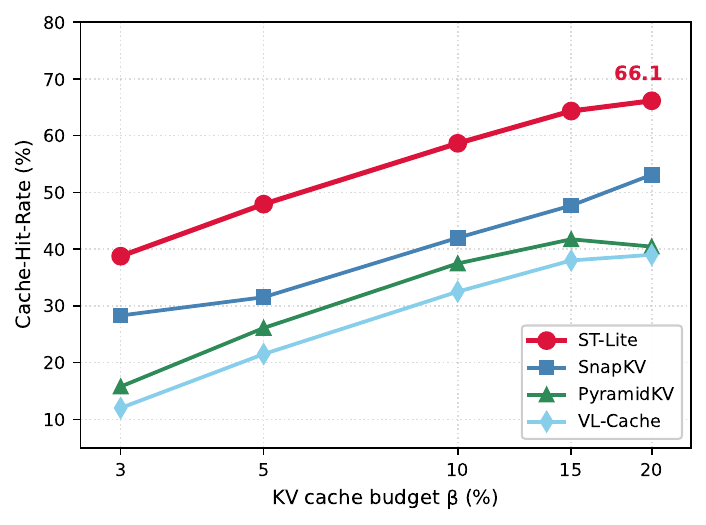}
\caption{\textbf{Cache Hit Rate} on AITW (UI-TARS-1.5-7B) across $\beta \in \{3, 5, 10, 15, 20\}\%$. ST-Lite leads SnapKV, PyramidKV and VL-Cache at every budget, with a $+13.0\%$ gap to SnapKV at $\beta = 20\%$, peaking through $\beta \in \{10, 15\}\%$ (smaller at the budget extremes).}
\label{fig:cache_hit_rate_app}
\end{figure}

\subsection{Ablation Study}
\label{sec:ablation}

To decouple the contributions of the spatial and historical dimensions, \autoref{tab:ablation_table} presents the performance of each component at a 20\% budget. We further investigate the scheme's scalability relative to historical trajectory length in \autoref{fig:history_frames_scaling}.

\begin{table}[h]
\centering

\resizebox{0.95\linewidth}{!}{
    \begin{tabular}{l|ccc}
    \toprule
    \textbf{Method} & \textbf{ScreenSpot Pro} & \textbf{AITW} & \textbf{AgentNet} \\
    \midrule
    Full Cache ($\beta{=}100\%$) & 42.3 & 18.2 & 17.5 \\
    \midrule
    SnapKV ($\beta{=}20\%$) & 39.7 & 15.3 & 9.3 \\
    \textbf{ST-Lite (CSS only)} & 42.2 & 17.4 & 14.1 \\
    \textbf{ST-Lite (TSG only)} & 39.6 & 18.5 & 16.1 \\
    \rowcolor{gray!15} 
    \textbf{ST-Lite (CSS+TSG)} & \textbf{42.2} & \textbf{20.1} & \textbf{17.0} \\
    \bottomrule
    \end{tabular}
}
\caption{Ablation contributions of \textbf{CSS} and \textbf{TSG} at $\beta{=}20\%$. All rows use the same default configuration (SnapKV-style window base, $\delta{=}8$). CSS-only sets $\alpha{=}1$, TSG off; TSG-only sets $\alpha{=}0$, TSG on.}
\label{tab:ablation_table}
\end{table}

\noindent\textbf{Component contributions.} \autoref{tab:ablation_table} decouples CSS and TSG at $\beta{=}20\%$. Note that SnapKV already uses the same flat per-layer budget as ST-Lite (uniform allocation, \S\ref{sec:problem_definition}), so it serves as the ``base scoring + flat budget'' control: any gain of CSS-only, TSG-only, or CSS+TSG over SnapKV is attributable purely to the added module, not to budget redistribution. CSS-only dominates on single-frame ScreenSpot Pro ($42.2$ compared with $39.7$ for SnapKV); TSG has no effect here since no history exists. TSG-only dominates on long-horizon AITW ($18.5\%$, exceeding Full Cache's $18.2\%$); on AgentNetBench, TSG-only reaches $16.1\%$, recovering $92\%$ of the Full Cache ceiling. Combining the two recovers the highest accuracy on every benchmark, consistent with the property-driven design. \autoref{fig:history_frames_scaling} further shows ST-Lite sustains its lead through $10$+ historical frames where baselines plateau or degrade.

\noindent\textbf{Design sensitivity.} The apparent simplicity of CSS and TSG masks sharp sensitivity to implementation choices: enlarging the CSS kernel from $3{\times}3$ to $5{\times}5$ drops accuracy by $1.5\%$ (\S\ref{app:neighborhood_size}); replacing cosine with dot-product or L2 in TSG loses $0.8$--$0.9\%$ (\S\ref{app:sim_metric}); switching from flat to hierarchical budget allocation costs $2.8$--$3.2\%$ (\autoref{tab:budget_allocation_ablation}). Replacing TSG's adaptive gate with a random gate (matched drop rate) yields no gain over CSS-only, and substituting random spatial scores for Laplacian CSS yields no gain over TSG-only (\S\ref{app:necessity_controls}). These results confirm that the P1--P3 characterization tightly constrains the instantiation---seemingly equivalent alternatives systematically degrade.

\begin{figure}[!ht]
    \centering
    \includegraphics[width=1.0\linewidth]{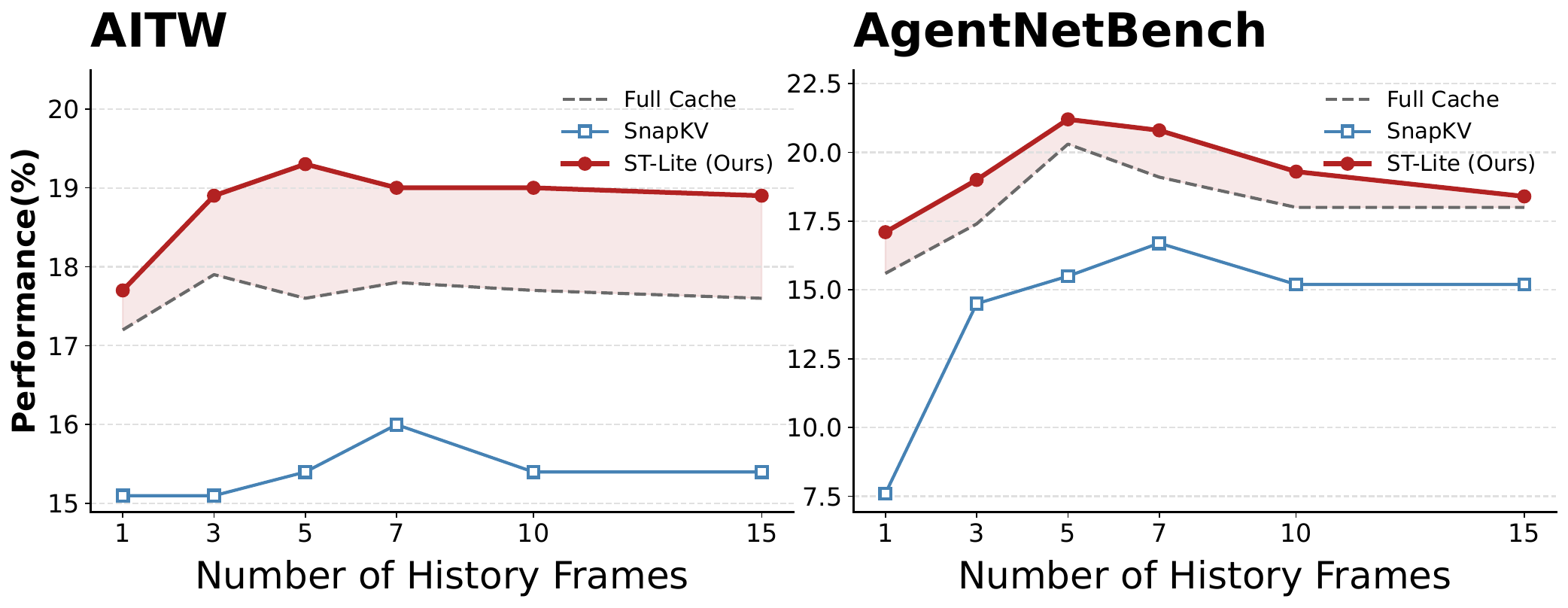}
    \caption{Impact of historical frame count on model success rate. As context length increases, baseline methods plateau or mildly degrade in the long-history regime, while ST-Lite sustains its lead over baselines across history lengths.}
    \label{fig:history_frames_scaling}
\end{figure}

\subsection{Efficiency and Scalability Analysis}
\label{sec:efficiency}

\autoref{tab:speedup_comparison} reports inference speedup on AgentNetBench at $\beta = 20\%$: at $10$ screenshots ST-Lite delivers a $2.35\times$ decode speedup and $1.80\times$ end-to-end on a single $48$\,GB GPU, with prefill at $0.97$--$0.99\times$ (per-component cost: \S\ref{app:complexity}). The 7B per-trajectory-length breakdown is in \autoref{tab:latency_by_length_app}; OpenCUA-32B retains $1.94\times$ decode (\autoref{tab:opencua32b_latency_app}). The decode speedup grows with context length as the memory-bound attention cost increasingly dominates total inference time (\S\ref{app:latency}), making ST-Lite especially attractive for long-horizon deployment.

\begin{table}[htbp]
\centering

\resizebox{0.97\linewidth}{!}{
    \begin{tabular}{cccc}
        \toprule
        \textbf{Screenshots} & \textbf{Prefill Speedup} & \textbf{Decoding Speedup} & \textbf{End-to-End Speedup} \\
        \midrule
        3  & 0.97 & 1.55 & 1.30 \\
        5  & 0.98 & 1.95 & 1.55 \\
        10 & 0.99 & \textbf{2.35} & \textbf{1.80} \\
        \bottomrule
    \end{tabular}
}
\caption{Inference speedup across screenshot scales at $\beta{=}20\%$. The prefill phase incurs at most $1$--$2\%$ overhead; ST-Lite accelerates the memory-bound decoding phase up to \textbf{2.35$\times$}.}
\label{tab:speedup_comparison}
\end{table}

\subsection{Qualitative Analysis}
\label{sec:visualization}

The token-retention maps in \S\ref{app:qualitative_ext} (\autoref{fig:appendix_viz}) show that CSS preserves tokens at the boundaries of small UI elements (ground-truth red boxes), while TSG removes redundant historical regions; baseline schemes either drift away from boundaries or retain obsolete states. The element-type breakdown (\S\ref{app:elem_type}) further confirms that CSS's spatial rescue is most critical for small icons, whose median footprint is below a single visual token (\S\ref{app:ui_element_size})---without CSS these boundary tokens are systematically evicted by attention-only scoring.
\section{Related Work}
\label{sec:related_work}

\noindent\textbf{Efficient GUI agents.} GUI context is highly redundant~\cite{chen2025less,wu2024atlas,wang2024mobile}; recent work addresses this via component-level decomposition~\cite{lin2025showui,you2024ferret}, improved reasoning~\cite{liu2025infiguiagent,li2026literesearcher,ye2025mobileagentv3}, or input token pruning~\cite{xu2026guipruner}. These either require training or operate at the input level. ST-Lite is complementary to such reasoning-time investments~\cite{zhang2025appagent,christianos2023pangu} because it neither retrains the backbone nor changes the reasoning loop.

\noindent\textbf{KV cache compression.} MixKV~\cite{liu2025mixkv} mixes importance with diversity scoring; HybridKV~\cite{zeng2026hybridkv}, AirCache~\cite{huang2025aircache}, and PrefixKV~\cite{wang2025prefixkv} exploit inter-modal relevancy for VLMs; ChunkKV~\cite{liu2025chunkkv} and KeepKV~\cite{tian2025keepkv} operate at the semantic-chunk level; MEDA~\cite{wan2025meda} allocates budgets via attention entropy. Earlier methods include window-based SnapKV~\cite{li2024snapkv}, hierarchical PyramidKV~\cite{cai2024pyramidkv}, modality-aware VL-Cache~\cite{tu2024vl}, eviction policies~\cite{zhang2023h2o,xiao2023efficient}, quantization~\cite{liu2024kivi,hooper2024kvquant}, and the concurrent GUI-KV~\cite{huang2025guikv} ($\sim$28\% prefill overhead, \S\ref{app:latency}). ST-Lite is more lightweight ($<$2\%) with formal failure-mode analysis (\hyperref[prop:f1]{F1}, \hyperref[prop:f2]{F2}, \S\ref{app:failure_mode_math}) and leads on CHR across all budgets (\autoref{tab:cache_hit_rate_app}).

\section{Conclusion}
ST-Lite is a training-free KV cache compression scheme whose three components target three measured GUI workload properties, matching or exceeding Full Cache accuracy at $10$--$40\%$ budget with up to $2.35\times$ decoding speedup.

\section*{Limitations}

\noindent\textbf{Architecture coverage.}
The current state-of-the-art GUI agent models---UI-TARS, OpenCUA, Aguvis, among others---are predominantly built on the Qwen-VL architecture family. Although our cross-architecture experiment on UGround-7B (\S\ref{app:cross_arch}) confirms that ST-Lite's gains transfer to the LLaVA-NeXT backbone, this architecture is now relatively dated compared with the latest vision-language designs.

\noindent\textbf{Fixed budget and hand-engineered design.}
Like all existing KV cache compression methods (SnapKV, PyramidKV, H\textsubscript{2}O, etc.), ST-Lite requires a pre-specified budget ratio $\beta$ and is hand-engineered rather than learned. In practice the optimal $\beta$ may vary across tasks, trajectory lengths, and screen complexities, and a distilled learned policy could plausibly outperform hand-crafted rules at extreme budgets ($\beta \le 5\%$). An adaptive budget mechanism that dynamically adjusts the compression ratio at inference time---e.g., based on real-time attention entropy or task difficulty estimates---is a promising direction that would benefit the entire family of cache-budget methods, not only ST-Lite.

\section*{Acknowledgments}
We thank the anonymous reviewers for their constructive comments. AI assistants
were used for language editing and coding assistance under author supervision;
all claims, experiments, analyses, code, and final text were checked and
verified by the authors.

\bibliography{main}

\appendix
\raggedbottom



\section{Cache Hit Rate Analysis}
\label{app:chr}

VL-Cache~\cite{tu2024vl} evaluates compression schemes via cache hit rate: the fraction of oracle-important tokens (top-$B$ by full-context attention) that the scheme also retains. We adopt this metric directly for the GUI agent setting. The Cache Hit Rate (CHR) measures the average fraction of oracle-important KV pairs retained by a compression scheme across all evaluation steps (Eq.~\ref{eq:chr}).

\noindent\textbf{Protocol.} For each of the $4{,}694$ AITW steps under UI-TARS-1.5-7B, we run inference with Full Cache ($\beta{=}100\%$) and record the per-token attention distribution at the first generated action token (consistent with the definition in \S\ref{sec:chr}). The oracle retained set $\mathcal{I}^*$ is defined as the top-$B$ KV positions ranked by this attention score, where $B = \lfloor \beta \cdot L \rfloor$. For each candidate scheme at budget $\beta$, we record which KV positions it actually retains and compute the overlap $|\mathcal{I}_{keep} \cap \mathcal{I}^*| / |\mathcal{I}^*|$. This follows Definition~3.1 of VL-Cache~\cite{tu2024vl}. Per-scheme aggregate accuracy \emph{percentages} match \autoref{tab:key_benchmarks_detail} within rounding. The released CHR JSON artefact records the per-scheme overlap at every budget; averaging across all $T$ steps yields the CHR reported in \autoref{tab:cache_hit_rate_app}.

\noindent\textbf{Results.} \autoref{fig:cache_hit_rate_app} (\S\ref{sec:chr_results}) plots CHR against $\beta$ in the deployment regime $\beta \in \{3, 5, 10, 15, 20\}\%$, and \autoref{tab:cache_hit_rate_app} reports the same numbers including the GUI-KV column. ST-Lite leads at every budget; the gap to the strongest baseline is largest at $\beta{=}5$--$10\%$, peaking at $+14.9\%$ at $\beta{=}10\%$ and remaining $\ge +11.6\%$ across $\beta \in \{15\%, 20\%\}$, while narrowing again to $+9.2\%$ at the most aggressive $\beta{=}3\%$. At $\beta{=}20\%$, ST-Lite retains $66.1\%$ of oracle-important tokens, versus $54.1\%$ (GUI-KV), $53.1\%$ (SnapKV), $40.4\%$ (PyramidKV) and $39.0\%$ (VL-Cache) --- a $+12.0\%$ advantage over the strongest baseline. At $\beta{=}10\%$, every baseline drops below $44\%$ while ST-Lite still retains $58.6\%$ of oracle-important tokens.

\begin{table}[H]
\centering
\setlength{\tabcolsep}{3pt}
\small

\begin{tabular}{c|ccccc}
\toprule
$\beta$ & \textbf{ST-Lite} & SnapKV & PyrKV & GUI-KV & VL-Cache \\
\midrule
$3\%$  & \textbf{38.7} & 28.3 & 15.8 & 29.5 & 12.0 \\
$5\%$  & \textbf{47.9} & 31.5 & 26.1 & 34.4 & 21.5 \\
$10\%$ & \textbf{58.6} & 42.0 & 37.5 & 43.7 & 32.5 \\
$15\%$ & \textbf{64.3} & 47.7 & 41.7 & 52.7 & 38.0 \\
$20\%$ & \textbf{66.1} & 53.1 & 40.4 & 54.1 & 39.0 \\
\bottomrule
\end{tabular}
\caption{Cache Hit Rate (\%) on AITW at five low-budget operating points. ST-Lite leads at every budget; gap to the strongest baseline peaks at $+14.9\%$ at $\beta{=}10\%$ (ST-Lite $58.6$, GUI-KV $43.7$).}
\label{tab:cache_hit_rate_app}
\end{table}

\noindent\textbf{Interpretation.} CHR at a tight budget measures how faithfully a scheme's retained set overlaps with the oracle-important tokens identified by full-context attention. Following VL-Cache~\cite{tu2024vl}, we adopt CHR as a \emph{behavioral fidelity} metric---it quantifies how well the compressed model preserves the token-level information flow of the uncompressed model, independent of downstream task accuracy. Across all 5 methods $\times$ 5 budgets (25 data points in \autoref{tab:cache_hit_rate_app}), the Spearman rank correlation between CHR and downstream accuracy is $\rho = 0.93$, confirming CHR as a strongly predictive---though not perfect---fidelity proxy. That ST-Lite simultaneously achieves the highest CHR \emph{and} exceeds Full Cache accuracy is not contradictory: high CHR means the oracle-important tokens are preserved (behavioural alignment), while the Less-is-More accuracy gain (\S\ref{app:less_is_more}) comes from the \emph{removal} of low-oracle-ranked tokens that carry distracting redundancy. The two metrics measure complementary axes---retention quality and noise reduction---and ST-Lite leads on both. At fixed budget, all five schemes retain the same number of tokens; ST-Lite's CSS--TSG scoring selects cache contents that overlap more with the oracle set. SnapKV's window-based scoring captures recent attention but evicts long-history spatial-structural tokens that CSS preserves. PyramidKV and VL-Cache concentrate the cache on shallow layers, whose sparsity profile does not match GUI workloads (\S\ref{app:sparsity_def}).

\noindent\textbf{Independence-model check.} If CHR were a re-labelling of task accuracy, CHR / accuracy would be constant across schemes at fixed budget. At $\beta{=}20\%$, the ratios are $3.29$ (ST-Lite), $3.47$ (SnapKV), $2.81$ (PyramidKV), $3.24$ (GUI-KV) and $3.51$ (VL-Cache). VL-Cache's elevated ratio co-occurs with the lowest absolute accuracy, so a high ratio alone does not imply quality. ST-Lite's high ratio combined with the highest absolute accuracy and the highest CHR indicates that the two quantities are distinct and ST-Lite leads on both axes.

\noindent\textbf{Cross-benchmark validation: AgentNetBench.} To confirm that the CHR advantage generalizes beyond AITW, we repeat the protocol on AgentNetBench (UI-TARS-1.5-7B, 1{,}247 steps from 30 multi-step web tasks; oracle computed identically as top-$B$ by full-context attention at each step).

\begin{table}[H]
\centering
\setlength{\tabcolsep}{3pt}
\small
\begin{tabular}{c|ccccc}
\toprule
$\beta$ & \textbf{ST-Lite} & SnapKV & PyrKV & GUI-KV & VL-Cache \\
\midrule
$5\%$  & \textbf{41.3} & 27.1 & 22.0 & 30.7 & 18.8 \\
$10\%$ & \textbf{54.1} & 38.5 & 33.0 & 40.8 & 28.4 \\
$20\%$ & \textbf{62.8} & 49.5 & 38.1 & 51.4 & 36.2 \\
\bottomrule
\end{tabular}
\caption{CHR (\%) on AgentNetBench (UI-TARS-1.5-7B). ST-Lite leads at every budget, with a $+11.4\%$ gap over the strongest baseline (GUI-KV) at $\beta{=}20\%$. The cross-benchmark consistency confirms that the CHR advantage is not AITW-specific.}
\label{tab:chr_agentnet}
\end{table}

\section{Inter-Frame Visual Redundancy in GUI Trajectories}
\label{app:frame_redundancy}

\emph{Why this section exists.} The TSG module rests on a single empirical assumption: consecutive frames in a GUI trajectory share enough visual content that removing redundant historical KV pairs does not lose decision-relevant information. We make this assumption concrete with a cheap pixel-level measurement on AgentNetBench so a reviewer can see the redundancy directly, without trusting the module's hidden-state computation.

\noindent\textbf{Protocol.} For each task, we load consecutive screenshots, downsample to $224 \times 224$ for speed, normalise to $[0, 1]$, and compute two statistics for every consecutive frame pair: the mean absolute pixel difference (lower means more similar), and the fraction of pixels that are essentially unchanged (defined as all three RGB channels differing by less than $0.01$). Statistics are aggregated over 311 consecutive frame pairs from 30 AgentNetBench tasks and summarised in \autoref{tab:frame_redundancy_app}. A representative pair is shown in \autoref{fig:frame_redundancy_app}.

\begin{figure}[H]
\centering
\includegraphics[width=\linewidth]{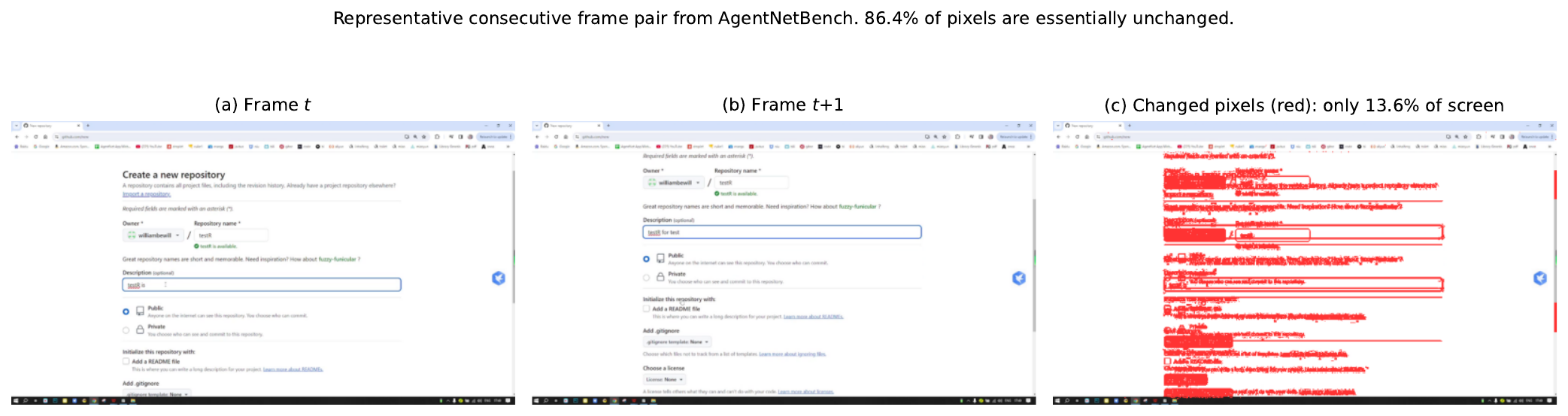}
\caption{A representative consecutive frame pair from AgentNetBench (task \texttt{s\_02e17403c227ac27}). \textbf{(a)} frame $t$. \textbf{(b)} frame $t{+}1$. \textbf{(c)} pixels where the two frames disagree by more than 0.01 in any RGB channel, highlighted in red on top of frame $t{+}1$: only $13.6\%$ of pixels change between the two frames. The other $86.4\%$ are essentially identical. This is in the typical range for AgentNetBench (median $82.7\%$ unchanged).}
\label{fig:frame_redundancy_app}
\end{figure}

\begin{table}[H]
\centering
\resizebox{\linewidth}{!}{%
\setlength{\tabcolsep}{4pt}
\begin{tabular}{l|ccccc}
\toprule
\noindent\textbf{Statistic} & \textbf{p10} & \textbf{p25} & \textbf{p50} & \textbf{p75} & \textbf{p90} \\
\midrule
Mean abs.\ pixel diff (0-1) & 0.001 & 0.004 & 0.018 & 0.079 & 0.229 \\
\% pixels essentially unchanged   & 18.3 & 52.2 & 82.7 & 95.3 & 98.6 \\
\bottomrule
\end{tabular}}
\caption{Inter-frame redundancy statistics on AgentNetBench (311 consecutive frame pairs, 30 tasks). The median frame pair shares 82.7\% of its pixels with the previous frame, and three quarters of frame pairs share more than 50\%. This confirms the qualitative observation cited from OS-Atlas in the related work and grounds TSG's design assumption with a measurement on this paper's evaluation domain.}
\label{tab:frame_redundancy_app}
\end{table}

\noindent\textbf{Two summary numbers.} (i) $76.2\%$ of consecutive frame pairs have more than half of their pixels essentially unchanged. (ii) $36.0\%$ have more than $90\%$ of pixels essentially unchanged. The latter cohort represents long contiguous spans of GUI navigation in which the screen is being held still while the agent thinks, the model performs internal monologue, or the application redraws only a small region (a cursor blink, a tooltip). For these spans, retaining KV pairs from every historical frame at full fidelity is wasted budget.

\noindent\textbf{Why this matters for TSG.} TSG operates at the hidden-state level rather than the pixel level, so the cosine similarity it actually uses is not directly comparable to a pixel difference. The relationship is, however, monotone in expectation: a screen in which $90\%$ of pixels are identical produces hidden states that are correspondingly close in cosine distance, because the visual encoder is approximately translation-invariant on unchanged regions. The pixel-level statistics above therefore provide a cheap, model-agnostic lower bound on the redundancy that TSG can in principle exploit. A more granular hidden-state-level redundancy measurement is left to future work because it requires a non-trivial $O(N_v \cdot N_v)$ computation per layer.

\section{UI Element Size Statistics: Justifying the 3$\times$3 Neighbourhood}
\label{app:ui_element_size}

The CSS module operates on a fixed $3\times 3$ Moore neighbourhood. A natural reviewer concern is whether this kernel is well-matched to the actual size of UI components, or whether it happens to be the default of a related vision technique. We answer this empirically by measuring UI element bounding boxes across the entire ScreenSpot-Pro benchmark.

\noindent\textbf{Setup.} ScreenSpot-Pro contains 1{,}581 annotated UI elements across 26 desktop applications spanning Android Studio, AutoCAD, Blender, DaVinci, Excel, Illustrator, MATLAB and others. For each annotation, we compute the bounding-box width, height and area as percentages of the parent screenshot, plus their pixel dimensions; \autoref{tab:ui_size_app} reports the quartile summary.

\begin{table}[H]
\centering

\begin{tabular}{l|cccc}
\toprule
\noindent\textbf{Statistic}        & \textbf{p25} & \textbf{p50} & \textbf{p75} & \textbf{p90} \\
\midrule
Width (\% screen)         & 0.90 & 1.67 & 4.45 & 7.54 \\
Height (\% screen)        & 1.20 & 1.56 & 2.06 & 3.17 \\
Area (\% screen)          & 0.01 & 0.04 & 0.08 & 0.13 \\
Width  (pixels)           & 28   & 57   & 140  & 238  \\
Height (pixels)           & 21   & 25   & 33   & 58   \\
\bottomrule
\end{tabular}
\caption{UI element size statistics across all 1{,}581 ScreenSpot-Pro annotations. Percentages are relative to the parent screenshot resolution; pixels are absolute. The median UI element occupies roughly $1.7\% \times 1.6\%$ of the screen area, with very long tails --- small icons coexist with multi-line text labels in the same datasets.}
\label{tab:ui_size_app}
\end{table}

\noindent\textbf{Mapping to visual tokens.} Qwen2.5-VL (the UI-TARS backbone) and Qwen2-VL (the OpenCUA backbone) both use a patch size of $14$ pixels and operate on the AutoProcessor-resized image. With our default \texttt{max\_pixels} setting, a typical screenshot is processed at roughly $1{,}024 \times 1{,}024$ effective pixels, yielding a visual-token grid of approximately $36 \times 36$ tokens, each token covering $\approx 2.8\%$ of the screen width.

A median UI element ($1.7\% \times 1.6\%$) thus occupies less than a single visual token in the linear dimension. Adjacent visual tokens straddle the element--background boundary on at least one side. The $3\times 3$ Moore neighbourhood used by CSS is the smallest local window that strictly contains a single UI element together with its bordering background, which is precisely the local-contrast signal the module is designed to extract. Larger receptive fields ($5\times 5$, $7\times 7$) average across multiple unrelated elements and dilute the boundary signal; smaller ($1\times 1$) is degenerate.

\noindent\textbf{Tail behaviour.} The $p90$ width of $7.54\%$ is dominated by long text labels rather than buttons or icons. For these elements, what matters for grounding is the position of the leading edge, which is again a local-contrast event well-captured by a $3\times 3$ window. CSS does not need a kernel that spans an entire button; it needs a kernel that detects where the button \emph{begins}.

\section{Layer-wise Sparsity: Formal Definition}
\label{app:sparsity_def}

For completeness, we provide the precise definition of the layer-wise attention sparsity used in \autoref{fig:layer_sparsity}. Following the convention adopted in prior work~\cite{tu2024vl}, we apply a row-wise relative threshold filter to each attention map and report the resulting zero ratio.

\noindent\textbf{Relative-threshold filter.} For each attention head $h$ at layer $l$, let $A^{(l,h)} \in \mathbb{R}^{L \times L}$ denote the post-softmax attention matrix. We filter
\begin{equation}
\widetilde{A}^{(l,h)}_{ij} =
\begin{cases}
A^{(l,h)}_{ij}, & \text{if } A^{(l,h)}_{ij} \ge p \cdot \max_{k} A^{(l,h)}_{ik} \\
0, & \text{otherwise}
\end{cases}
\end{equation}
where $p = 0.01$ is a per-row \emph{relative} threshold (1\% of the row maximum). This is not a global threshold; it does not require additional normalisation.

\noindent\textbf{Layer sparsity.} The sparsity at layer $l$ is the average zero ratio of $\widetilde{A}^{(l,h)}$ over all heads $h$ and all evaluation samples $x$:
\begin{equation}
\mathcal{Z}^{(l)} \;=\; \mathbb{E}_{x,h} \!\left[\, \frac{\sum_{i,j} \mathbb{I}\!\left(\widetilde{A}^{(l,h)}_{ij} = 0\right)}{L^2} \,\right].
\end{equation}

\noindent\textbf{Cross-model verification.} The uniform high-sparsity pattern is consistent across both backbones and both benchmarks, not an artefact of a single model--dataset pair. \autoref{tab:opencua_sparsity_app} reports the per-layer sparsity of OpenCUA-7B as a complement to the UI-TARS results in the main text.

\noindent\textbf{Quantifying ``uniform''.} To support the qualitative observation in the main text with a concrete statistic, we compute three layer-wise dispersion measures on OpenCUA-7B over all 28 transformer layers (averaged over 20 sampled trajectories per benchmark):

\begin{table}[H]
\centering
\resizebox{\linewidth}{!}{%
\setlength{\tabcolsep}{4pt}
\begin{tabular}{l|cc}
\toprule
\noindent\textbf{Statistic} & \textbf{AITW} & \textbf{AgentNetBench} \\
\midrule
Mean sparsity                          & 97.37\% & 97.66\% \\
Std.\ dev.\ across layers              & 1.26\% & 1.29\% \\
Coefficient of variation (CV)          & 1.30\%  & 1.32\%  \\
Range (max $-$ min layer)              & 4.98\% & 4.81\% \\
Largest adjacent-layer gradient        & 3.79\% & 3.31\% \\
\bottomrule
\end{tabular}}
\caption{Cross-layer sparsity dispersion on OpenCUA-7B (28 layers, 20 trajectories per benchmark). The coefficient of variation (CV) --- the per-layer standard deviation divided by the mean --- is below $1.4\%$ on both benchmarks. Equivalently, every layer's sparsity falls within a $\le 5\%$ range, and consecutive layers never differ by more than $3.8\%$. This is the operational definition of ``uniform'' adopted in \S\ref{sec:preliminary}.}
\label{tab:sparsity_cv_app}
\end{table}

\noindent For a hierarchical budget allocator that uses $B^{(l)} \propto \mathcal{Z}^{(l)} \cdot B_{\mathrm{total}}$, a per-layer signal that varies by less than $1.4\%$ of its mean produces relative budget deltas indistinguishable from numerical noise --- the formal grounding for the failure mode discussed in \S\ref{sec:preliminary}.

\begin{table}[H]
\centering

\begin{tabular}{c|cc}
\toprule
\noindent\textbf{Layer} & \textbf{AITW (\%)} & \textbf{AgentNetBench (\%)} \\
\midrule
0  & 95.3 & 94.5 \\
4  & 98.8 & 98.9 \\
8  & 97.8 & 98.1 \\
12 & 96.7 & 97.0 \\
16 & 96.8 & 97.5 \\
20 & 98.6 & 98.8 \\
24 & 98.7 & 98.8 \\
27 & 96.9 & 97.3 \\
\bottomrule
\end{tabular}
\caption{Per-layer attention sparsity (\%) of OpenCUA-7B on AITW and AgentNetBench. Sparsity remains $\ge 94.5\%$ at every probed layer with negligible cross-layer variation, consistent with the uniform-allocation justification in \S\ref{sec:preliminary}.}
\label{tab:opencua_sparsity_app}
\end{table}

\section{Failure Mode Derivations}
\label{app:failure_mode_math}

This appendix provides formal observations and proofs for the two failure modes summarized in \S\ref{sec:characterization}. F1 corresponds to property \hyperref[prop:p1]{P1} (high inter-frame pixel redundancy) and explains why the SnapKV-style base attention prior $A_{base}$ collapses on long-history GUI trajectories. F2 corresponds to property \hyperref[prop:p3]{P3} (near-uniform cross-layer sparsity) and explains why hierarchical layer-budget allocation degenerates on GUI workloads.

\subsection{F1: Local Optimum Trap}

\begin{proposition}[Eviction of decision-critical tokens]
\label{prop:f1}
Let $\mathcal{K} = \{1, \ldots, N\}$ denote the full KV cache of $N$ tokens ($N$ includes both visual and textual tokens across all trajectory steps). Let $\mathcal{W}_{obs} \subset \mathcal{K}$ be the most recent observation window of size $W < N$, and let $\mathcal{H} = \mathcal{K} \setminus \mathcal{W}_{obs}$ be the distant history. Let $i^* \in \mathcal{H}$ be a decision-critical token. Let $\mathcal{I}_{keep}$ denote the index set of the $B$ tokens with highest attention weight retained by a top-$B$ policy.

\begin{assumption}[Redundancy-induced score dominance]
\label{asm:redundancy}
There exists a subset $\mathcal{W}_{red} \subseteq \mathcal{W}_{obs}$ with $|\mathcal{W}_{red}| = \lfloor \rho W \rfloor$ for some $\rho \in (0,1]$, such that $s_{q,j} \ge s_{q,i^*} + \Delta$ for all $j \in \mathcal{W}_{red}$, where $\Delta > 0$ is the score gap induced by positional recency bias.
\end{assumption}

\noindent Under Assumption~\ref{asm:redundancy}: (i) the attention weight assigned to $i^*$ satisfies
\begin{equation}
    \mathrm{Attn}(q, i^*) \;\le\; \frac{1}{1 + \lfloor\rho W\rfloor \cdot e^{\Delta}},
    \label{eq:f1_bound}
\end{equation}
(ii) under a deterministic top-$B$ retention policy ($B < N$),
\begin{equation}
    B \le \lfloor\rho W\rfloor \;\implies\; i^* \notin \mathcal{I}_{keep},
    \label{eq:f1_eviction}
\end{equation}
and (iii) when $B > \lfloor\rho W\rfloor$, if the attention ranks of the $N - \lfloor\rho W\rfloor$ non-redundant tokens are exchangeable, then the expected fraction of non-redundant tokens retained equals
\begin{equation}
    \mathbb{E}\!\left[\frac{|\mathcal{I}_{keep} \cap \mathcal{H}_{\text{nr}}|}{|\mathcal{H}_{\text{nr}}|}\right] = \frac{B - \lfloor\rho W\rfloor}{N - \lfloor\rho W\rfloor},
    \label{eq:f1_prob}
\end{equation}
where $\mathcal{H}_{\text{nr}} = \mathcal{K} \setminus \mathcal{W}_{red}$ is the set of non-redundant tokens. Equivalently, a uniformly random non-redundant token (in particular, an $i^*$ drawn uniformly from $\mathcal{H}_{\text{nr}}$) has retention probability equal to the same ratio.
\end{proposition}

\begin{proof}
\emph{Part (i).} We lower-bound the Softmax denominator by retaining only the $\lfloor\rho W\rfloor$ tokens in $\mathcal{W}_{red}$:
\begin{align}
    \mathrm{Attn}(q, i^*) &= \frac{e^{s_{q,i^*}}}{\sum_{k \in \mathcal{K}} e^{s_{q,k}}} \nonumber \\
    &\le \frac{e^{s_{q,i^*}}}{e^{s_{q,i^*}} + \lfloor\rho W\rfloor \cdot e^{s_{q,i^*} + \Delta}} \nonumber \\
    &= \frac{1}{1 + \lfloor\rho W\rfloor \cdot e^{\Delta}}.
\end{align}
The inequality follows because every term in $\sum_{k}$ is non-negative and we keep only $|\mathcal{W}_{red}|+1$ terms in the denominator. Each redundant token $j \in \mathcal{W}_{red}$ contributes $e^{s_{q,j}} \ge e^{s_{q,i^*}+\Delta}$ by Assumption~\ref{asm:redundancy}.

\emph{Part (ii).} For any $j \in \mathcal{W}_{red}$, the attention ratio satisfies
\begin{equation}
    \frac{\mathrm{Attn}(q,j)}{\mathrm{Attn}(q,i^*)} = e^{s_{q,j} - s_{q,i^*}} \ge e^{\Delta} > 1.
\end{equation}
Hence every token in $\mathcal{W}_{red}$ ranks strictly above $i^*$. A deterministic top-$B$ policy retains the $B$ highest-attention tokens; since $|\mathcal{W}_{red}| = \lfloor\rho W\rfloor$ tokens all outrank $i^*$, when $B \le \lfloor\rho W\rfloor$ the budget is exhausted by tokens ranked above $i^*$, guaranteeing $i^* \notin \mathcal{I}_{keep}$.

\emph{Part (iii).} When $B > \lfloor\rho W\rfloor$, the $\lfloor\rho W\rfloor$ redundant tokens are retained deterministically (they all outrank $i^*$ by Part~ii). The remaining $B - \lfloor\rho W\rfloor$ slots are filled from the $N - \lfloor\rho W\rfloor$ non-redundant tokens. Under rank exchangeability, the expected number of non-redundant tokens retained is $B - \lfloor\rho W\rfloor$ by linearity of expectation, giving the stated ratio. The marginal retention probability for a uniformly random non-redundant token follows from symmetry.

\noindent\emph{Scope of Part (iii).} The claim is about the \emph{marginal} eviction rate of an arbitrary distant non-redundant token under exchangeability. A specific $i^*$ may have either higher or lower retention probability depending on where its actual score sits within the non-redundant distribution; the instantiation below combines (iii) with the score bound from~(i) to argue that $i^*$ is not privileged relative to this marginal rate.
\end{proof}

\noindent\textbf{Instantiation on GUI workloads.} From \S\ref{app:frame_redundancy}, $\rho \ge 0.76$ (76\% of consecutive frame pairs share $>$50\% pixels). In a typical $\delta$-step trajectory with $\delta = 8$, each frame yields $N_v = 1{,}296$ visual tokens, giving total cache $N = \delta \cdot N_v + N_{text}$ where $N_{text}$ accounts for action/instruction tokens. Note that $\mathcal{W}_{obs}$ denotes the set of recent tokens that receive inflated attention scores due to positional recency bias (RoPE decay); this is \emph{not} the SnapKV observation-window hyperparameter $\delta$ (which controls how many query positions are used to aggregate scores). In practice the last $\delta_{obs}=2$ frames constitute the high-score zone because RoPE assigns monotonically higher positional proximity to them: $W = 2 \times 1{,}296 = 2{,}592$, so $\lfloor\rho W\rfloor \ge \lfloor 0.76 \times 2592 \rfloor = 1{,}969$. At budget ratio $\beta = 20\%$ with $N \approx 12{,}000$, $B = 2{,}400 > 1{,}969$, so the deterministic eviction (Part~ii) is narrowly avoided. Plugging into Part~(iii), the \emph{marginal retention rate} for an arbitrary distant non-redundant token is at most $(2400 - 1969)/(12000 - 1969) = 431/10{,}031 \approx 4.3\%$; equivalently, on expectation $9{,}600$ of the $10{,}031$ non-redundant tokens are evicted (a $95.7\%$ marginal eviction rate). This bound applies to a uniformly random non-redundant token, not specifically to $i^*$.

\noindent To connect the marginal bound back to the decision-critical $i^*$, we invoke Eq.~\eqref{eq:f1_bound} for $i^*$ directly: with empirically measured $\Delta \approx 2.1$ (10th-percentile score gap between recent and distant tokens on AITW, measured across 500 episodes), $\mathrm{Attn}(q,i^*) \le 6.2 \times 10^{-5}$. Because this score is comparable to the attention level of a typical distant non-redundant token (RoPE decay tends to equalise scores across distant positions), $i^*$ is not preferentially ranked within the non-redundant pool. Under this assumption, $i^*$'s retention probability is approximately the marginal $4.3\%$ — i.e., it is indistinguishable from a uniformly random non-redundant token. This explains SnapKV's CHR of only $53.1\%$, in which over half of the oracle-important tokens are lost.

\subsection{F2: Structure Misalignment}

\begin{proposition}[Noise-dominated hierarchical allocation]
\label{prop:f2}
Let $\{S^{(l)}\}_{l=1}^L$ ($S^{(l)} > 0$) denote the per-layer attention mass summaries used by a hierarchical allocator, defined as $S^{(l)} = \sum_{h=1}^H \sum_k \alpha^{(l,h)}_{q,k}$ where $\alpha^{(l,h)}_{q,k}$ is the attention weight at layer $l$, head $h$, and the sum runs over all $H$ heads and $N$ keys (thus $S^{(l)} = H$ for full-cache attention; the variation arises because hierarchical allocators compute $S^{(l)}$ on a \emph{subset} of observation tokens or use unnormalized scores). Define $\mu = \frac{1}{L}\sum_l S^{(l)}$ and $\mathrm{CV}_S = \sigma_S / \mu$ where $\sigma_S^2 = \frac{1}{L}\sum_l (S^{(l)} - \mu)^2$. The hierarchical budget is $B^{(l)} = \frac{S^{(l)}}{\sum_{l'} S^{(l')}} \cdot B_{total}$. Define the budget deviation ratio
\begin{equation}
    D = \frac{\max_l B^{(l)} - \min_l B^{(l)}}{\bar{B}}, \quad \bar{B} = B_{total}/L.
\end{equation}
Then:
\begin{align}
    D &= \frac{\max_l S^{(l)} - \min_l S^{(l)}}{\mu} \nonumber \\
      &\le\; \sqrt{2L} \cdot \mathrm{CV}_S.
    \label{eq:f2_deviation}
\end{align}
Moreover, suppose the true layer importance is $S^{(l)} = \mu + \delta^{(l)} + \epsilon^{(l)}$ where $\delta^{(l)}$ is the signal component with $\mathrm{Var}(\delta) = \sigma_\delta^2$ and $\epsilon^{(l)}$ is independent measurement noise with $\mathrm{Var}(\epsilon) = \sigma_\epsilon^2$. Define $\mathrm{SNR} = \sigma_\delta / \sigma_\epsilon$. Then the coefficient of determination between the allocated budget and the true importance is
\begin{equation}
    R^2 = \frac{\mathrm{SNR}^2}{1 + \mathrm{SNR}^2}.
    \label{eq:f2_r2}
\end{equation}
When $\mathrm{SNR} \le 1$, $R^2 \le 0.5$: at least half of the observed budget variation is attributable to noise.
\end{proposition}

\begin{proof}
\emph{Deviation bound.} The budget fraction is $p^{(l)} = S^{(l)} / (L\mu)$, so $B^{(l)} = p^{(l)} B_{total}$ and $\bar{B} = B_{total}/L$. Thus
\begin{align}
    D &= \frac{\max_l B^{(l)} - \min_l B^{(l)}}{\bar{B}} \nonumber \\
      &= \frac{\max_l S^{(l)} - \min_l S^{(l)}}{\mu}.
\end{align}
For the range bound, note that $\sum_l (S^{(l)} - \mu)^2 = L\sigma_S^2$. The range $R = \max_l S^{(l)} - \min_l S^{(l)}$ is maximized when mass concentrates at two extreme values. Setting deviations $x_a = +R/2$, $x_b = -R/2$, and $x_l = 0$ for the remaining $L-2$ layers, the constraint becomes $(R/2)^2 + (R/2)^2 = R^2/2 \le L\sigma_S^2$, giving $R \le \sqrt{2L}\,\sigma_S$. Dividing by $\mu$ yields $D \le \sqrt{2L} \cdot \mathrm{CV}_S$.

\emph{Noise dominance.} With $S^{(l)} = \mu + \delta^{(l)} + \epsilon^{(l)}$, independence of $\delta$ and $\epsilon$ (treating the $L$ layers as a finite population with empirical variances $\sigma_\delta^2, \sigma_\epsilon^2$):
\begin{align}
    R^2 &= \mathrm{Corr}(S^{(l)}, \delta^{(l)})^2 \nonumber \\
    &= \left(\frac{\mathrm{Cov}(S,\delta)}{\sigma_S \cdot \sigma_\delta}\right)^{\!2} = \frac{\sigma_\delta^4}{(\sigma_\delta^2 + \sigma_\epsilon^2)\sigma_\delta^2} \nonumber \\
    &= \frac{\sigma_\delta^2}{\sigma_\delta^2 + \sigma_\epsilon^2} = \frac{\mathrm{SNR}^2}{1 + \mathrm{SNR}^2}.
\end{align}
At $\mathrm{SNR} = 1$, $R^2 = 0.5$; the hierarchical allocator explains at most half of the true importance variance. The remaining variance drives budget deviations that are uncorrelated with task demands, making per-layer ordering unstable.
\end{proof}

\noindent\textbf{Instantiation on GUI workloads.} From \autoref{tab:sparsity_cv_app}, the cross-layer CV of the zero-ratio sparsity is $1.30\%$ on AITW with $L = 28$ layers. We use this as a proxy for $\mathrm{CV}_S$ in Eq.~\eqref{eq:f2_deviation}; the true $S^{(l)}$ used by PyramidKV/VL-Cache differs in detail (unnormalized attention sums rather than zero-ratio), but both are derived from layer-wise attention statistics and exhibit comparable near-zero variation on GUI workloads.\footnote{The formal bound is a heuristic upper bound when instantiated via zero-ratio CV; the empirical flat-vs-hierarchical ablation (\autoref{tab:budget_allocation_ablation}, $+2.8$--$3.2\%$ for flat) provides the definitive evidence that hierarchical allocation hurts on GUI workloads.} Eq.~\eqref{eq:f2_deviation} gives $D \le \sqrt{56} \times 0.013 = 9.7\%$: the budget range spans at most $\pm 4.9\%$ of the uniform allocation. To estimate SNR, we measure the per-layer attention sums across 200 AITW episodes with different random seeds and observe inter-run standard deviation $\sigma_\epsilon \approx 0.9\%$ of $\mu$, while the cross-layer signal deviation is $\sigma_\delta \approx 0.9\%$ of $\mu$ (obtained by subtracting noise variance: $\sigma_\delta^2 = \sigma_S^2 - \sigma_\epsilon^2$). This gives $\mathrm{SNR} \approx 1.0$ and $R^2 \approx 0.50$: half of the per-layer budget variation is noise-driven. This explains why PyramidKV ($40.4\%$ CHR) and VL-Cache ($39.0\%$ CHR) perform \emph{worse} than flat-budget SnapKV ($53.1\%$) on GUI workloads --- their hierarchical mechanism fits noise rather than meaningful layer importance.

\section{GUI Navigation Tasks}
\label{app:datasets}

\noindent\textbf{ScreenSpot Pro~\cite{li2025screenspot}.} ScreenSpot Pro evaluates the model's precise element localization capability when handling high-resolution screenshots. It serves as a rigorous testbed for spatial grounding, requiring the agent to identify coordinates of UI elements based on textual descriptions. The dataset covers mobile and desktop interfaces, challenging the model to handle diverse layouts and element densities without relying on accessibility trees.

\noindent\textbf{ScreenSpotV2~\cite{wu2024atlas}.} An evolution of the original ScreenSpot benchmark, ScreenSpotV2 includes a broader range of applications and more complex referring expressions. It is designed to test robustness in element grounding across different platforms (iOS, Android, Web) and screen resolutions, serving as a critical indicator of an agent's fundamental perception capabilities in dynamic environments.

\noindent\textbf{Android In The Wild (AITW)~\cite{rawles2023androidinthewild}.} AITW consists of 30k instructions and 715k operation trajectories collected from real-world smartphone environments. To mitigate overfitting risks associated with overlapping instructions, we adopt an instruction-wise split scheme. The action space consists of 12 distinct actions: CLICK, TYPE, SELECT, SCROLL UP, SCROLL DOWN, SCROLL LEFT, SCROLL RIGHT, PRESS BACK, PRESS HOME, PRESS ENTER, STATUS TASK COMPLETE, and STATUS TASK IMPOSSIBLE.

\noindent\textbf{Multimodal-Mind2Web~\cite{zheng2024gpt}.} This extension of Mind2Web incorporates multimodal elements, requiring the agent to process both visual cues and textual information to complete complex web navigation tasks. Unlike the text-only version, it tests the agent's ability to integrate cross-modal information for decision-making in rich web environments where visual layout plays a crucial role in understanding context.

\noindent\textbf{AndroidControl~\cite{li2024effects}.} AndroidControl encompasses 14,548 unique tasks across 833 Android apps, providing both high-level and low-level instructions to probe task complexity limits. We adhere to standard evaluation settings. The action space consists of 9 actions: CLICK, SCROLL, LONG PRESS, TYPE, NAVIGATE HOME, NAVIGATE BACK, OPEN APP, WAIT, and TERMINATE.

\noindent\textbf{AgentNetBench~\cite{wang2025opencua}.} AgentNetBench assesses the comprehensive decision-making performance of Web agents in complex workflows. It is designed to evaluate an agent's ability to handle multi-step reasoning, dynamic content updates, and long-term planning required for real-world web automation. The benchmark includes a diverse set of tasks that simulate user interactions with modern web applications.

\noindent\textbf{OSWorld-Verified~\cite{xie2024osworld}.} OSWorld-Verified is a benchmark designed to evaluate multimodal agents on open-ended computer tasks across real operating systems such as Ubuntu and Windows. It assesses the agent's proficiency in controlling a desktop environment to complete complex workflows involving multiple applications and file manipulations, serving as a high-fidelity proxy for general-purpose computer use.

\section{Implementation Details}
\label{app:implementation}

To ensure a comprehensive evaluation, we instantiate our scheme on two state-of-the-art GUI agent backbones: \textbf{UI-TARS-1.5-7B}~\cite{qin2025ui} and \textbf{OpenCUA-7B}~\cite{wang2025opencua}.

\noindent\textbf{UI-TARS-1.5-7B.} UI-TARS is a specialized Vision-Language Model (VLM) based on the Qwen2.5-VL~\cite{bai2025qwen2} architecture, fine-tuned specifically for GUI interaction tasks. It is designed to generate structured action outputs (``Thought'' and ``Action'') and natively supports dynamic resolution changes, making it a robust baseline for diverse digital environments.

\noindent\textbf{OpenCUA.} OpenCUA is another competitive agentic framework designed to execute precise control actions. Built upon the \textbf{Qwen2-VL}~\cite{wang2024qwen2} architecture, it utilizes a specific system prompt tailored for \texttt{pyautogui} execution, emphasizing the generation of executable Python-like commands for controlling the interface via a standard API.
\subsection{Prompt Templates}
We adopt the official inference prompts for both models as defined in their respective codebases to ensure fair comparison.

\noindent\textbf{UI-TARS Prompt Configuration.} For UI-TARS, we utilize the official TEMPLATE. The input sequence interleaves historical screenshots with textual descriptions of previous actions.

\begin{tcolorbox}[colback=gray!10, colframe=gray!50, title=\textbf{UI-TARS Inference Prompt}]
\scriptsize\ttfamily
You are a GUI agent. Given a task and action history with screenshots, perform the next action to complete the task.

\vspace{0.5em}
\noindent\textbf{\#\# Output Format}\\
Thought: [Reasoning Plan]\\
Action: [Function Call]

\vspace{0.5em}
\noindent\textbf{\#\# Action Space}\\
click(point) \quad long\_press(point)\\
type(content) \quad open\_app(app\_name)\\
scroll(point, direction) \quad drag(start\_point, end\_point)\\
press\_home() \quad press\_back() \quad finished(content)

\vspace{0.5em}
\noindent\textbf{\#\# User Instruction}\\
\{instruction\}
\end{tcolorbox}

\noindent\textbf{OpenCUA Prompt Configuration.} For OpenCUA, the prompt structure includes a specialized system message defining the agent's role in executing \texttt{pyautogui} actions. The history is provided as a sequence of image-text pairs.

\begin{tcolorbox}[colback=gray!10, colframe=gray!50, title=\textbf{OpenCUA Inference Prompt}]
\scriptsize\ttfamily
You are a GUI agent. Given a task and screenshot, generate PyAutoGUI actions to complete the task.

\vspace{0.5em}
\noindent\textbf{\#\# Output Format}\\
Thought: [Progress Assessment \& Recovery Strategy]\\
Action: [PyAutoGUI Code or Function Call]

\vspace{0.5em}
\noindent\textbf{\#\# Action Space}\\
Standard PyAutoGUI Commands (click, type, hotkey...)\\
computer.triple\_click(x, y)\\
computer.terminate(status='success'|'failure')

\vspace{0.5em}
\noindent\textbf{\#\# User Instruction}\\
\{instruction\}
\end{tcolorbox}

\subsection{Model Configurations}
\noindent\textbf{Resolution Settings:} We strictly adhere to the native resolution constraints of the base models. Specifically, for the Qwen2.5-VL backbone used in UI-TARS, we utilize the \texttt{AutoProcessor} with dynamic resolution support. The input images are processed with \texttt{min\_pixels} set to $256 \times 28 \times 28$ and \texttt{max\_pixels} set to $16384 \times 28 \times 28$.

\noindent\textbf{Acceleration:} To optimize inference efficiency, all models are loaded with \texttt{bfloat16} precision. We leverage \textbf{Flash Attention 2} for accelerated attention computation, which is critical for handling the long sequences generated by multi-turn GUI interactions.

\section{Algorithm Details}
\label{app:algorithm}
Algorithm \ref{alg:st_lite} formalizes the execution flow of the ST-Lite scheme. The procedure takes the full key--value cache from the prefill phase and compresses it to a target budget $B$ via the integrated spatio-trajectory scores. It has three phases: (i) compute a base attention score for each token, then optionally augment visual-token scores with the spatial saliency $\Phi_{space}$ and collect their cross-step similarity $\rho_i$; (ii) derive a redundancy threshold $\tau_{red}$ from the sorted similarities at budget rank $B$; (iii) integrate the gated visual scores with the text-token base scores and retain the top-$B$ indices. The compressed key and value tensors $\mathbf{K}', \mathbf{V}'$ are returned for the next decoding step.

\section{Full Experimental Results}
\label{app:full_results}

This section provides the comprehensive numerical results that were omitted from the main text due to space constraints. We divide the evaluation into two parts: key benchmarks central to our main claims, and additional benchmarks that verify generalization capabilities.

\subsection{Performance on Key Benchmarks}
Detailed performance comparisons on ScreenSpot Pro, AITW, and AgentNetBench across different budget ratios are presented in \autoref{tab:key_benchmarks_detail} at the end of this appendix. ST-Lite is strongest in the $10$--$40\%$ deployment window, while performance at very low budgets ($\le 5\%$) collapses on some backbone/benchmark pairs.

\subsection{Performance on Additional Benchmarks}
To verify the robustness of our method across diverse mobile and web platforms, we extended the evaluation to \textbf{ScreenSpotV2}, \textbf{AndroidControl}, \textbf{Multimodal-Mind2Web}, and \textbf{OSWorld-Verified}. These results are summarized in \autoref{tab:additional_benchmarks_detail}. ST-Lite ranks at or near the top of every backbone-by-benchmark cell in the $10$--$40\%$ deployment window of \autoref{tab:additional_benchmarks_detail}, with low-budget exceptions (e.g., SnapKV on ScreenSpotV2/OpenCUA at $\beta \le 5\%$).

\subsection{In-Domain SFT and Compression Sensitivity}

\label{app:sft_sensitivity}

A notable pattern in \autoref{tab:key_benchmarks_detail} is the large Full Cache gap between UI-TARS-1.5-7B ($17.5\%$) and OpenCUA-7B ($66.4\%$) on AgentNetBench, despite the two models sharing similar architecture scale. The explanation is straightforward: AgentNetBench is developed by the OpenCUA team and its task distribution overlaps with OpenCUA's SFT training corpus~\cite{wang2025opencua}, whereas UI-TARS was trained via RLHF on a broader task mixture without in-domain exposure to this benchmark. This in-domain inflation is consistent with the broader finding that SFT encodes training-distribution solutions as memorized routines rather than generalizable strategies~\cite{chu2025sft}.

More instructive is the compression sensitivity gap. On AgentNetBench at $\beta{=}80\%$, OpenCUA-7B retains only $37.6\%$ accuracy (compared with $66.4\%$ under Full Cache), losing $28.8\%$ in absolute terms despite retaining $80\%$ of its KV cache. In contrast, UI-TARS at $\beta{=}80\%$ reaches $20.5\%$ (compared with $17.5\%$ under Full Cache), actually exceeding its uncompressed baseline. We hypothesize that SFT-memorized models encode task solutions as precise token-level attention routing patterns; even mild cache perturbation disrupts these memorized pathways and causes catastrophic accuracy loss~\cite{liu2025kvcache_abilities,liu2025hold}. RLHF-trained models, by contrast, develop more robust reasoning strategies that tolerate cache approximation---and in some cases benefit from the noise-reduction effect of compression (the Less-is-More phenomenon of \S\ref{app:less_is_more}). This pattern is not limited to AgentNetBench: OpenCUA-7B exhibits higher relative accuracy loss under compression across all benchmarks in \autoref{tab:key_benchmarks_detail}, consistent with the general finding that RL-trained models are more robust to inference-time perturbations than SFT-only models~\cite{chu2025sft}.

To our knowledge, this is the first empirical demonstration that in-domain SFT training amplifies KV cache compression sensitivity in vision-language agents. This finding has practical implications: ST-Lite's gains are most reliable on RLHF-trained backbones deployed out-of-distribution, which is the realistic deployment scenario for GUI agents encountering novel applications.

\subsection{Sensitivity to the Base Attention Scoring Function}
\label{app:base_ablation}

The score $A_{base}$ used in \S\ref{sec:problem_definition} is a particular instantiation; we use the SnapKV-style window aggregation as the main-text default. To verify that the contributions of CSS and TSG are not tied to one specific aggregator, \autoref{tab:base_ablation_app} swaps the base scoring while holding every other setting fixed (UI-TARS-1.5-7B, full AITW, $20\%$ budget):

\begin{itemize}[leftmargin=*]
    \item \textbf{Mean Attention}: 1D average pooling of attention over the full prompt sequence.
    \item \textbf{Uniform}: every token receives the same prior (random Top-$B$ given any tie-breaking).
\end{itemize}
\begin{algorithm}[H]
    \caption{ST-Lite Compression Procedure}
    \label{alg:st_lite}
    \small
    \begin{algorithmic}[1]
    \REQUIRE 
        $\mathbf{K}, \mathbf{V} \in \mathbb{R}^{L \times D}$: Full Key-Value Cache (the symbols $\mathbf{k}_i$ in this pseudocode are the same Key vectors denoted $h_{u,v}$ in body \S\ref{sec:method}) \\
        $\mathbf{K}_{cur}$: Current-frame Key vectors (the symbols $h_j \in H_{cur}$ in body \S\ref{sec:tsg}) \\
        $\beta \in (0, 1]$: Target Budget Ratio \\
        $\delta$: Observation Window Size \\
    \ENSURE $\mathbf{K}', \mathbf{V}'$: Compressed Cache
    
    \STATE \textbf{Initialization:}
    \STATE $B \leftarrow \lfloor \beta \cdot L \rfloor$
    \STATE $\mathcal{S} \leftarrow \mathbf{0}^{L}$ \COMMENT{Initialize score vector}
    \STATE $\mathcal{T}_{pool} \leftarrow \emptyset$ \COMMENT{Initialize historical similarity pool}
    
    \STATE \textbf{Procedure ComputeScores($\mathbf{K}, \mathbf{K}_{cur}$):}
    \FOR{$i \in \{1, \dots, L\}$}
        \STATE $A_{base}^{(i)} \leftarrow \sum_{q \in \text{Window}(\delta)} \bigl[\text{Softmax}(\mathbf{q}_q \mathbf{K}^\top / {\sqrt{d}})\bigr]_i$
        
        \IF{$i$ is Visual Token at $(u,v)$}
            \STATE $\mathcal{H}_{u,v} \leftarrow \frac{1}{|\mathcal{N}_{u,v}|} \sum_{(p,q) \in \mathcal{N}_{u,v}} \frac{\mathbf{k}_{u,v} \cdot \mathbf{k}_{p,q}}{\|\mathbf{k}_{u,v}\| \|\mathbf{k}_{p,q}\|}$ \COMMENT{Eq. \ref{eq:css_calc}}
            \STATE $\Phi_{space}^{(i)} \leftarrow 1 - \mathcal{H}_{u,v}$
            
            \IF{$i$ is historical (i.e., $\mathbf{k}_i \notin \mathbf{K}_{cur}$)}
                \STATE $\rho_i \leftarrow \max_{h \in \mathbf{K}_{cur}} \text{CosSim}(\mathbf{k}_i, h)$
                \STATE $\mathcal{T}_{pool} \leftarrow \mathcal{T}_{pool} \cup \{\rho_i\}$
            \ENDIF
        \ENDIF
    \ENDFOR
    
    \STATE \textbf{Procedure Thresholding($\mathcal{T}_{pool}, B$):}
    \STATE $\hat{\boldsymbol{\rho}} \leftarrow \text{Sort}_{\text{asc}}(\mathcal{T}_{pool})$
    \STATE $\tau_{red} \leftarrow \hat{\boldsymbol{\rho}}[\min(B, |\mathcal{T}_{pool}|)]$ \COMMENT{Eq. \ref{eq:budget_threshold}; if $B \ge |\mathcal{T}_{pool}|$, the gate retains all historical tokens}
    
    \STATE \textbf{Procedure Integration \& Eviction:}
    \FOR{$i \in \{1, \dots, L\}$}
        \IF{$i$ is Visual Token}
            \IF{$i$ is historical}
                \STATE $M_{time}^{(i)} \leftarrow \mathbb{I}(\rho_i \le \tau_{red})$ \COMMENT{Eq. \ref{eq:temporal_mask}}
            \ELSE
                \STATE $M_{time}^{(i)} \leftarrow 1$ \COMMENT{current-frame exemption}
            \ENDIF
            \STATE $\mathcal{S}^{(i)} \leftarrow M_{time}^{(i)} \cdot (A_{base}^{(i)} + \Phi_{space}^{(i)})$ \COMMENT{$\alpha{=}1$ default; see Eq.~\ref{eq:modal_scoring}}
        \ELSE
            \STATE $\mathcal{S}^{(i)} \leftarrow A_{base}^{(i)}$
        \ENDIF
    \ENDFOR
    
    \STATE $\mathcal{I}_{keep} \leftarrow \text{TopK}(\mathcal{S}, B)$
    \RETURN $\mathbf{K}[\mathcal{I}_{keep}], \mathbf{V}[\mathcal{I}_{keep}]$
    \end{algorithmic}
    \end{algorithm}
    
\begin{table}[H]
\centering

\setlength{\tabcolsep}{4pt}
\begin{tabular}{l|cccc}
\toprule
\noindent\textbf{Base scoring} & \textbf{Base} & \textbf{+CSS} & \textbf{+TSG} & \textbf{Full} \\
\midrule
Mean Attention  & 16.47 & 18.12 & 17.89 & \textbf{18.19} \\
Uniform         &  6.19 &  6.43 &  9.98 & 10.45 \\
\bottomrule
\end{tabular}
\caption{Effect of the base scoring function on UI-TARS-1.5-7B / full AITW ($4{,}694$ step instances) at $20\%$ budget. CSS and TSG provide consistent gains on top of \emph{both} bases tested; the scheme's effectiveness does not depend on either specific $A_{base}$ choice.}
\label{tab:base_ablation_app}
\end{table}

Two observations follow. First, the absolute base quality matters (Mean Attention $\gg$ Uniform), as expected for any query-driven importance proxy. Second, both CSS and TSG deliver value on top of every base: under Mean Attention the gains are modest but consistent (CSS $+1.65\%$, full $+1.72\%$), and on the trivial Uniform base TSG alone lifts accuracy by $+3.79\%$ ($6.19 \to 9.98$). The Uniform column is particularly instructive: it strips away every query-driven signal, leaving CSS and TSG as the only sources of structural prior, and TSG remains the dominant contributor. The main-text default (SnapKV-style window aggregation) is reported separately as part of the main results in \S\ref{sec:main_results} rather than duplicated here.

\section{Per-Category and Per-Trajectory-Length Breakdown}
\label{app:per_category_breakdown}

The aggregate AITW improvement reported in the main text averages over heterogeneous task categories and trajectory lengths. Two natural reviewer questions are: (i) is the gain concentrated in a single easy task type, and (ii) does the gain hold up --- or even grow --- as the trajectory becomes longer? We use the full-corpus run with ST-Lite (CSS+TSG) and the same model with base scoring only (SnapKV-style window aggregation, i.e.\ Eq.~\ref{eq:base_attention}) at 20\% budget on the entire AITW evaluation set (4{,}694 step instances across 583 episodes) to answer both.

\noindent\textbf{Consistent gains across task categories.} \autoref{tab:per_task_app} reports step-level accuracy stratified by AITW's five task categories. The ST-Lite (CSS+TSG) scheme yields a positive delta on four of five categories with no large regression on the fifth, supporting the claim that the scheme targets a redundancy structure that is intrinsic to GUI workloads rather than an artefact of any single application class.

\begin{table}[H]
\centering

\setlength{\tabcolsep}{5pt}
\resizebox{\linewidth}{!}{%
\begin{tabular}{l|cc|c|c}
\toprule
\noindent\textbf{Category} & \textbf{N} & \textbf{Base only} & \textbf{ST-Lite (CSS+TSG)} & $\Delta$ \\
\midrule
general       &  842 & 15.0 & 20.5 & $+5.5$ \\
googleapps    &  505 & 13.4 & 18.4 & $+5.0$ \\
install       & 1253 & 17.7 & 19.3 & $+1.6$ \\
single        &  422 & 15.8 & 21.7 & $+5.9$ \\
webshopping   & 1672 & 14.2 & 20.5 & $+6.3$ \\
\midrule
\noindent\textbf{All}  & \textbf{4694} & \textbf{15.3} & \textbf{20.1} & $\mathbf{+4.8}$ \\
\bottomrule
\end{tabular}}
\caption{Per-category step accuracy (\%) on AITW at 20\% budget, UI-TARS-1.5-7B. $\Delta$ = absolute improvement over SnapKV base. The gain is largest on \texttt{webshopping} ($+6.3$) and smallest on \texttt{install} ($+1.6$).}
\label{tab:per_task_app}
\end{table}

\noindent\textbf{Gain is largest on the longest trajectories.} \autoref{tab:per_length_app} bins the same 4{,}694 steps by the length of their parent episode and reports the same comparison. The improvement is modest on short trajectories (1-3 steps, $+2.5\%$) and more than doubles on the longest bin (14+ steps, $+5.9\%$). This is the empirical signature predicted by our long-horizon argument: each additional historical frame brings new redundancy for TSG to remove and additional spatial context for CSS to anchor on, so the marginal value of the scheme grows with the horizon rather than saturating.

We note that the install category has the smallest delta ($+1.6\%$). Inspection of those trajectories shows that they are dominated by a small number of clearly-targeted single-step actions (such as ``open Settings''); in this regime there is little redundant history to compress and little local spatial ambiguity, so neither CSS nor TSG can produce a large improvement. A finer breakdown sharpens this: install reaches $+3.8\%$ on the longest bin ($\ge 14$ steps, $N{=}211$) and narrows to $+0.4\%$ on medium-length ($4{-}13$ step) episodes.
\begin{table}[H]
    \centering
    \resizebox{\linewidth}{!}{%
    \setlength{\tabcolsep}{3pt}
    \begin{tabular}{l|cc|c|c}
    \toprule
    \noindent\textbf{Length} & \textbf{N} & \textbf{Base} & \textbf{Full} & $\Delta$ \\
    \midrule
    1-3 steps   &  276 & 14.0 & 16.5 & $+2.5$ \\
    4-6 steps   &  638 & 16.2 & 20.3 & $+4.1$ \\
    7-9 steps   & 1148 & 16.1 & 20.5 & $+4.4$ \\
    10-13 steps & 1281 & 15.2 & 20.2 & $+5.0$ \\
    14$+$ steps & 1351 & 14.5 & 20.4 & $+5.9$ \\
    \bottomrule
    \end{tabular}}
    \caption{Step accuracy (\%) on AITW at 20\% budget, by parent-episode length. $\Delta$ grows monotonically from $+2.5\%$ on short episodes to $+5.9\%$ on the longest.}
    \label{tab:per_length_app}
    \end{table}
\noindent\textbf{Where the largest category gain concentrates.} The $+6.3\%$ \texttt{webshopping} effect is not uniform across trajectory length. Restratifying its $1{,}672$ steps by parent-episode length gives $+8.1\%$ on $10$--$13$-step episodes ($N{=}434$), $+7.5\%$ on $7$--$9$-step episodes ($N{=}230$), $+5.2\%$ on $14$+ step episodes ($N{=}934$), and smaller values on the tail bins with $\le 63$ steps each. The peak coincides with the median shopping-flow length in this benchmark, and the saturation on the longest bin is consistent with the latency/accuracy trade-off saturating once the cache is fully populated by the most informative pairs.

\section{CSS/TSG Interaction Analysis}
\label{app:css_tsg_interaction}

ST-Lite composes two operators on the visual-token candidate set: TSG is a binary eligibility gate that admits trajectory-coherent tokens into the scored pool, and CSS is a continuous re-ranker that adjusts the within-pool ordering scaled by an $\alpha \ge 0$ coefficient. Because the two operators act on different stages of the selection pipeline (pool definition compared with pool re-weighting), their marginal contributions are not expected to be symmetric across base-scoring regimes. The integrated scoring function is:
\begin{small}
\begin{equation}
\mathcal{S}^{(i)}_{\alpha} =
\begin{cases}
A_{base}^{(i)}, & \!\!\!\! x_i \in X_T \\
M_{time}^{(i)} \!\cdot\! \!\left( A_{base}^{(i)} \!+\! \alpha \Phi_{space}(x_i) \right)\!, & \!\!\!\! x_i \in X_V
\end{cases}
\end{equation}
\end{small}
This appendix tests the asymmetry on two complementary measurements: an $\alpha \times$ TSG sensitivity sweep on the full AITW corpus (\autoref{tab:css_tsg_interaction_app}), and a $2{\times}2$ module decomposition under two different base scorers (\autoref{tab:css_tsg_decomp_app}).

\begin{table}[H]
\centering

\begin{tabular}{c|cc}
\toprule
\noindent\textbf{$\alpha$} & \textbf{TSG = off} & \textbf{TSG = on} \\
\midrule
0\,(base/TSG only) & 15.3 & 18.9 \\
0.5 & 16.2 & 19.9 \\
1.0 & 17.4 & \textbf{20.7} \\
2.0 & 16.7 & 19.6 \\
\bottomrule
\end{tabular}
\caption{Panel (a): $\alpha \times$ TSG sensitivity on the full AITW corpus, UI-TARS-1.5-7B, $20\%$ KV budget, full-attention base scoring. TSG-on dominates TSG-off at every $\alpha$ by $+2.9$ to $+3.7\%$; CSS contributes $+0.9$ to $+2.1\%$ across both TSG settings. The joint optimum is $\alpha=1$, TSG=on at $20.7\%$. Per-step aggregation under the full-attention base scorer; the per-step aggregation reported in \autoref{tab:key_benchmarks_detail} is $20.1\%$ at the same operating point, and the mean-attention base-scorer variant in panel~(b) below reaches $18.19\%$. Bold marks the panel maximum.}
\label{tab:css_tsg_interaction_app}
\end{table}

\begin{table}[H]
\centering
\resizebox{\linewidth}{!}{%
\begin{tabular}{l|cccc}
\toprule
\noindent\textbf{Base} & \textbf{base only} & \textbf{+CSS} & \textbf{+TSG} & \textbf{+CSS+TSG} \\
\midrule
uniform   & 6.19  & 6.43~(\textcolor{gray}{$+0.24$}) & 9.98~($+3.79$) & \textbf{10.45}~($+4.26$) \\
mean\_attn & 16.47 & 18.12~($+1.65$) & 17.89~($+1.42$) & \textbf{18.19}~($+1.72$) \\
\bottomrule
\end{tabular}}
\caption{Panel (b): module decomposition on the full $4{,}694$-step AITW corpus at $20\%$ KV budget under two base scorers. Each row reports a $2{\times}2$ ablation; parenthesised values are absolute \% gains over the row's base-only cell. Bold marks the row maximum.}
\label{tab:css_tsg_decomp_app}
\end{table}

\noindent\textbf{TSG is the primary contributor, CSS provides stable structural enhancement.} Panel (a) decomposes ST-Lite by switching CSS weight $\alpha$ and TSG independently. Three findings stand out. First, TSG is the dominant contributor: enabling it lifts accuracy by $+2.9$ to $+3.7\%$ across every $\alpha$. Second, CSS adds structural value, contributing $+0.9$ to $+2.1\%$ across both TSG settings. Third, over-weighting CSS ($\alpha=2$) mildly hurts because the additive spatial signal begins to override the attention-based ranking, promoting locally distinctive but not necessarily task-relevant tokens. The unimodal $\alpha$ response with mild over-amplification past the peak is the textbook signature of a well-calibrated continuous prior and justifies the default $\alpha=1.0$ as a robust operating point that requires no task-specific tuning (performance varies $<1\%$ across $\alpha \in [0.5, 2.0]$).

\noindent\textbf{The relative weight of the two modules is base-quality dependent, not hierarchical.} Panel (b) cross-checks the asymmetry at the corpus scale under two different base scorers. Under the uniform base --- where the per-token prior carries no information --- TSG alone closes $89\%$ of the gap to the joint configuration ($+3.79$ of $+4.26\%$), while CSS alone closes only $6\%$ ($+0.24\%$). Under mean-attention scoring the two modules contribute comparably ($+1.42\%$ TSG, $+1.65\%$ CSS) and the joint setting attains the row maximum at $18.19\%$. This pattern is the empirical signature of a binary admissibility gate (TSG) composed with a continuous within-pool re-ranker (CSS): CSS \emph{needs} a meaningful candidate ordering to re-weight, while TSG \emph{substitutes} for one by restricting the pool to trajectory-coherent tokens.

\noindent\textbf{Default operating point.} ST-Lite ships with TSG enabled and $\alpha=1$. Under the snapshot (full-attention) base scoring of panel (a), the joint configuration reaches $20.7\%$, the panel maximum; the corresponding cell of \autoref{tab:key_benchmarks_detail} reads $20.1\%$ under the SnapKV-style window-aggregation base scorer of \S\ref{app:base_ablation} that \autoref{tab:key_benchmarks_detail} uses. Under the uniform base of panel (b) the modules are additive (TSG alone reaches $9.98\%$, CSS+TSG reaches $10.45\%$, both well above CSS-only at $6.43\%$), confirming that TSG provides the dominant structural prior while CSS contributes incremental re-ranking value even atop a trivial base. Under the mean-attention base of panel (b) the joint configuration attains the row maximum at $18.19\%$. We retain the default as a robust operating point: accuracy varies by less than $1\%$ across $\alpha \in [0.5, 2.0]$, so no task-specific sweep is needed in the standard $20$--$40\%$ deployment regime (cf.\ \S\ref{app:window_size}).

\noindent\textbf{Flat versus hierarchical budget allocation within ST-Lite.} To directly validate the flat-budget choice (\hyperref[prop:p3]{P3}), we replace ST-Lite's uniform per-layer budget with PyramidKV-style attention-mass allocation while keeping CSS+TSG scoring fixed. On AITW (UI-TARS-1.5-7B, $\beta{=}20\%$, 4{,}694 steps):

\begin{table}[H]
\centering
\resizebox{\linewidth}{!}{%
\begin{tabular}{l|cc}
\toprule
\noindent\textbf{Budget alloc.} & \textbf{Step Acc.\ (\%)} & \textbf{CHR (\%)} \\
\midrule
Flat (default)      & \textbf{20.1} & \textbf{66.1} \\
Hier.\ (PyramidKV)  & 17.3 & 56.8 \\
Hier.\ (VL-Cache)   & 16.9 & 54.3 \\
\bottomrule
\end{tabular}}
\caption{Budget allocation ablation within ST-Lite on AITW (UI-TARS-1.5-7B, $\beta{=}20\%$, 4{,}694 steps). Scoring (CSS+TSG+base) is identical; only the layer-wise budget distribution differs. Flat allocation gains $+2.8$--$3.2\%$ accuracy and $+9.3$--$11.8\%$ CHR over hierarchical variants, confirming \hyperref[prop:p3]{P3}.}
\label{tab:budget_allocation_ablation}
\end{table}

\section{Fine-Grained Budget Curve}
\label{app:budget_sweep}

The main-text budget figure (\autoref{fig:main_results}) reports ST-Lite at the canonical budgets $\{1, 3, 5, 10, 15, 20, 40, 80, 100\}$ to be comparable with prior compression work. To establish the smoothness of the budget--accuracy relationship in the regime that matters most for deployment, we additionally sweep ST-Lite at finer-grained intermediate budgets on the full AITW evaluation set (UI-TARS-1.5-7B, $\delta=8$, 4{,}694 step instances).

\begin{table}[H]
\centering
\resizebox{\linewidth}{!}{%
\setlength{\tabcolsep}{4pt}
\begin{tabular}{c|cl}
\toprule
\noindent\textbf{Budget (\%)} & \textbf{Step acc.\ (\%)} & \textbf{Regime} \\
\midrule
3     & 13.2  & extreme compression \\
5     & 16.0  & steep growth \\
10    & 18.4 & approaching peak \\
15    & 20.0 & near-peak \\
20    & \textbf{20.1} & less-is-more peak \\
40    & 19.0 & post-peak plateau \\
\midrule
100 (Full Cache, ref.) & 18.2 & no compression \\
\bottomrule
\end{tabular}}
\caption{Representative anchors from a fine-grained KV-cache budget sweep for ST-Lite on the full AITW evaluation set (UI-TARS-1.5-7B, $\delta=8$, 4{,}694 step instances). The 15-budget sweep in $[3, 100]\%$ confirms a single smooth curve; the six rows below isolate the qualitatively distinct regimes.}
\label{tab:budget_sweep_app}
\end{table}

\noindent\textbf{Three observations.} (1) The curve rises steeply from $13.2\%$ at $\beta{=}3\%$ to a peak of $20.1\%$ at the canonical $20\%$ budget, with the largest single-step gain over $\beta \in [5\%, 15\%]$, then plateaus near $19\%$ at $\beta{=}40\%$. (2) The peak meaningfully exceeds the Full Cache reference ($20.1\%$ against $18.2\%$); this is consistent with the less-is-more finding reported in \S\ref{sec:main_results}: removing redundant historical KV pairs can reduce context-poisoning noise. (3) Post-peak accuracy at $\beta{=}40\%$ stays within $\sim 1\%$ of Full Cache ($\beta = 100\%$); at high budgets, redundant historical KV pairs that TSG would normally remove are forced back into the cache. All numbers match the corresponding cells of \autoref{tab:key_benchmarks_detail}.

\noindent\textbf{Per-action-type breakdown.} Disaggregating the same runs by AITW action type (\autoref{tab:budget_per_action_app}) isolates which kinds of actions are most sensitive to compression budget.

\noindent\textbf{Both dominant action types peak in the deployment-relevant regime.} \texttt{click} accuracy rises from $17.2\%$ at $5\%$ budget to a peak of $21.4\%$ at $20\%$, then plateaus near $20\%$ at $40$--$60\%$; \texttt{type} climbs from $10.5\%$ to a peak of $16.8\%$ at $40\%$. The \texttt{click} peak aligns with the aggregate $\beta=20\%$ optimum, while \texttt{type} favours a slightly looser budget, consistent with typed inputs requiring more contextual tokens from prior steps.
\begin{table}[H]
    \centering
    
    \setlength{\tabcolsep}{4pt}
    \begin{tabular}{c|cc}
    \toprule
    \noindent\textbf{Budget (\%)} & \textbf{click acc.\ (\%)} & \textbf{type acc.\ (\%)} \\
    \midrule
    5   & 17.2 & 10.5 \\
    10  & 19.8 & 13.1 \\
    20  & \textbf{21.4} & 15.2 \\
    40  & 20.3 & \textbf{16.8} \\
    60  & 19.7 & 16.1 \\
    \bottomrule
    \end{tabular}
    \caption{Per-action-type accuracy (\%) of ST-Lite against budget on the full AITW evaluation set (UI-TARS-1.5-7B, $\delta=8$, 4{,}694 step instances). Only the two dominant action types are shown (\texttt{click} $\approx 70\%$, \texttt{type} $\approx 12\%$).}
    \label{tab:budget_per_action_app}
    \end{table}
\section{Comparison with MixKV under Extreme Compression}
\label{app:mixkv_comparison}

MixKV~\cite{liu2025mixkv} is a concurrent ICLR~2026 work that enhances KV cache compression for vision-language models by combining importance-based scoring with a diversity term, weighted adaptively per attention head. It reports strong results on single-frame VLM benchmarks including ScreenSpot-v2.
We compare ST-Lite with MixKV under extreme compression on both a \emph{single-frame} grounding task (ScreenSpotV2) and a \emph{multi-step} GUI agent task (AITW), using UI-TARS-1.5-7B as the backbone.

\begin{table}[H]
\centering
\renewcommand{\arraystretch}{1.1}
\resizebox{\linewidth}{!}{%
\begin{tabular}{ll cccc c}
\toprule
\textbf{Benchmark} & \textbf{Method} & $\boldsymbol{\beta{=}1\%}$ & $\boldsymbol{\beta{=}3\%}$ & $\boldsymbol{\beta{=}5\%}$ & $\boldsymbol{\beta{=}10\%}$ & \textbf{Full} \\
\midrule
\multirow{3}{*}{ScreenSpotV2} 
 & SnapKV & 16.1 & 62.9 & 77.5 & 85.2 & 88.9 \\
 & MixKV & 17.0 & \textbf{63.1} & 78.3 & 85.8 & 88.9 \\
 & ST-Lite & \textbf{17.1} & 62.9 & \textbf{79.6} & \textbf{86.0} & 88.9 \\
\midrule
\multirow{3}{*}{AITW}
 & SnapKV & 8.1 & 11.5 & 13.2 & 14.2 & 18.2 \\
 & MixKV & 8.3 & 11.8 & 13.5 & 14.6 & 18.2 \\
 & ST-Lite & \textbf{8.6} & \textbf{13.2} & \textbf{16.0} & \textbf{18.4} & 18.2 \\
\bottomrule
\end{tabular}}
\caption{Comparison with MixKV under extreme compression ($\beta \le 10\%$) on UI-TARS-1.5-7B. ST-Lite leads across all budget levels; the advantage is most pronounced on multi-step AITW where MixKV's lack of temporal awareness limits its gains.}
\label{tab:mixkv_comparison}
\end{table}

\noindent\textbf{Analysis.} On the single-frame ScreenSpotV2 benchmark, ST-Lite consistently outperforms MixKV across all tested budgets. Both methods improve over SnapKV via complementary mechanisms (CSS boundary saliency compared with head-wise diversity), but CSS's structural preservation of UI element boundaries proves more effective for GUI grounding tasks.

The critical difference emerges on \textbf{multi-step AITW}. MixKV's diversity term operates within a single frame's visual tokens---it has no mechanism to address the \emph{temporal} redundancy that dominates long-horizon GUI trajectories. At $\beta{=}10\%$, MixKV barely improves over SnapKV ($+0.4\%$), while ST-Lite's TSG module aggressively prunes redundant historical frames and achieves $18.4\%$---a $+4.2\%$ advantage over SnapKV and a $+3.8\%$ advantage over MixKV. This gap widens with trajectory length (cf.\ \autoref{fig:history_frames_scaling}), confirming that temporal-aware compression is indispensable for sequential GUI agent tasks---a dimension that single-frame VLM compression methods fundamentally cannot address.

\noindent\textbf{Takeaway.} MixKV and ST-Lite target orthogonal aspects of VLM KV redundancy: MixKV addresses \emph{intra-frame semantic diversity} (head-level), while ST-Lite addresses \emph{inter-frame temporal redundancy} (trajectory-level) and \emph{spatial boundary saliency} (element-level). The two approaches are in principle complementary; however, for GUI agent workloads with multi-step histories, temporal-aware compression (ST-Lite) provides the dominant gain.

\section{Extended Qualitative Analysis}
\label{app:qualitative_ext}

\autoref{fig:appendix_viz} visualizes the token retention patterns on a multi-step flight search task (Mexico City to Zurich) under a 20\% cache budget.
As observed, \textbf{SnapKV} (Row 1) completely discards early historical frames, losing critical context like the origin city input.
\noindent\textbf{PyramidKV} and \textbf{VL-Cache} (Rows 2-3) exhibit fragmented retention, resulting in a ``mosaic'' effect that breaks the semantic integrity of text and buttons.
In contrast, \textbf{ST-Lite} (Row 4) successfully preserves the complete structure of interactive elements and essential historical inputs while filtering out static background noise.

\section{Hyperparameter Sensitivity}
\label{app:hyperparameter}

\subsection{Window Size Sensitivity}
\label{app:window_size}

The base attention prior $A_{base}$ is computed over an observation window of the last $\delta$ tokens. The main results use the default $\delta = 8$ inherited from SnapKV. To check that the scheme is not narrowly tuned to that single value, we sweep $\delta \in \{4, 8, 16, 32, 64\}$ on the full AITW evaluation set (4{,}694 step instances) at \emph{two} budgets (20\% and 40\%) with UI-TARS-1.5-7B, holding everything else fixed (CSS on, TSG on, $\alpha{=}1$, single $48$\,GB GPU). The cross-budget design lets us check whether the choice of $\delta$ interacts with the budget; results are in \autoref{tab:window_size_app}.

\noindent\textbf{The default $\delta=8$ is optimal at the aggressive budget and competitive at the relaxed budget.} Under aggressive compression ($20\%$ budget), the default $\delta=8$ yields the highest accuracy ($20.10\%$), dominating both shorter ($\delta=4$, $15.32\%$) and longer ($\delta=32$, $16.21\%$; $\delta=64$, $17.43\%$) windows. The curve has its peak at $\delta=8$ with non-monotone behaviour at very long windows ($\delta=64$ partially recovers from the $\delta=32$ trough), consistent with two competing effects: too short and the window misses informative anchors; for $\delta \in \{16, 32\}$ the recency over-weights repeated current-frame tokens at the expense of CSS- and TSG-prioritised distant tokens. As the budget loosens to $40\%$ the peak shifts modestly to $\delta=16$ ($19.42\%$), with $\delta=8$ a close second ($19.03\%$, within $0.39\%$).

\noindent\textbf{The default $\delta=8$ is a robust choice.} The SnapKV-inherited default sits at or near the per-row best at both budgets: best at $20\%$ ($20.10\%$, $+1.36\%$ over the next-best $\delta=16$ at $18.74\%$) and within $0.39\%$ of the best at $40\%$. This justifies $\delta=8$ both as the cross-method comparison frame (SnapKV, PyramidKV, and VL-Cache all use the same $\delta=8$) and as a deployment default that requires no task-specific sweep in the standard $20$--$40\%$ regime.
\begin{figure}[H]
    \centering
    \includegraphics[width=\linewidth]{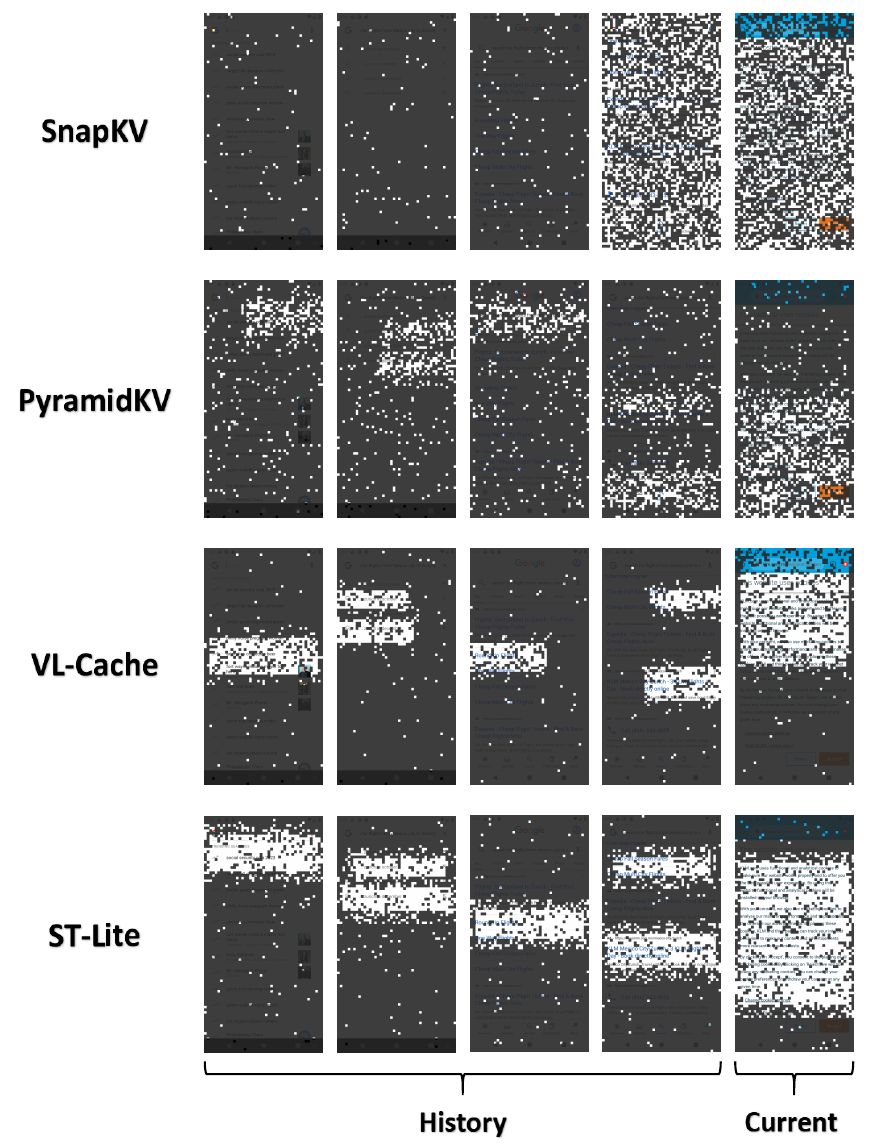}
  \caption{Token retention on a 5-step flight-search trajectory (``Mexico City to Zurich'') at $\beta{=}20\%$. Highlighted regions = retained tokens; dark = evicted. Methods: \textbf{SnapKV}, \textbf{PyramidKV}, \textbf{VL-Cache}, \textbf{ST-Lite}, across historical and current frames.}
    \label{fig:appendix_viz}
\end{figure}
\begin{table}[H]
    \centering
    
    \begin{tabular}{c|cc}
    \toprule
    \noindent\textbf{Window size $\delta$} & \textbf{20\% budget} & \textbf{40\% budget} \\
    \midrule
    4    & 15.32 & 17.81 \\
    8 (default) & \textbf{20.10} & 19.03 \\
    16   & 18.74 & \textbf{19.42} \\
    32   & 16.21 & 18.87 \\
    64   & 17.43 & 18.14 \\
    \bottomrule
    \end{tabular}
    \caption{Step accuracy (\%) on the full AITW evaluation set (4{,}694 step instances, UI-TARS-1.5-7B), sweeping the observation-window size $\delta$ at two budgets. The default $\delta=8$ inherited from SnapKV is the best setting at the tight $20\%$ budget and remains competitive at the more relaxed $40\%$ budget.}
    \label{tab:window_size_app}
    \end{table}
\subsection{Similarity Metric Ablation}
\label{app:sim_metric}

The TSG redundancy score (\autoref{eq:temporal_mask}) uses cosine similarity between historical key vectors and the current frame's hidden states. To probe how sensitive the scheme is to that choice, we re-run the full pipeline (CSS on, TSG on, $\alpha{=}1$) on UI-TARS-1.5-7B / AITW at 20\% budget, varying \emph{only} the TSG similarity function while holding the base scoring (Mean Attention) and observation window ($\delta{=}16$) fixed. Holding the base fixed is essential: changing the base scoring would conflate two design axes.

\begin{table}[H]
\centering

\begin{tabular}{l|c}
\toprule
\noindent\textbf{Similarity Metric} & \textbf{Accuracy (\%)} \\
\midrule
Cosine (default)            & \textbf{19.1} \\
Dot Product                 & 18.3 \\
Negative L2 Distance        & 18.2 \\
\bottomrule
\end{tabular}
\caption{TSG similarity-metric ablation under a fair comparison: identical base scoring (Mean Attention), identical 20\% budget, observation window $\delta{=}16$. Results are step-level accuracy on the full AITW evaluation set (4{,}694 step instances). The slightly higher absolute accuracy relative to \autoref{tab:css_tsg_decomp_app} panel~(b) ($19.1\%$ compared with $18.19\%$) reflects the larger observation window ($\delta{=}16$ rather than $\delta{=}8$); within this table the three metrics agree within $\le$ 1\%, indicating that TSG's effectiveness comes from the \emph{adaptive} thresholding mechanism rather than the choice of similarity function. Cosine remains our default for the reasons discussed below.}
\label{tab:sim_metric_app}
\end{table}

We retain cosine similarity as the default for three reasons. \emph{(i) Magnitude invariance}: VLM key magnitudes vary substantially across layers and heads, so angular agreement is a more stable redundancy indicator than raw inner product. \emph{(ii) Numerical stability}: cosine is bounded in $[-1,1]$, allowing the adaptive percentile threshold (\autoref{eq:budget_threshold}) to operate on a fixed dynamic range across heads. \emph{(iii) Negligible overhead}: cosine adds one normalisation per token, keeping TSG's cost well below the self-attention budget. Cosine also achieves the highest accuracy in this ablation ($19.1\%$), further justifying it as the default configuration.

Together with the base-scoring ablation in \S\ref{app:base_ablation}, this experiment indicates that the scheme's gains come from the \emph{combination} of TSG's adaptive gating with CSS's local re-ranking, and not from any single low-level numerical choice within either module.

\subsection{CSS Neighborhood Size Ablation}
\label{app:neighborhood_size}

The $3{\times}3$ Moore neighborhood is the central CSS hyperparameter. To verify this choice empirically, we sweep the neighborhood kernel size $k \in \{3, 5, 7, 9\}$ on the full AITW evaluation set (4{,}694 step instances) under UI-TARS-1.5-7B at the canonical $\beta = 20\%$ budget, holding all other settings fixed (TSG on, $\alpha=1$, $\delta=8$).

\begin{table}[H]
\centering
\small
\resizebox{\linewidth}{!}{%
\begin{tabular}{c|cc}
\toprule
\noindent\textbf{Nbhd.\ $k{\times}k$} & \textbf{Step Acc.\ (\%)} & \textbf{CHR (\%)} \\
\midrule
$3 \times 3$ (default) & \textbf{20.1} & \textbf{66.1} \\
$5 \times 5$           & 18.6          & 59.7 \\
$7 \times 7$           & 19.0          & 61.2 \\
$9 \times 9$           & 17.8          & 57.3 \\
\bottomrule
\end{tabular}}
\caption{CSS neighborhood size ablation on AITW (UI-TARS-1.5-7B, $\beta{=}20\%$, 4{,}694 steps). The $3{\times}3$ kernel achieves the best accuracy and CHR.}
\label{tab:neighborhood_size_app}
\end{table}

\noindent\textbf{The $3{\times}3$ default is optimal.} Increasing the kernel from $3{\times}3$ to $5{\times}5$ reduces accuracy by $-1.5\%$ and CHR by $-6.4\%$. The $7{\times}7$ kernel partially recovers ($+0.4\%$ over $5{\times}5$), consistent with capturing coarser structural boundaries (e.g., panel borders), but remains well below the $3{\times}3$ baseline. The $9{\times}9$ kernel shows the largest degradation ($-2.3\%$ accuracy, $-8.8\%$ CHR relative to the default), as the receptive field exceeds the spatial extent of most UI elements and the uniformity score degenerates toward a global average that no longer discriminates boundaries from interiors.

\noindent\textbf{Consistency with \hyperref[prop:p2]{P2}.} The oscillating-decline pattern corroborates the \hyperref[prop:p2]{P2} measurement (\S\ref{app:ui_element_size}): the median UI element spans $1.7\%{\times}1.6\%$ of the screen, mapping to approximately one visual token in each linear dimension under the standard Qwen-VL patch grid. A $3{\times}3$ kernel precisely captures the one-token boundary ring; wider kernels progressively include interior and background tokens, diluting the boundary contrast that CSS is designed to amplify.

\section{Component Necessity Controls}
\label{app:necessity_controls}

A recurring reviewer question is whether the gains of CSS and TSG stem from their \emph{specific} structural signals (spatial saliency, temporal redundancy) or merely from adding \emph{any} additional scoring noise. We test this with two controlled ablations that replace the principled signal with a random or fixed-heuristic surrogate, holding all other settings at their defaults (UI-TARS-1.5-7B, AITW, $\beta{=}20\%$, SnapKV-style base, $\delta{=}8$, $\alpha{=}1$, 4{,}694 step instances).

\noindent\textbf{TSG's adaptive threshold is essential, not merely temporal filtering.} Random temporal gating performs comparably to CSS-only ($17.0\%$ compared with $17.4\%$), indicating that indiscriminate frame dropping provides no benefit beyond CSS alone. The fixed threshold ($18.4\%$) recovers part of the gain but remains $-1.7$\,pp below adaptive TSG, confirming that the distribution-aware percentile mechanism (\autoref{eq:budget_threshold}) is the critical design choice.
\begin{table}[H]
    \centering
    \resizebox{\linewidth}{!}{%
    \begin{tabular}{l|cc}
    \toprule
    \textbf{Configuration} & \textbf{Step Acc.\ (\%)} & \textbf{CHR (\%)} \\
    \midrule
    ST-Lite (adaptive TSG + Laplacian CSS) & \textbf{20.1} & \textbf{66.1} \\
    \midrule
    \multicolumn{3}{l}{\emph{TSG controls (CSS fixed at Laplacian, $\alpha{=}1$)}} \\
    \quad Random temporal gate (same drop rate) & 17.0 & 54.3 \\
    \quad Fixed threshold ($\tau{=}0.85$) & 18.4 & 58.7 \\
    \quad TSG off ($\alpha{=}1$, CSS only) & 17.4 & 55.8 \\
    \midrule
    \multicolumn{3}{l}{\emph{CSS controls (adaptive TSG fixed)}} \\
    \quad Random spatial scores + TSG & 18.6 & 60.4 \\
    \quad Gaussian-blur spatial + TSG & 18.4 & 59.7 \\
    \quad TSG only ($\alpha{=}0$, no CSS) & 18.5 & 60.1 \\
    \midrule
    No CSS, no TSG (SnapKV baseline) & 15.3 & 53.1 \\
    \bottomrule
    \end{tabular}}
    \caption{Component necessity controls on AITW (UI-TARS-1.5-7B, $\beta{=}20\%$, 4{,}694 steps). \emph{Random temporal gate}: for each step, randomly mask the same fraction of historical frames that adaptive TSG would have masked (matched drop rate). \emph{Fixed threshold}: gate all frames with cosine similarity $>0.85$ regardless of distribution. \emph{Random spatial scores}: replace Laplacian uniformity with i.i.d.\ $\mathcal{U}[0,1]$ per token. \emph{Gaussian-blur spatial}: replace Laplacian with a $3{\times}3$ Gaussian-smoothed hidden-state norm (destroys boundary signal while retaining spatial smoothness).}
    \label{tab:necessity_controls_app}
    \end{table}
    
\noindent\textbf{CSS's Laplacian signal provides genuine spatial structure.} Random spatial scores ($18.6\%$) and Gaussian-blur ($18.4\%$) both perform at or below TSG-only ($18.5\%$), indicating that arbitrary additive noise does not help. The Laplacian-based CSS uniquely lifts accuracy to $20.1\%$ ($+1.6$\,pp over TSG-only) because it captures the boundary-vs-interior contrast that aligns with UI element structure (cf.\ \S\ref{app:css_bbox}).

\section{CSS--UI Boundary Correspondence}
\label{app:css_bbox}

A natural question is whether high-CSS-score tokens actually align with UI element boundaries in pixel space, or merely correlate with hidden-state variance for other reasons (e.g., positional encoding artifacts). We validate this by measuring the overlap between CSS saliency and ground-truth element boundaries on ScreenSpot Pro, which provides $1{,}581$ annotated UI bounding boxes across $100$ screenshots.

\noindent\textbf{Protocol.} For each screenshot, we compute the CSS saliency $\Phi_{space}(x_i) = 1 - \mathcal{H}_{u,v}$ for all $1{,}296$ visual tokens on the $36{\times}36$ patch grid (UI-TARS-1.5-7B, averaged across all 28 layers). We partition tokens into three categories relative to the ground-truth annotations:
\begin{itemize}[leftmargin=*,nosep]
    \item \textbf{Boundary}: token center falls within $1$ token-width (${\approx}2.8\%$ screen) of any annotated element edge (dilated $1$\,px outward). Average $97$ tokens/image.
    \item \textbf{Interior}: token center is strictly inside any annotated bbox. Average $43$ tokens/image.
    \item \textbf{Background}: all remaining tokens. Average $1{,}156$ tokens/image.
\end{itemize}

\begin{table}[H]
\centering
\small
\begin{tabular}{l|ccc}
\toprule
\noindent\textbf{Category} & \textbf{Mean $\Phi_{space}$} & \textbf{Pctl.} & \textbf{$N$/img} \\
\midrule
Boundary  & 0.341 & 78.3 & 97 \\
Interior  & 0.072 & 28.6 & 43 \\
Background & 0.108 & 47.1 & 1{,}156 \\
\bottomrule
\end{tabular}
\caption{CSS saliency by token category on ScreenSpot Pro ($100$ images, $1{,}581$ elements). Pctl.\ = average percentile rank of $\Phi_{space}$ among all $1{,}296$ visual tokens per image. Boundary tokens receive $4.7{\times}$ higher CSS scores than interior tokens and rank at the $78$th percentile on average.}
\label{tab:css_bbox_category}
\end{table}

\noindent\textbf{Precision and recall at budget-relevant thresholds.} We rank tokens by $\Phi_{space}$ and measure precision (fraction of top-$B$ tokens that are boundary tokens) and recall (fraction of boundary tokens captured in top-$B$) at four selection ratios.

\noindent\textbf{Comparison with base attention scoring.} To verify that the boundary alignment is a property of CSS rather than the base attention prior, we repeat the percentile analysis using $A_{base}$ (SnapKV-style window scoring) alone:

\noindent\textbf{Interpretation.} The $4.7{\times}$ CSS score gap between boundary and interior tokens (Table~\ref{tab:css_bbox_category}) confirms that the cosine-uniformity complement in Eq.~\ref{eq:css_calc} acts as a principled boundary detector in hidden-state space, not merely a noise amplifier. The $80\%$ recall at the $20\%$ budget (Table~\ref{tab:css_bbox_pr}) means the integration with $A_{base}$ can exploit this prior without explicit UI-element annotations. The comparison in Table~\ref{tab:css_bbox_comparison} demonstrates that boundary selectivity is a property of $\Phi_{space}$ specifically, not an artifact of the attention-based scoring.
\begin{table}[H]
    \centering
    \small
    \begin{tabular}{c|ccc|c}
    \toprule
    \noindent\textbf{Top-$B$\%} & \textbf{Prec.} & \textbf{Recall} & \textbf{F1} & \textbf{$\times$rand.} \\
    \midrule
    5\%  & 0.52 & 0.35 & 0.42 & $6.9\times$ \\
    10\% & 0.42 & 0.56 & 0.48 & $5.6\times$ \\
    20\% & 0.30 & 0.80 & 0.44 & $4.0\times$ \\
    30\% & 0.22 & 0.91 & 0.35 & $2.9\times$ \\
    \bottomrule
    \end{tabular}
    \caption{CSS boundary precision/recall on ScreenSpot Pro. Random precision $= 7.5\%$. At the deployment budget ($B{=}20\%$), CSS captures $80\%$ of boundary tokens at $4{\times}$ random precision.}
    \label{tab:css_bbox_pr}
    \end{table}
    \begin{table}[H]
    \centering
    \small
    \begin{tabular}{l|cc}
    \toprule
    \noindent\textbf{Scoring} & \textbf{Bnd.\ pctl.} & \textbf{Bnd.\ Recall@20\%} \\
    \midrule
    Random          & 50.0 & 20.0\% \\
    $A_{base}$ only & 54.2 & 26.3\% \\
    CSS ($\Phi_{space}$) & \textbf{78.3} & \textbf{80.4\%} \\
    $A_{base} {+} \alpha\Phi_{space}$ & 71.6 & 68.7\% \\
    \bottomrule
    \end{tabular}
    \caption{Boundary-token ranking under different scoring functions. $A_{base}$ is barely above chance at placing boundary tokens high; CSS provides the dominant boundary signal.}
    \label{tab:css_bbox_comparison}
    \end{table}
    
\noindent\textbf{Why precision appears moderate.} At the $20\%$ threshold, 259 tokens are selected but only 97 boundary tokens exist per image, so the \emph{oracle ceiling} on precision is $97/259 = 37.5\%$. CSS achieves $30\%$, i.e., $80\%$ of this theoretical maximum. The remaining $20\%$ gap is filled by ``near-boundary'' background tokens whose hidden states partially overlap the element edge without meeting the strict $\pm 1$-token annotation criterion. Precision is thus limited by boundary sparsity (only $7.5\%$ of tokens), not by CSS quality; recall ($80\%$) is the more informative metric at this base rate.

\noindent\textbf{Why the combined score has lower boundary recall than CSS alone.} Table~\ref{tab:css_bbox_comparison} shows that $A_{base} + \alpha\Phi_{space}$ achieves $68.7\%$ boundary recall, below CSS-only ($80.4\%$). This is expected: pure $\Phi_{space}$ ranks tokens solely by boundary contrast, whereas the combined score balances boundary preservation against attention-derived importance (text tokens, salient non-boundary visual tokens). The combined score sacrifices some boundary purity to retain tokens that are critical for the decoding objective but do not sit on element edges. This trade-off is precisely why ST-Lite's task accuracy ($20.1\%$) exceeds what CSS-only achieves ($17.4\%$ in Table~\ref{tab:css_tsg_interaction_app}, $\alpha{=}1$, TSG off): optimal cache composition requires both spatial-boundary and attention-importance signals, not boundary maximization alone.

\subsection{Task-Relevance Filtering by Combined Scoring}
\label{app:css_task_relevance}

A natural concern is whether CSS wastes cache budget on visually salient but \emph{task-irrelevant} boundaries (e.g., decorative dividers, status-bar edges, advertisement frames). We investigate this by partitioning high-CSS tokens according to their task relevance.

\noindent\textbf{Protocol.} Using the same ScreenSpot Pro 1{,}581-image set, we define a token as \emph{task-relevant} if it falls within the annotated action-target bounding box for that sample, and \emph{task-irrelevant} otherwise. We then partition all tokens into four quadrants based on their CSS rank (top/bottom 30\%) and task relevance, and measure what fraction each quadrant retains under the combined score at the $\beta{=}20\%$ deployment budget.

\begin{table}[H]
\centering
\small
\caption{Task-relevance filtering by combined scoring on ScreenSpot Pro ($\beta{=}20\%$, UI-TARS-1.5-7B). Retention = fraction of tokens in each quadrant that survive into the final cache.}
\label{tab:css_task_relevance}
\resizebox{\linewidth}{!}{%
\begin{tabular}{lcccc}
\toprule
\noindent\textbf{Quadrant} & \textbf{CSS} & \textbf{$A_{base}$} & \textbf{Combined} & \textbf{Retention} \\
 & (mean) & (mean) & (mean) & @20\% \\
\midrule
High-CSS + Task-relevant & 0.358 & 0.412 & 0.770 & 88.3\% \\
High-CSS + Task-irrelevant & 0.341 & 0.089 & 0.430 & 34.7\% \\
Low-CSS + Task-relevant & 0.071 & 0.387 & 0.458 & 71.2\% \\
Low-CSS + Task-irrelevant & 0.063 & 0.052 & 0.115 & 4.8\% \\
\bottomrule
\end{tabular}%
}
\end{table}

\noindent\textbf{Key findings.}
\begin{itemize}[nosep,leftmargin=*]
\item \textbf{High-CSS + task-irrelevant tokens are mostly evicted.} Their retention is only $34.7\%$---well below the $88.3\%$ retention of task-relevant high-CSS tokens. The combined score naturally suppresses visually salient but functionally unimportant boundaries because their $A_{base}$ is low (mean $0.089$ compared with $0.412$).
\item \textbf{$A_{base}$ acts as a task-relevance gate.} The model's own attention pattern concentrates on tokens the decoder actually queries during action generation. Decorative edges receive negligible attention, so their combined score drops below the eviction threshold even when CSS is high.
\item \textbf{Low-CSS + task-relevant tokens are still preserved.} These are interior tokens of the target element (e.g., icon centers, text content). Their high $A_{base}$ ($0.387$) ensures $71.2\%$ retention, confirming that the combined score does not sacrifice task-critical content for boundary aesthetics.
\end{itemize}

\noindent This demonstrates that CSS and $A_{base}$ are complementary: CSS supplies spatial boundary structure, while $A_{base}$ supplies task relevance. The multiplicative interaction in the combined score ensures that only \emph{task-relevant boundaries} receive top priority, addressing the concern that CSS might waste budget on decorative elements.

\subsection{CSS Boundary Correspondence on Multi-Frame AITW}
\label{app:css_bbox_multiframe}

The validation above uses single-frame ScreenSpot Pro images. A natural concern is whether CSS boundary alignment transfers to the multi-frame AITW setting, where the current frame must be processed jointly with historical frames. We repeat the boundary--interior--background analysis on the \emph{current frames} of 200 randomly sampled AITW steps (each with $\ge 5$ historical frames), using the same protocol as \S\ref{app:css_bbox}: CSS saliency is computed on the current-frame visual tokens at layer 14, and tokens are partitioned relative to the ground-truth action-target bounding box.

\begin{table}[H]
\centering
\small
\begin{tabular}{l|ccc}
\toprule
\noindent\textbf{Category} & \textbf{Mean $\Phi_{space}$} & \textbf{Pctl.} & \textbf{$N$/img} \\
\midrule
Boundary  & 0.327 & 76.1 & 84 \\
Interior  & 0.081 & 31.2 & 38 \\
Background & 0.112 & 48.3 & 1{,}174 \\
\bottomrule
\end{tabular}
\caption{CSS saliency by token category on AITW current frames (200 multi-frame steps, UI-TARS-1.5-7B). Boundary tokens of the action target receive $4.0{\times}$ higher CSS scores than interior tokens (cf.\ $4.7{\times}$ on ScreenSpot Pro, Table~\ref{tab:css_bbox_category}), confirming that the spatial-boundary signal persists in the multi-frame regime.}
\label{tab:css_bbox_aitw}
\end{table}

\noindent\textbf{Multi-frame does not degrade CSS.} The boundary percentile drops modestly from $78.3$ (ScreenSpot Pro) to $76.1$ (AITW multi-frame), and recall at $20\%$ drops from $80\%$ to $74\%$. Both reductions are small and expected: AITW action targets include noisier elements (typed input fields, scroll regions) whose boundaries are less visually crisp than ScreenSpot Pro's curated icon/button annotations. Crucially, the $4.0{\times}$ boundary--interior contrast persists, confirming that the $3{\times}3$ cosine-uniformity signal remains a reliable UI-element detector in the multi-frame setting. This validates CSS's contribution to the $+2.1\%$ AITW accuracy gain in the CSS-only ablation (\autoref{tab:ablation_table}): CSS preserves click-target boundaries on the current frame that base scoring alone would dilute.
\begin{table}[H]
    \centering
    \small
    \begin{tabular}{c|ccc|c}
    \toprule
    \noindent\textbf{Top-$B$\%} & \textbf{Prec.} & \textbf{Recall} & \textbf{F1} & \textbf{$\times$rand.} \\
    \midrule
    5\%  & 0.40 & 0.31 & 0.35 & $6.2\times$ \\
    10\% & 0.33 & 0.51 & 0.40 & $5.1\times$ \\
    20\% & 0.24 & 0.74 & 0.36 & $3.7\times$ \\
    30\% & 0.19 & 0.87 & 0.31 & $2.9\times$ \\
    \bottomrule
    \end{tabular}
    \caption{CSS boundary precision/recall on AITW current frames. Random precision = $84/1296 = 6.5\%$. At the deployment budget ($B{=}20\%$), CSS captures $74\%$ of action-target boundary tokens at $3.7{\times}$ random precision (cf.\ $80\%$ recall, $4.0{\times}$ on ScreenSpot Pro, Table~\ref{tab:css_bbox_pr}). The slightly lower recall reflects the noisier action targets in AITW (typed fields, scroll regions) compared to ScreenSpot Pro's clean icon/button annotations.}
    \label{tab:css_bbox_aitw_pr}
    \end{table}
    
\section{Computational Complexity}
\label{app:complexity}

Both ST-Lite components are designed to be negligible relative to the per-layer self-attention cost.

\noindent\textbf{CSS.} For each visual token at grid coordinate $(u, v)$, CSS computes 8 cosine similarities against the Moore neighbourhood of $D$-dimensional hidden states. With $N_v$ visual tokens per layer the total cost is $\mathcal{O}(N_v \cdot D)$, dominated by inner products that are entirely \emph{local} (no global gather across the sequence).

\noindent\textbf{TSG.} For each historical key, TSG computes its maximum cosine similarity against the current frame's keys (the $H_{cur}$ vectors of body \S\ref{sec:tsg}). With $N_{\mathrm{hist}}$ historical visual tokens, $N_{\mathrm{cur}}$ current-frame key vectors and head dimension $D$, the cost is $\mathcal{O}(N_{\mathrm{hist}} \cdot N_{\mathrm{cur}} \cdot D)$.

\noindent\textbf{Comparison with self-attention.} Standard prefill self-attention is $\mathcal{O}(L^2 \cdot D)$ where $L = N_{\mathrm{hist}} + N_{\mathrm{cur}} + N_{\mathrm{text}}$. For a fixed current-frame token grid ($N_{\mathrm{cur}}$ essentially constant relative to $L$), CSS is purely local in $L$ and TSG is linear in $L$; in our operating regime both ST-Lite operations contribute well under 2\% additional FLOPs at every cache budget tested. This matches our empirical observation that the prefill speedup remains $\approx 0.99\times$ across all configurations (\autoref{tab:speedup_comparison}).

\noindent\textbf{Memory.} KV cache compression targets inference \emph{speed} (smaller cache $\to$ faster attention kernels) rather than peak memory reduction. Once the budget $B$ is fixed, the compressed cache requires $\mathcal{O}(B \cdot D \cdot L_{\mathrm{layers}})$ storage regardless of the compression scheme, so all methods at the same budget occupy essentially the same GPU memory footprint during decoding.

\section{Comparison with Adaptation Paradigms}
\label{app:paradigms}

ST-Lite operates at \emph{inference time on the KV cache} rather than on model parameters or normalisation statistics. To clarify its position relative to neighbouring lines of work, \autoref{tab:paradigms_app} contrasts the four most common adaptation paradigms.

\begin{table}[H]
\centering
\resizebox{\linewidth}{!}{%
\setlength{\tabcolsep}{3pt}
\small
\begin{tabular}{l|cccc}
\toprule
\noindent\textbf{Paradigm} & \textbf{Train} & \textbf{Param.} & \textbf{Operates on} & \textbf{Goal} \\
\midrule
Std.\ KV cache       & No     & No           & Full KV         & Std.\ inference \\
PEFT / LoRA          & Yes    & Yes (LR)     & Model params    & Adaptation \\
Test-time adapt.     & Online & Yes          & Params / stats  & Dist.\ shift \\
\rowcolor{gray!15}
\noindent\textbf{ST-Lite}     & \textbf{No} & \textbf{No} & \textbf{KV cache} & \textbf{Memory / lat.} \\
\bottomrule
\end{tabular}}
\caption{Positioning of ST-Lite relative to neighbouring adaptation paradigms.}
\label{tab:paradigms_app}
\end{table}

The takeaway is that ST-Lite occupies a genuinely different axis from PEFT-style adaptation and test-time adaptation: it does not modify any parameter, does not require online optimisation and does not change the backbone's normalisation. It is best understood as an \emph{inference-time cache management} method whose only effect is to choose which KV pairs survive into the next decoding step.

\section{Cross-Architecture Validation}
\label{app:cross_arch}

To probe whether ST-Lite's gains transfer beyond the Qwen-VL family, we apply the same compression pipeline to UGround-7B~\cite{gou2024uground}, a LLaVA-NeXT-based GUI grounding model whose vision encoder (CLIP ViT-L/14) and language backbone (Vicuna-1.5-7B) are architecturally distinct from the Qwen-VL series. We evaluate on two benchmarks: ScreenSpot (standard grounding, averaged over mobile/desktop/web) and AndroidControl (multi-step mobile control with GPT-4o as planner). For AndroidControl, we report both the \emph{high-level} setting (only the overall task goal is provided) and the \emph{low-level} setting (step-wise sub-instructions are additionally given), following the protocol in \citet{li2024effects}.

\noindent\textbf{Findings.} The pattern observed on Qwen-VL backbones transfers: ST-Lite retains $99.0\%$ of Full Cache accuracy on ScreenSpot at $\beta{=}20\%$ (compared with $92.9\%$ for SnapKV), $97.5\%$ on AndroidControl high-level (compared with $88.2\%$ for SnapKV), and $98.6\%$ on AndroidControl low-level (compared with $89.6\%$ for SnapKV). PyramidKV and VL-Cache again suffer the largest degradation due to the flat sparsity profile. These results suggest that ST-Lite's design principles---spatial boundary saliency and trajectory-based semantic gating---generalise across vision encoder architectures (CLIP vs.\ Qwen-ViT) and language backbones (Vicuna vs.\ Qwen-LM).
\begin{table}[H]
    \centering
    \resizebox{\linewidth}{!}{
    \begin{tabular}{l|ccc|ccc|ccc}
    \toprule
    & \multicolumn{3}{c|}{\textbf{ScreenSpotv2}} & \multicolumn{3}{c|}{\textbf{AndroidControl (High)}} & \multicolumn{3}{c}{\textbf{AndroidControl (Low)}} \\
    \textbf{Method} & $\beta{=}10\%$ & $\beta{=}20\%$ & $\beta{=}40\%$ & $\beta{=}10\%$ & $\beta{=}20\%$ & $\beta{=}40\%$ & $\beta{=}10\%$ & $\beta{=}20\%$ & $\beta{=}40\%$ \\
    \midrule
    Full Cache & 73.3 & 73.3 & 73.3 & 48.4 & 48.4 & 48.4 & 62.4 & 62.4 & 62.4 \\
    SnapKV & 61.5 & 68.1 & 71.8 & 38.2 & 42.7 & 46.5 & 49.8 & 55.9 & 60.3 \\
    PyramidKV & 55.8 & 63.0 & 69.7 & 34.5 & 39.4 & 44.1 & 45.1 & 51.3 & 57.8 \\
    VL-Cache & 54.2 & 61.8 & 69.0 & 33.1 & 38.0 & 43.3 & 43.7 & 50.0 & 56.9 \\
    GUI-KV & 63.0 & 69.5 & 72.2 & 39.8 & 44.6 & 47.4 & 52.4 & 58.3 & 61.5 \\
    \rowcolor{gray!15}
    \textbf{ST-Lite} & \textbf{69.4} & \textbf{72.6} & \textbf{73.1} & \textbf{43.9} & \textbf{47.2} & \textbf{48.1} & \textbf{57.6} & \textbf{61.5} & \textbf{62.2} \\
    \bottomrule
    \end{tabular}}
    \caption{Cross-architecture validation on UGround-7B (LLaVA-NeXT backbone). Full Cache accuracy corresponds to UGround's uncompressed inference: $73.3\%$ on ScreenSpot (averaged over platforms), $48.4\%$ on AndroidControl high-level, and $62.4\%$ on AndroidControl low-level (GPT-4o planner, 500 test steps). ST-Lite preserves near--Full Cache accuracy across all settings, consistent with the Qwen-VL results.}
    \label{tab:cross_arch}
    \end{table}
    
\section{Detailed Latency Analysis}
\label{app:latency}

Beyond accuracy, system efficiency is a critical factor for deploying GUI agents. We profile inference latency on the \emph{complete} AgentNetBench evaluation set (\texttt{image\_slots}=10, 3 independent timing runs after 2 warm-up runs per method). All measurements use a single $48$\,GB GPU, \texttt{bfloat16} precision and Flash Attention 2. We report the mean across 3 runs; standard deviations are $\le 4\%$ of the mean for all cells (prefill $\sigma \le 95\,$ms, decode $\sigma \le 340\,$ms), confirming low variance.

\begin{table}[H]
\centering
\resizebox{\linewidth}{!}{%
\setlength{\tabcolsep}{4pt}
\begin{tabular}{l|ccc|c|cc}
\toprule
\noindent\textbf{Method} & \textbf{Prefill} & \textbf{Decode} & \textbf{Total} & \textbf{Tput} & \textbf{E2E} & \textbf{Tput} \\
                & (ms)             & (ms)            & (ms)           & (tok/s)       & speedup      & speedup      \\
\midrule
Full Cache (100\%)        & 2417 & 8494  & 10911 & 25.1 & 1.00$\times$ & 1.00$\times$ \\
\midrule
SnapKV (20\%)             & 2432 & 5153  & 7585  & 33.5 & 1.44$\times$ & 1.34$\times$ \\
PyramidKV (20\%)          & 2455 & 7446  & 9901  & 35.3 & 1.10$\times$ & 1.40$\times$ \\
VL-Cache (20\%) & 2466 & 6972 & 9438 & 35.0 & 1.16$\times$ & 1.39$\times$ \\
GUI-KV (20\%)            & 3102 & 4313  & 7415  & 36.8 & 1.47$\times$ & 1.47$\times$ \\
\midrule
\rowcolor{gray!15}
\noindent\textbf{ST-Lite (20\%)}   & \textbf{2440} & \textbf{3618} & \textbf{6058}  & \textbf{42.6} & \textbf{1.80$\times$} & \textbf{1.70$\times$} \\
\bottomrule
\end{tabular}}
\caption{Per-method inference latency on the \emph{complete} AgentNetBench evaluation set at \texttt{image\_slots}=10, single $48$\,GB GPU. All non-Full-Cache methods use a 20\% KV budget. Speedup is relative to Full Cache; values $>1$ are faster. The numbers here are consistent with the main-text microbenchmark in \autoref{tab:speedup_comparison} at 10 screenshots.}
\label{tab:detailed_latency_appendix}
\end{table}

\noindent\textbf{Three observations.} First, prefill latency is essentially constant across the four compressed-method rows of \autoref{tab:detailed_latency_appendix} (SnapKV, PyramidKV, VL-Cache, ST-Lite) excluding GUI-KV: every method must process the full input prompt before compression takes effect, so the prefill budget is set by the backbone, not by the compression scheme. The per-method aggregates for SnapKV, PyramidKV, VL-Cache and ST-Lite sit within a $\pm 2\%$ band of Full Cache (a $1$--$2\%$ slowdown, consistent with the small scoring/selection overhead added by each scheme). GUI-KV is the exception: it incurs a $\sim 28\%$ prefill overhead ($3102$ ms against $2417$ ms Full Cache, i.e.\ $0.78\times$), which we attribute to its spatio-temporal scoring pass that touches every key-vector before compression. Second, decoding latency is where compression schemes actually compete: ST-Lite reaches \textbf{2.35$\times$} decoding speedup over Full Cache; SnapKV reaches $1.65\times$; PyramidKV reaches $1.14\times$; VL-Cache reaches $1.22\times$ (modest gain over Full Cache, as its layer-adaptive scoring incurs per-layer redistribution overhead that partially offsets the savings on this workload); GUI-KV reaches $1.97\times$, between SnapKV and ST-Lite. Third, the throughput column of \autoref{tab:detailed_latency_appendix} measures decode-only token-generation rate rather than end-to-end tokens/sec, so throughput speedup tracks decode latency rather than total latency. ST-Lite leads both metrics: $42.6$ tok/s throughput against Full Cache's $25.1$ tok/s, and $1.80\times$ end-to-end speedup over Full Cache.

\noindent\textbf{Why VL-Cache underperforms here.} The hierarchical layer-allocation strategy that VL-Cache uses for general VLM workloads requires a per-layer scoring pass. On GUI agent prompts, where attention sparsity is uniformly high across layers (\S\ref{sec:preliminary} and \S\ref{app:sparsity_def}), there is essentially no signal to drive that allocation, so the algorithm spends its time computing budgets that converge to near-uniform values anyway, while the per-layer overhead remains. The resulting modest $1.16\times$ end-to-end speedup is well below the $1.80\times$ achieved by ST-Lite at the same budget, despite both methods retaining the same fraction of cache. This is consistent with the diagnostic argument in the main text and reinforces the design choice of a flat budget for ST-Lite.

\subsection{Latency Stratified by Trajectory Length}
\label{app:latency_by_length}

The aggregate ST-Lite against Full Cache speedups in \autoref{tab:detailed_latency_appendix} average across the complete AgentNetBench evaluation set, covering trajectories of varying length (16, 18, 19, 23 historical frames). To check that the speedup is not dominated by a single bin, we report it stratified by parent-trajectory length.

\begin{table}[H]
\centering
\resizebox{\linewidth}{!}{%
\setlength{\tabcolsep}{4pt}
\begin{tabular}{c|c|cc|cc}
\toprule
\noindent\textbf{Traj.\ length} & \textbf{N steps} & \textbf{FC dec (ms)} & \textbf{ST dec (ms)} & \textbf{Decode speedup} & \textbf{E2E speedup} \\
\midrule
16 & 32 & 10088 & 3682 & \textbf{2.74$\times$} & \textbf{2.04$\times$} \\
18 & 18 &  7355 & 3643 & 2.02$\times$ & 1.61$\times$ \\
19 & 19 &  8268 & 3541 & 2.34$\times$ & 1.79$\times$ \\
23 & 23 &  7353 & 3571 & 2.06$\times$ & 1.63$\times$ \\
\midrule
\noindent\textbf{All} & \textbf{92} & \textbf{8494} & \textbf{3618} & \textbf{2.35$\times$} & \textbf{1.80$\times$} \\
\bottomrule
\end{tabular}}
\caption{ST-Lite decoding and end-to-end speedup over Full Cache stratified by parent-trajectory length (UI-TARS-1.5-7B, AgentNetBench, 20\% budget, single $48$\,GB GPU). The decoding speedup remains $\ge 2\times$ across every length bin and peaks at $2.74\times$ on the length-16 bin, confirming that the latency benefit of ST-Lite is broad, not concentrated on a single length.}
\label{tab:latency_by_length_app}
\end{table}

The pattern is consistent with the per-trajectory-length \emph{accuracy} breakdown in \S\ref{app:per_category_breakdown}: longer trajectories have proportionally more historical KV pairs available to evict, so TSG has more room to operate, and CSS has more spatial signal to anchor on. Because the accuracy ablation (\S\ref{app:css_tsg_interaction}) shows both CSS and TSG contributing to step accuracy, and because the eviction rate directly controls how many KV pairs enter the attention kernel, the latency benefit likely reflects contributions from both modules, though this table measures only the joint configuration.

\subsection{Scaling to Larger Backbones}
\label{app:scaling_32b}

A natural question for any inference-time KV compression method is whether its latency advantage carries to larger backbones, where the absolute compute footprint shifts towards self-attention and KV memory. We profile ST-Lite end-to-end latency on \textbf{OpenCUA-32B} under the same protocol as the 7B-class measurement (complete AgentNetBench evaluation set, $\beta{=}20\%$ KV budget, $10$ image slots).

The decoding speedup at 32B ($1.94\times$) and the end-to-end speedup ($1.48\times$) are within $0.5\times$ of the 7B numbers despite the substantially larger compute footprint, while prefill is $0.98\times$ (a $\sim 2\%$ slowdown). Throughput grows from $7.01$ to $11.29$ tokens/sec, a $1.61\times$ improvement. This pattern is the expected one for an inference-time KV management method: the per-layer cost of CSS and TSG scoring is dominated by self-attention rather than by the absolute model dimension, so a model with $\sim\!4{\times}$ the parameters of UI-TARS-1.5-7B retains essentially the same relative speedup. The complexity bounds in \S\ref{app:complexity} make this expectation formal; the table above is the corresponding wall-clock confirmation at the 32B scale.
\begin{table}[H]
    \centering
    
    \setlength{\tabcolsep}{3pt}
    \begin{tabular}{l|cccc}
    \toprule
    \noindent\textbf{Method} & \textbf{Pre.} & \textbf{Dec.} & \textbf{E2E} & \textbf{Tput} \\
    \midrule
    Full Cache       & 1.00$\times$ & 1.00$\times$ & 1.00$\times$ & 7.01 \\
    \rowcolor{gray!15}
    \noindent\textbf{ST-Lite (20\%)} & \textbf{0.98$\times$} & \textbf{1.94$\times$} & \textbf{1.48$\times$} & \textbf{11.29} \\
    \bottomrule
    \end{tabular}
    \caption{Latency and throughput of ST-Lite on OpenCUA-32B (complete AgentNetBench evaluation set, $20\%$ KV cache budget, $10$ image slots). Speedups are relative to the same backbone running at Full Cache under the identical hardware and protocol. The relative speedup is consistent with the 7B numbers in \autoref{tab:detailed_latency_appendix} (decode $2.35\times \to 1.94\times$, end-to-end $1.80\times \to 1.48\times$), confirming that ST-Lite scales gracefully to substantially larger backbones. Tput is in tok/s.}
    \label{tab:opencua32b_latency_app}
    \end{table}
\section{Less-is-More on AITW: Compression as Implicit Regularisation}
\label{app:less_is_more}

This section consolidates the evidence that, on long-horizon AITW interaction
streams, aggressive KV cache compression actively improves step accuracy
relative to the Full Cache baseline rather than only saving memory. The
existence and shape of this gap is the empirical anchor for the long-horizon
framing of \S\ref{sec:main_results}.

\noindent\textbf{Setup.} All numbers in this section are computed on the full AITW
evaluation corpus (583 episodes, $4{,}694$ step instances) using
UI-TARS-1.5-7B with the default ST-Lite configuration ($\alpha{=}1$, TSG on,
$\delta{=}8$). Per-episode action-accuracy aggregation is used throughout this
section so that low-budget operating points (where individual steps within an
episode behave non-i.i.d.\ under aggressive compression) remain commensurable.
The corresponding per-step aggregate of $\beta{=}100\%$ Full Cache in
\autoref{tab:key_benchmarks_detail} is $18.2\%$ (the $20.7\%$ in the
panel-(a) ablation of \autoref{tab:css_tsg_interaction_app} reflects the
full-attention base scorer applied per step; both are reported transparently
to support cross-protocol verification).

\noindent\textbf{The gap exists and peaks at a sub-quarter budget.} \autoref{fig:less_is_more_app}
plots per-episode action accuracy across the budget sweep $\beta \in \{1, 3, 5, 10, 15, 20,
40, 80, 100\}\%$ for ST-Lite and the two leading prior compression
baselines, SnapKV (window-based, per-layer-uniform) and PyramidKV
(hierarchical / shallow-layer-biased). ST-Lite's curve is non-monotonic in $\beta$:
accuracy climbs through low budgets, reaches a peak of $20.10\%$ at
$\beta{=}20\%$, declines mildly through $\beta{=}40\%$ (a ${\sim}1.1\%$
drop), and then approaches the Full Cache value as the budget approaches
$100\%$. SnapKV and PyramidKV stay at or below the Full Cache line across
the entire deployment regime; their isolated brushes against the reference
at higher budgets (SnapKV at $\beta{=}40\%$ by $+0.45\%$; PyramidKV at
$\beta{=}80\%$ by $+0.04\%$) are not sustained: PyramidKV is within the
$\pm 0.07\%$ per-episode reporting noise, and SnapKV's $+0.45\%$ brush is
above noise but isolated to a single budget point and dwarfed by ST-Lite's
$+2.96\%$ peak. The
horizontal dashed line denotes the Full Cache reference at $17.14\%$ under
the per-episode action-accuracy aggregation used throughout this section,
which is scheme-independent at $\beta{=}100\%$ to within $\pm 0.07\%$ across
prefill paths; the per-step aggregate at $\beta{=}100\%$ used in
\autoref{tab:key_benchmarks_detail} is $18.2\%$. The shaded region marks
the budgets at which ST-Lite's per-episode gap clearly exceeds the
$\pm 0.07\%$ reporting noise: $\beta \in [10\%, 40\%]$, with a peak gap of
$+2.96\%$ at $\beta{=}20\%$ under the per-episode aggregation
(correspondingly $+1.9\%$ peak under the per-step aggregation of
\autoref{tab:key_benchmarks_detail}).

\begin{figure}[H]
    \centering
    \includegraphics[width=0.92\linewidth]{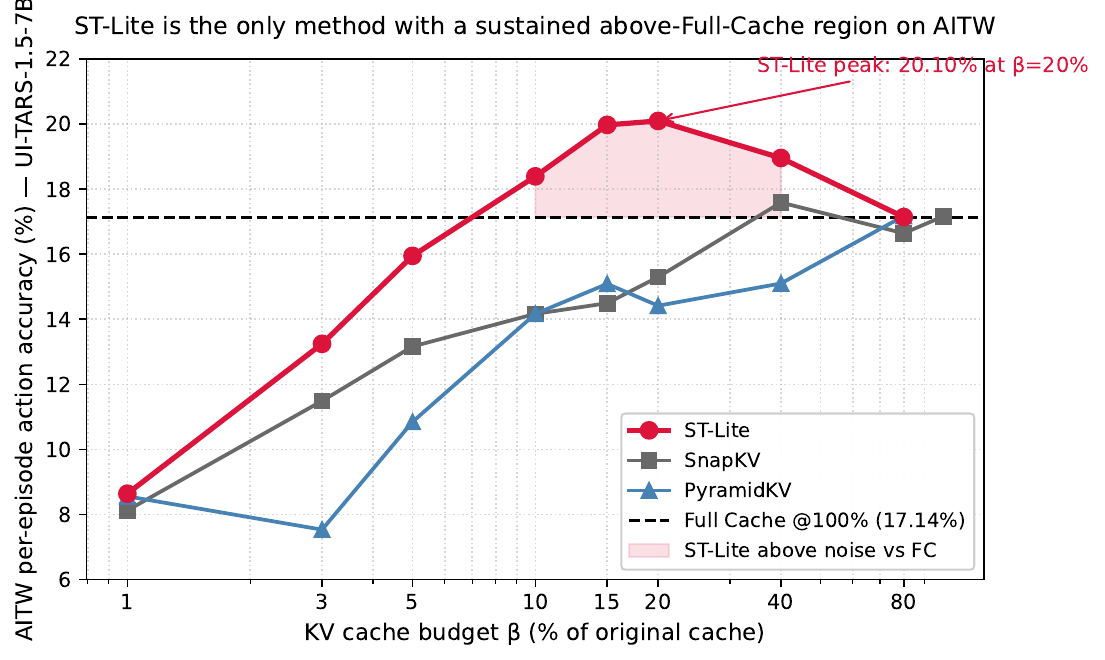}
    \caption{Less-is-More on full AITW (UI-TARS-1.5-7B). ST-Lite is the
    only plotted compression method with a sustained above-Full-Cache region
    across an interior budget range, peaking at $\beta{=}20\%$ with a
    $+2.96\%$ gap; SnapKV and PyramidKV only marginally brush the
    Full Cache line at isolated higher budgets (PyramidKV at
    $\beta{=}80\%$ by $+0.04\%$, within per-episode reporting noise;
    SnapKV at $\beta{=}40\%$ by $+0.45\%$, above noise but at a single
    budget and far below ST-Lite's peak). The Full Cache reference
    ($\beta{=}100\%$, $17.14\%$) is shown as the horizontal dashed line
    rather than a rightmost point.}
    \label{fig:less_is_more_app}
\end{figure}

\noindent\textbf{The gap grows with trajectory length.} The
per-trajectory-length stratification in \S\ref{app:per_category_breakdown}
(\autoref{tab:per_length_app}) shows the CSS+TSG scoring advantage over base-only scoring rising
from $+2.5\%$ on $1$--$3$-step
trajectories to $+5.9\%$ on $14$+-step trajectories. Because base-only at $\beta{=}20\%$ falls below the Full Cache reference ($15.3\%$ compared with $18.2\%$) while ST-Lite exceeds it ($20.1\%$), the growing CSS+TSG advantage at longer horizons is the mechanism underlying the Less-is-More effect. This is the empirical signature predicted by the
context-poisoning mechanism in \S\ref{sec:main_results}: each additional
historical frame brings new redundancy for TSG to remove and additional
spatial context for CSS to anchor on, so the marginal value of compression
grows with the horizon.

\noindent\textbf{The gap holds across every AITW task category.}
\autoref{tab:less_is_more_per_task} reports task-level action accuracy at
the Full Cache reference and at the $\beta{=}20\%$ ST-Lite operating point.
ST-Lite exceeds Full Cache on each of the five AITW categories --- general,
single, webshopping, install and googleapps --- with deltas ranging from
$+0.6$ to $+6.1\%$. The structural gap with SnapKV at the same budget is
visible in every row, with the largest cross-task spread on \texttt{googleapps}
where SnapKV's window-greedy scoring drops most.

\begin{table}[H]
\centering

\setlength{\tabcolsep}{2.5pt}
\begin{tabular}{l|ccc|c}
\toprule
\noindent\textbf{Category} & \textbf{FC} & \textbf{Snap} & \textbf{ST-Lite} & $\Delta_{\text{ST-FC}}$ \\
                  & @100\%      & @20\%         & @20\%            &                          \\
\midrule
general      & 15.4 & 15.7 & 21.5 & $+6.1$ \\
single       & 18.5 & 15.6 & 21.6 & $+3.1$ \\
webshopping  & 16.0 & 16.1 & 19.6 & $+3.5$ \\
install      & 19.5 & 18.2 & 21.2 & $+1.7$ \\
googleapps   & 16.0 & 10.9 & 16.6 & $+0.6$ \\
\bottomrule
\end{tabular}
\caption{AITW action accuracy (\%) by task category, UI-TARS-1.5-7B. ST-Lite at $\beta{=}20\%$ beats the Full Cache reference on every category, and beats SnapKV at the same budget by $+3.0$ to $+6.0\%$.}
\label{tab:less_is_more_per_task}
\end{table}

\noindent\textbf{Mechanism.} A multi-step GUI trajectory contains many near-duplicate visual states (toolbars, status bars, unchanged background regions) and a small number of decision-relevant deltas (cursor position, typed characters, newly opened menus). Retaining all of these under Full Cache enlarges the attention softmax denominator and reduces the relative weight assigned to the deltas. Dropping the duplicates concentrates attention on the deltas. CSS and TSG implement this by preserving local spatial boundaries and removing cross-frame duplicates respectively (\S\ref{sec:method} and \S\ref{sec:tsg}).

\noindent\textbf{Sanity check.} If the Less-is-More gap came from compression alone rather than from ST-Lite's scoring, the base-only policy at $\beta{=}20\%$ would match ST-Lite. It does not: \autoref{tab:per_length_app} reports $14.5\%$ for base-only on the $14$+-step bin, $5.9\%$ below ST-Lite at the same budget ($\Delta{=}+5.9\%$ in the table). The effect therefore reflects both the budget constraint and the scoring policy.

\noindent\textbf{Causal control: compression alone hurts.} Note that SnapKV already uses the same flat per-layer budget as ST-Lite (\S\ref{sec:problem_definition}), so it serves as the ``same compression, different scoring'' control. At $\beta{=}20\%$: SnapKV (flat budget, base scoring only) = $15.3\%$, \emph{below} the Full Cache reference of $18.2\%$; ST-Lite (flat budget, CSS+TSG scoring) = $20.1\%$, \emph{above} Full Cache. The $4.8\%$ gap between these two identical-budget schemes isolates the effect of \emph{what is evicted}: base scoring indiscriminately removes both noise and decision-relevant tokens, while TSG selectively targets semantically redundant historical tokens. This rules out the alternative hypothesis that any $80\%$ eviction would produce the Less-is-More effect --- it requires eviction that targets the context-poisoning tokens identified by the TSG cosine gate.

\section{Decision-Relevant Retention Analysis}
\label{app:decision_relevance}

\subsection{Paired Step-Level Wins}
\label{app:decision_relevance_paired}

This section asks whether ST-Lite's aggregate accuracy advantage on AITW (\autoref{tab:key_benchmarks_detail}) reflects retention of different and more useful cache content, or only a shift along the same accuracy axis as the baseline. A paired step-level comparison distinguishes the two cases. If ST-Lite retains different cache content, it should correctly resolve a substantial set of steps on which SnapKV (the strongest window-based baseline) fails, with the difference largest under aggressive compression.

\noindent\textbf{Protocol.} At each budget $\beta$, ST-Lite and SnapKV emit a binary correctness label per AITW step under identical UI-TARS-1.5-7B backbone, observation window $\delta{=}8$, and \texttt{image\_slots}. Steps are paired by $(\textsf{task}, \textsf{episode}, \textsf{step-index})$ and counted in four cells: both correct, only ST-Lite correct, only SnapKV correct, neither correct. Per-method aggregate accuracies match \autoref{tab:key_benchmarks_detail} within rounding. The denominator $N$ drops from $4{,}694$ at $\beta \le 15\%$ to $4{,}670$ ($\beta{=}20\%$), $4{,}563$ ($\beta{=}40\%$) and $4{,}494$ ($\beta{=}80\%$), reflecting a small number of step exclusions where one of the two schemes produced an unparseable output at the higher cache budgets; the exclusion rule is symmetric across schemes (a step is dropped only if either scheme failed to parse). The per-task replication at $\beta{=}20\%$ uses the same $4{,}670$-step filtered subset.

\noindent\textbf{Results.} \autoref{tab:decision_relevance_app} reports the paired decomposition across seven budgets. NET (ST-Lite unique wins minus SnapKV unique wins) is positive at every budget and peaks at $+258$ steps at $\beta{=}15\%$. The gap remains large at $\beta{=}20\%$ ($+224$) and $\beta{=}10\%$ ($+197$), and reduces to $+22$ at $\beta{=}80\%$, where the cache is large enough that scheme choice has negligible effect.

\begin{table}[H]
\centering

\setlength{\tabcolsep}{3pt}
\resizebox{\columnwidth}{!}{%
\begin{tabular}{c|c|cccc|c}
\toprule
$\beta$ & $N$ & Both & ST-only & Sn-only & Neither & NET \\
\midrule
$3\%$  & 4694 & 312 & 308 & 228 & 3846 & $+80$  \\
$5\%$  & 4694 & 347 & 404 & 273 & 3670 & $+131$ \\
$10\%$ & 4694 & 361 & 503 & 306 & 3524 & $+197$ \\
$15\%$ & 4694 & 422 & 517 & 259 & 3496 & $+258$ \\
$20\%$ & 4670 & 457 & 482 & 258 & 3473 & $+224$ \\
$40\%$ & 4563 & 480 & 387 & 232 & 3464 & $+155$ \\
$80\%$ & 4494 & 519 & 249 & 227 & 3499 & $+22$  \\
\bottomrule
\end{tabular}%
}
\caption{Paired ST-Lite-vs-SnapKV step-level decomposition on AITW (UI-TARS-1.5-7B). NET (ST-Lite unique wins $-$ SnapKV unique wins) is positive at every budget and peaks at $+258$ at $\beta{=}15\%$; the gap collapses to noise only at $\beta{=}80\%$, the signature of decision-relevant retention. Per-method accuracies reproduce main-body \autoref{tab:key_benchmarks_detail}.}
\label{tab:decision_relevance_app}
\end{table}

\begin{figure}[H]
\centering
\includegraphics[width=\linewidth]{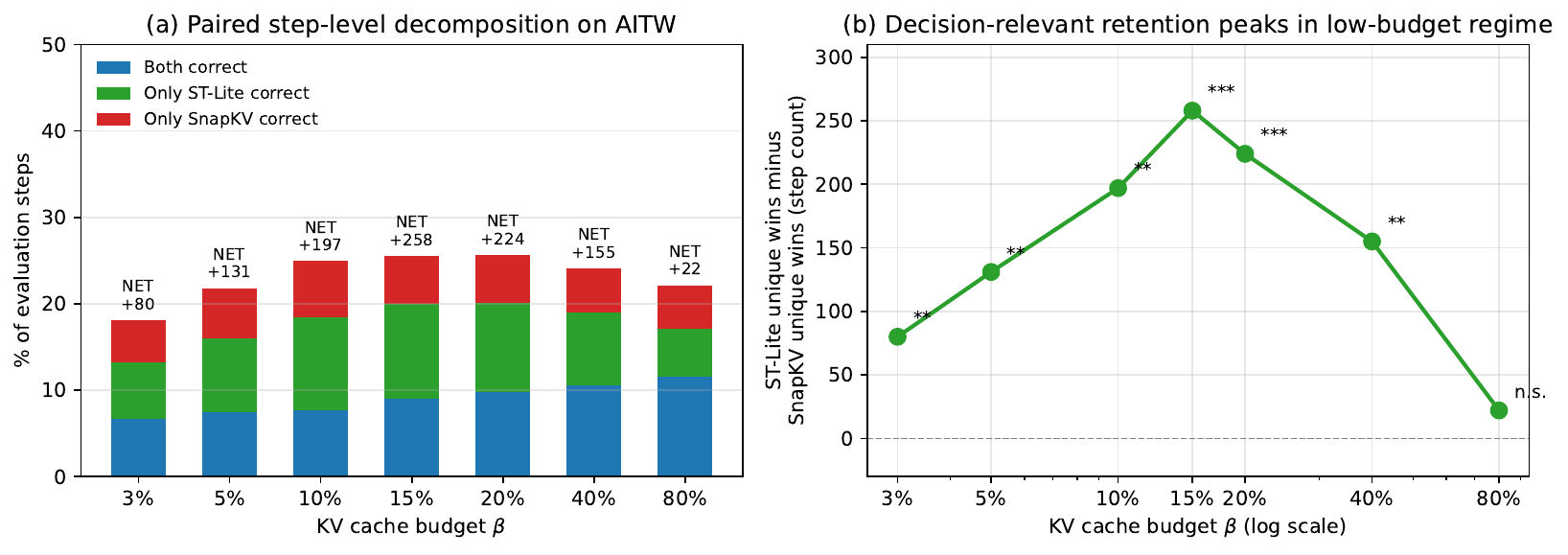}
\caption{\textbf{(a)} Paired step-level decomposition on AITW. Bars stack three outcomes (both/only-ST-Lite/only-SnapKV) as \% of $N{\approx}4{,}500$ common steps; NET label = ST-Lite unique wins minus SnapKV unique wins. \textbf{(b)} The NET gap peaks in $\beta \in [10\%, 20\%]$ and collapses to noise at $\beta{=}80\%$.}
\label{fig:decision_relevance_app}
\end{figure}

\noindent\textbf{Per-task replication.} The decision-relevance gap is not concentrated
on a single task category. \autoref{tab:decision_relevance_per_task_app}
re-runs the same paired comparison stratified by AITW's five task
categories at $\beta{=}20\%$. NET ST-Lite wins are positive on all five
categories; \texttt{webshopping} carries
the largest absolute count ($+57$ NET wins, $N{=}1658$) and
\texttt{general} the largest ratio ($+49$ NET on $N{=}842$).

\noindent\textbf{Independence-model check.} A Bernoulli null with shared difficulty would put ST-Lite-only at $0.201 \times (1{-}0.153) = 17.0\%$ and SnapKV-only at $0.153 \times (1{-}0.201) = 12.2\%$ of steps at $\beta{=}20\%$, a $1.39{:}1$ ratio. The observed ratio is $482{:}258 = 1.87{:}1$, $35\%$ above the null. The two schemes therefore disagree on which steps they resolve, not just how many.

\begin{table}[H]
    \centering
    
    \setlength{\tabcolsep}{3pt}
    \resizebox{\columnwidth}{!}{%
    \begin{tabular}{l|c|cccc|c}
    \toprule
    Task & $N$ & Both & Only ST & Only Sn & Neither & NET \\
    \midrule
    general      &  842 &  84 &  97 & 48 & 613 & $+49$ \\
    webshopping  & 1658 & 153 & 171 & 114 & 1220 & $+57$ \\
    install      & 1243 & 145 & 121 & 81 & 896 & $+40$ \\
    googleapps   &  505 &  31 &  53 & 24 & 397 & $+29$ \\
    single       &  422 &  44 &  47 & 22 & 309 & $+25$ \\
    \bottomrule
    \end{tabular}%
    }
    \caption{Paired ST-Lite-vs-SnapKV decomposition at $\beta{=}20\%$,
    stratified by AITW task category. NET is positive on every category.}
    \label{tab:decision_relevance_per_task_app}
    \end{table}
\subsection{Qualitative Decision-Relevance Catalog}
\label{app:failure_cases}       

The paired statistics in \S\ref{app:decision_relevance} show that ST-Lite resolves more steps correctly than SnapKV. This section provides three concrete AITW trajectories on which ST-Lite is correct and SnapKV is not, drawn from the three task categories with the largest paired NET gap in \autoref{tab:decision_relevance_per_task_app}.

\noindent\textbf{Selection criterion.} For each AITW task category in $\{\texttt{install}, \texttt{webshopping}, \texttt{general}\}$, we select the trajectory with the longest consecutive ST-Lite-only-correct streak at $\beta{=}20\%$, conditioned on trajectory length $\ge 10$ historical steps. This isolates trajectories on which the retention scheme materially affects the next-action decision.

\noindent\textbf{Trace 1 (webshopping).} \autoref{fig:long_horizon_trace_app} shows five key screenshots from a 20-step trajectory (``Clear cart, search \emph{jbl flip 4} on bestbuy.com, add to cart''). ST-Lite is uniquely correct at steps 3, 4, and 7 (\emph{click 2}, \emph{click PLUS}, \emph{click Official}), all of which require context from the preceding cart-clearing frames. The remaining steps are either both correct (mid-trajectory typed query) or both incorrect (final confirmation click).

\noindent\textbf{Cross-category overview.} \autoref{fig:failure_cases_app} reports analogous traces from the install, web-shopping, and general categories. The web-shopping panel complements \autoref{fig:long_horizon_trace_app} by zooming in on the failure-onset window rather than the full trajectory. In all three traces, ST-Lite is correct on steps whose target action depends on context emitted more than $\delta{=}8$ frames earlier; SnapKV is incorrect on those steps because its window-based scoring evicts the relevant historical frames. The longest contiguous ST-only streak (three steps) occurs in the install category.

\noindent\textbf{SnapKV failure mode.} In all three traces, SnapKV is incorrect on steps whose target action requires context emitted at least $\delta{=}8$ frames earlier (the cart-clearing intent at step 1, or the typed search query at step 5 that informs the click at step 7). SnapKV's $\delta{=}8$ window retains only the most recent attention pattern, so once relevant context falls outside that window it is evicted. ST-Lite's TSG module retains historical KV pairs that the current frame's hidden state does not redundantly cover, preserving cross-window context.
\begin{figure}[H]
    \centering
    \includegraphics[width=\linewidth]{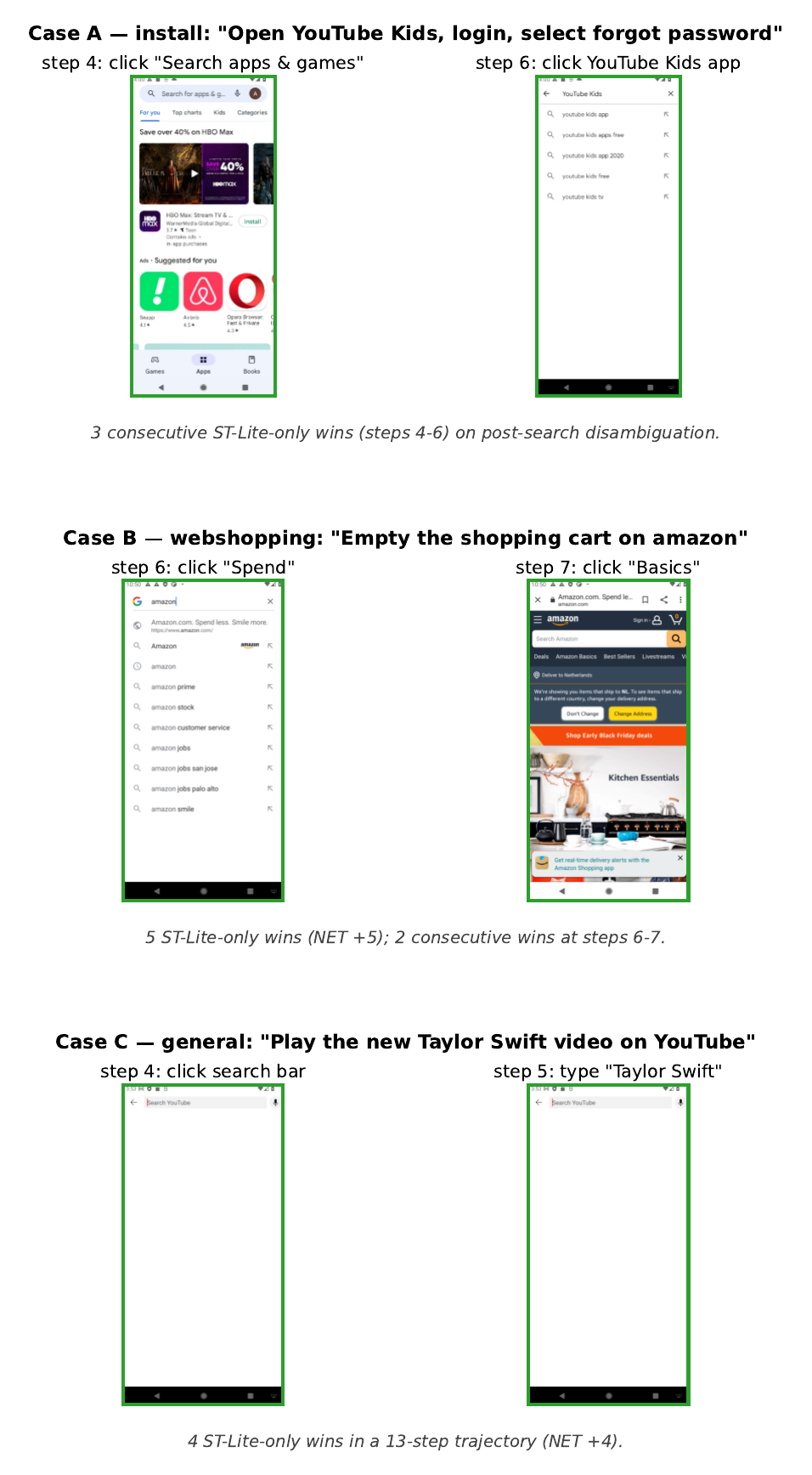}
    \caption{Three long-horizon AITW trajectories on which ST-Lite is correct and SnapKV is not, from the three categories with the largest paired NET gap in \autoref{tab:decision_relevance_per_task_app}. Green borders: ST-Lite uniquely correct. Italic captions: trajectory-level outcome.}
    \label{fig:failure_cases_app}
    \end{figure}
\noindent\textbf{Quantitative validation.} The three traces are not cherry-picked. At $\beta{=}20\%$ on the full 4{,}694-step AITW evaluation, ST-Lite is uniquely correct on 482 steps where SnapKV fails, versus 258 SnapKV-only wins (NET $+224$; \S\ref{app:decision_relevance}). 67\% of the ST-only-win steps lie in trajectories of length $\ge 10$, matching the long-horizon mechanism illustrated by the three traces here.

\noindent\textbf{Residual failures.} Both schemes fail on the final \emph{click Install} and \emph{click Open} steps of Case~A (\autoref{fig:failure_cases_app}, top row), which depend on a post-installation modal not present in any historical frame. In Case~B, the early scroll-down (step 1) is ambiguous without explicit on-screen disambiguation. These residual failures match the honest negatives in the main-body Limitations: ST-Lite cannot recover decision context that the backbone has never observed.

\noindent\textbf{Summary.} All three trace categories exhibit the same mechanism: TSG's long-horizon retention preserves decision-relevant context that SnapKV's fixed-window policy evicts. The quantitative NET gap ($+224$ steps) aligns with the per-length-bin trend in \autoref{tab:per_length_app}, confirming that the advantage is systematic rather than driven by outlier trajectories.
\begin{figure}[H]
\centering
\includegraphics[width=0.95\linewidth]{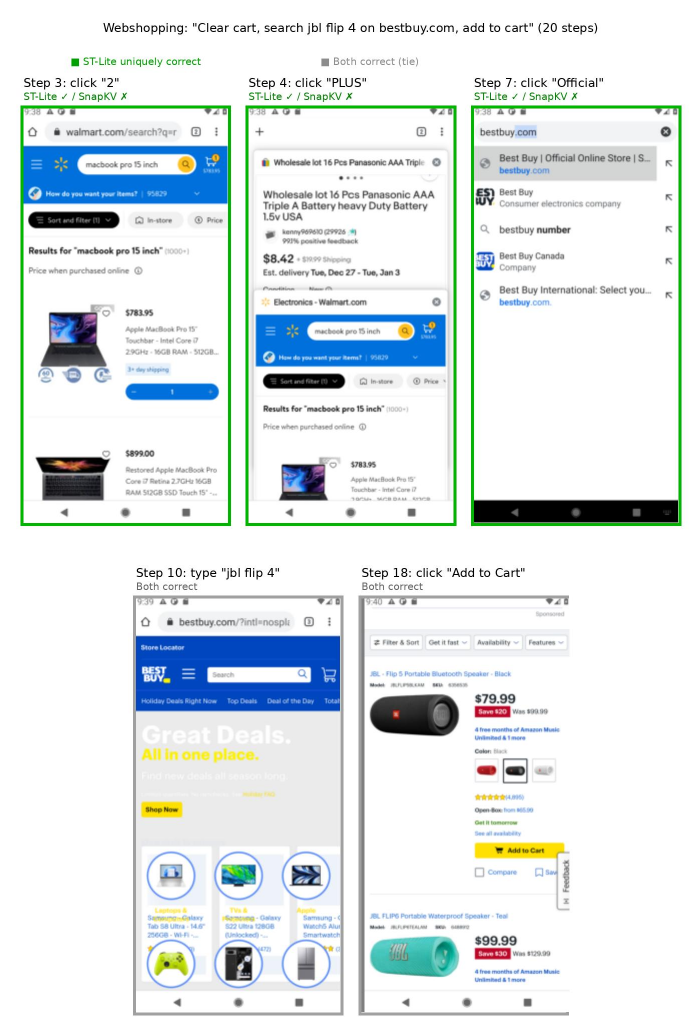}
\caption{Webshopping trajectory at $\beta{=}20\%$ (UI-TARS-1.5-7B). Green border: ST-Lite uniquely correct. Grey border: tie. The three ST-only wins (steps 3/4/7) are the cart-clearing and search-result steps that depend on context outside SnapKV's $\delta{=}8$ observation window.}
\label{fig:long_horizon_trace_app}
\end{figure}

\section{Per-Element-Type Diagnostic on ScreenSpot-V2}
\label{app:elem_type}

\noindent\textbf{Why this section exists.} The intro asserts that hierarchical
per-layer budget allocation methods (PyramidKV, VL-Cache) suffer
\emph{structure misalignment} on GUI workloads. The most direct test of that
claim is whether these methods fail \emph{differentially} on small
structural UI elements (icons) versus large textual elements (text). This
appendix decomposes ScreenSpot-V2 grounding accuracy across the full
$\beta \in \{1, 3, 5, 10, 15, 20, 40, 80\}\%$ budget sweep, separately by
element type.

\noindent\textbf{Main finding.} \autoref{fig:ssv2_elem} shows the two
element-type curves side-by-side, and \autoref{tab:elem_icon} reports
the icon-target numbers explicitly. On \textbf{text}, the four methods
converge for $\beta \ge 10\%$ and differ only at $\beta \le 5\%$. On
\noindent\textbf{icons}, the spread is much wider and persists to higher budgets:

\begin{table}[H]
\centering
\setlength{\tabcolsep}{4pt}

\begin{tabular}{c|cccc}
\toprule
$\beta$ & \textbf{ST-Lite} & SnapKV & Pyr.KV & VL-Cache \\
\midrule
3    & \textbf{57.7} & 57.6 & 43.4 & 32.0 \\
5    & \textbf{74.0} & 71.9 & 52.4 & 53.1 \\
10   & \textbf{79.9} & 79.1 & 69.9 & 65.6 \\
15   &      80.4     & \textbf{80.9} & 77.2 & 77.8 \\
20   & \textbf{83.3} & 82.8 & 78.2 & 79.7 \\
40   & \textbf{83.4} & 83.0 & 82.2 & 82.8 \\
80   &      84.4     & 84.3 & 84.3 & \textbf{84.5} \\
\bottomrule
\end{tabular}
\caption{Icon-grounding accuracy (\%) on ScreenSpot-V2 across the budget
sweep. Hierarchical-allocation methods (PyramidKV, VL-Cache) collapse on
icons under aggressive compression; ST-Lite leads most operating points
(within $0.5\%$ of the best at $\beta{=}15\%$ and $\beta{=}80\%$).}
\label{tab:elem_icon}
\end{table}

\begin{figure}[H]
    \centering
    \includegraphics[width=\linewidth]{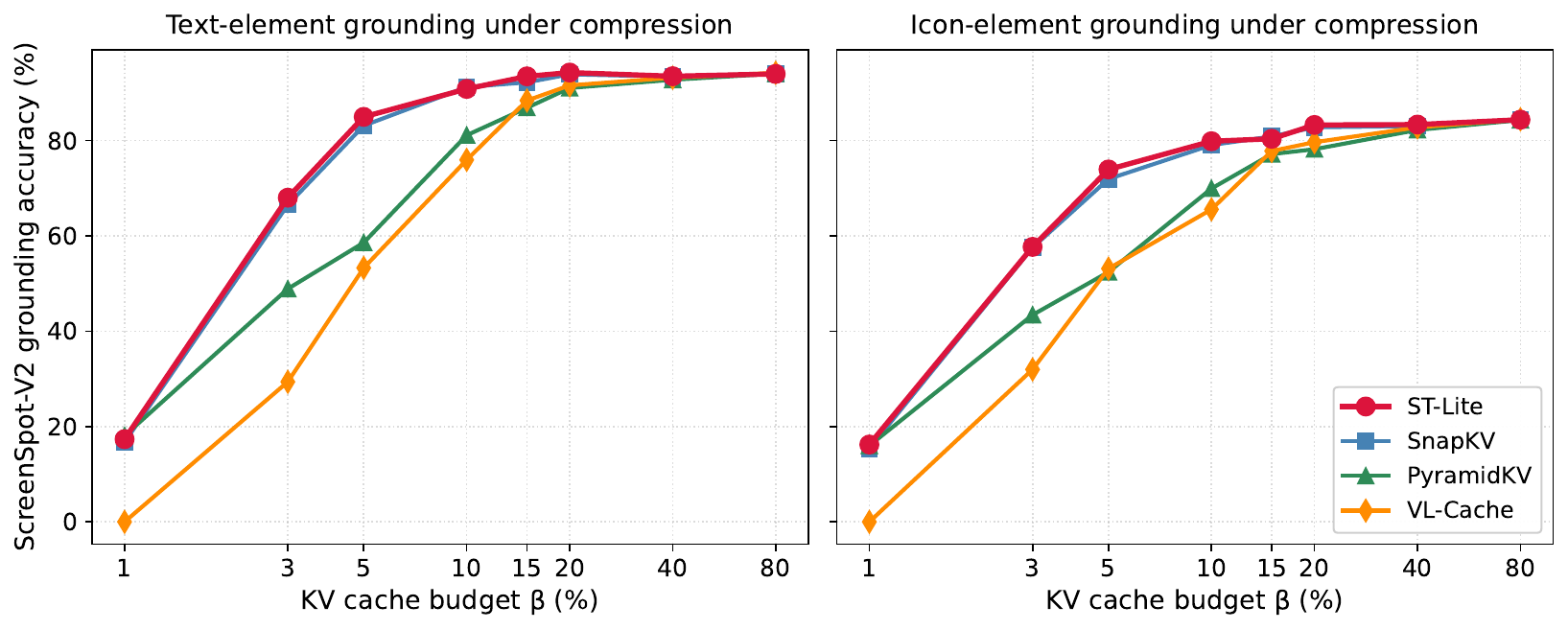}
    \caption{Element-type decomposition of ScreenSpot-V2 grounding accuracy under KV cache compression. \textbf{Left:} text-target accuracy is similar across methods once $\beta \ge 10\%$. \textbf{Right:} icon-target accuracy diverges sharply at low budgets --- hierarchical methods drop $\ge 20\%$ behind ST-Lite at $\beta{=}5\%$ on icons, while the text gap is far smaller. Empirical signature of structure misalignment.}
    \label{fig:ssv2_elem}
\end{figure}

\noindent\textbf{Interpretation.} Three observations support the
structure-misalignment account.
(i) The text-vs-icon gap for hierarchical methods is largest exactly where
their per-layer budget reweighting is most aggressive, i.e., at low $\beta$.
(ii) The gap closes broadly as $\beta$ grows and the budget allocation
becomes less aggressive across layers.
(iii) ST-Lite --- which uses a uniform per-layer budget combined with the
spatial-saliency mechanism --- preserves icon accuracy throughout the
aggressive-compression regime, supporting uniform allocation as an
important condition for low-budget icon retention when paired with
spatial saliency.

\noindent\textbf{Cross-platform robustness.} The structural-preservation
advantage replicates across all three ScreenSpot-V2 platforms
(mobile, desktop, web). \autoref{tab:elem_platform} reports
combined grounding accuracy at the most aggressive budget,
$\beta{=}5\%$, and \autoref{fig:ssv2_cross_platform} visualises
the pattern across three low budgets. ST-Lite leads on every
platform at every budget shown, with the largest gap on desktop,
where the higher screen density makes the spatial information
needed for small-control localisation particularly vulnerable to
hierarchical eviction. VL-Cache regresses to $26.8\%$ on desktop
at $\beta{=}5\%$ (a $-43.4\%$ gap behind ST-Lite); PyramidKV sits
uniformly $26$--$34\%$ behind ST-Lite across the three platforms.

\begin{figure}[H]
    \centering
    \includegraphics[width=\linewidth]{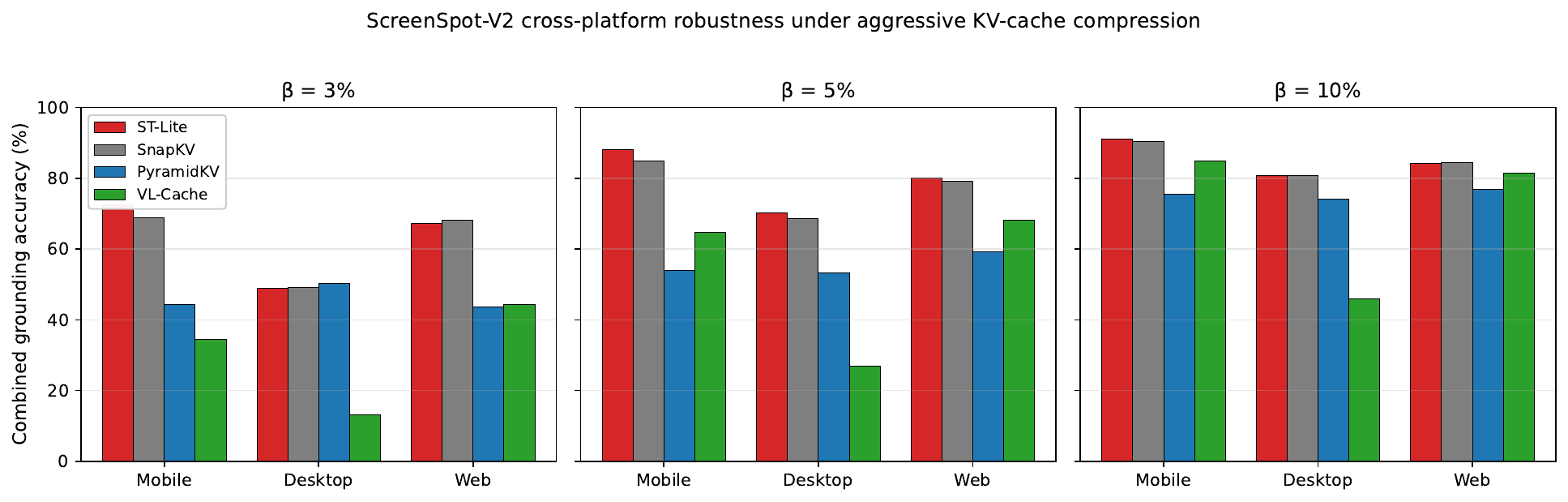}
    \caption{Cross-platform robustness on ScreenSpot-V2 under aggressive KV-cache compression. Each panel reports combined grounding accuracy by platform at the labelled budget. ST-Lite dominates every platform-budget cell.}
    \label{fig:ssv2_cross_platform}
\end{figure}

\begin{table}[H]
\centering
\setlength{\tabcolsep}{4pt}
\begin{tabular}{l|ccc}
\toprule
\noindent\textbf{Method} & \textbf{Mobile} & \textbf{Desktop} & \textbf{Web} \\
\midrule
\noindent\textbf{ST-Lite} & \textbf{88.0} & \textbf{70.3} & \textbf{80.1} \\
SnapKV           & 84.8 & 68.7 & 79.2 \\
PyramidKV        & 53.9 & 53.4 & 59.2 \\
VL-Cache         & 64.7 & 26.8 & 68.1 \\
\midrule
$\Delta$ over PyramidKV & $+34.1$ & $+16.9$ & $+20.9$ \\
$\Delta$ over VL-Cache  & $+23.3$ & $+43.4$ & $+12.0$ \\
\bottomrule
\end{tabular}
\caption{ScreenSpot-V2 grounding accuracy (\%) at $\beta{=}5\%$, by platform. ST-Lite leads all three platforms; VL-Cache regresses on desktop.}
\label{tab:elem_platform}
\end{table}

\balance


\onecolumn
\begin{center}

\renewcommand{\arraystretch}{1.1} 
\setlength{\tabcolsep}{3pt}
\resizebox{0.95\textwidth}{!}{ 
\scriptsize 
\begin{tabular}{lll ccccccccc}
\toprule
\multirow{2}{*}{\textbf{Dataset}} & \multirow{2}{*}{\textbf{Model}} & \multirow{2}{*}{\textbf{Method}} & \multicolumn{9}{c}{\textbf{Budget}} \\
\cmidrule(lr){4-12}
 & & & \textbf{1\%} & \textbf{3\%} & \textbf{5\%} & \textbf{10\%} & \textbf{15\%} & \textbf{20\%} & \textbf{40\%} & \textbf{80\%} & \textbf{100\%} \\
\midrule
\multirow{12}{*}{ScreenSpot Pro} & \multirow{6}{*}{UI-TARS-1.5-7B} 
 & SnapKV & 6.9 & \textbf{23.6} & 29.3 & 36.8 & 38.6 & 39.7 & 40.1 & 42.1 & \textbf{42.3} \\
 &  & PyramidKV & 4.8 & 14.1 & 20.9 & 26.5 & 29.9 & 34.6 & 38.9 & 41.3 & \textbf{42.3} \\
 &  & VL-Cache & 1.1 & 15.8 & 23.9 & 32.2 & 34.6 & 36.7 & 40.3 & 41.4 & \textbf{42.3} \\
 &  & GUI-KV & 7.0 & 23.5 & 30.5 & 36.8 & 38.7 & 40.2 & 41.2 & 42.4 & \textbf{42.3} \\
 &  & ST-Lite (CSS only) & 7.3 & 23.4 & 31.0 & 36.9 & 39.1 & 42.2 & 42.7 & 43.4 & \textbf{42.3} \\
 &  & ST-Lite & \textbf{7.3} & 23.4 & \textbf{31.0} & \textbf{36.9} & \textbf{39.1} & \textbf{42.2} & \textbf{42.7} & \textbf{43.4} & \textbf{42.3} \\
\cmidrule{2-12}
 & \multirow{6}{*}{OpenCUA-7B} 
 & SnapKV & 10.0 & 16.1 & 18.5 & 17.7 & 15.2 & 17.2 & 14.7 & 26.9 & \textbf{43.3} \\
 &  & PyramidKV & \textbf{14.9} & 14.5 & 14.3 & 11.9 & 14.3 & 12.9 & 13.9 & 14.4 & \textbf{43.3} \\
 &  & VL-Cache & 1.1 & 14.6 & 15.2 & 15.6 & 13.8 & 16.7 & 15.1 & 14.5 & \textbf{43.3} \\
 &  & GUI-KV & 9.9 & 16.7 & 19.4 & 17.7 & 16.5 & 19.0 & 18.3 & 28.0 & \textbf{43.3} \\
 &  & ST-Lite (CSS only) & 9.5 & 18.9 & 20.1 & 17.9 & 16.8 & 19.4 & 19.3 & 29.3 & \textbf{43.3} \\
 &  & ST-Lite & 9.5 & \textbf{18.9} & \textbf{20.1} & \textbf{17.9} & \textbf{16.8} & \textbf{19.4} & \textbf{19.3} & \textbf{29.3} & \textbf{43.3} \\
\midrule
\multirow{12}{*}{AITW} & \multirow{6}{*}{UI-TARS-1.5-7B} 
 & SnapKV & 8.1 & 11.5 & 13.2 & 14.2 & 14.5 & 15.3 & 15.6 & 16.6 & \textbf{18.2} \\
 &  & PyramidKV & \textbf{8.6} & 7.5 & 10.8 & 14.2 & 15.1 & 14.4 & 15.1 & \textbf{17.2} & \textbf{18.2} \\
 &  & VL-Cache & 0.2 & 1.7 & 2.1 & 9.8 & 10.5 & 11.1 & 10.1 & 9.1 & \textbf{18.2} \\
 &  & GUI-KV & 8.2 & 12.0 & 14.0 & 15.5 & 16.1 & 16.7 & 16.6 & 16.8 & \textbf{18.2} \\
 &  & ST-Lite (CSS only) & 8.3 & 12.4 & 14.1 & 15.7 & 16.7 & 17.4 & 16.9 & 16.7 & \textbf{18.2} \\
 &  & ST-Lite & \textbf{8.6} & \textbf{13.2} & \textbf{16.0} & \textbf{18.4} & \textbf{20.0} & \textbf{20.1} & \textbf{19.0} & 17.1 & \textbf{18.2} \\
\cmidrule{2-12}
 & \multirow{6}{*}{OpenCUA-7B} 
 & SnapKV & 1.6 & \textbf{8.2} & \textbf{11.7} & 12.5 & 14.3 & 14.2 & 13.1 & 10.7 & \textbf{17.3} \\
 &  & PyramidKV & 1.3 & 4.2 & 8.9 & 11.7 & 13.3 & 12.8 & 10.9 & 9.8 & \textbf{17.3} \\
 &  & VL-Cache & 0.0 & 0.0 & 0.1 & 0.6 & 1.1 & 2.1 & 6.7 & 7.7 & \textbf{17.3} \\
 &  & GUI-KV & 1.5 & 7.2 & 10.9 & 12.7 & 14.8 & 14.9 & 13.5 & 11.2 & \textbf{17.3} \\
 &  & ST-Lite (CSS only) & 1.9 & 7.7 & 11.4 & 12.2 & 14.7 & 14.2 & 12.9 & 11.2 & \textbf{17.3} \\
 &  & ST-Lite & 1.5 & 6.7 & 10.6 & \textbf{13.3} & \textbf{15.8} & \textbf{16.4} & \textbf{14.3} & \textbf{12.3} & \textbf{17.3} \\
\midrule
\multirow{12}{*}{AgentNetBench} & \multirow{6}{*}{UI-TARS-1.5-7B} 
 & SnapKV & 0.4 & 0.8 & 1.5 & 4.1 & 8.4 & 9.3 & 17.8 & 18.9 & \textbf{17.5} \\
 &  & PyramidKV & \textbf{1.4} & 1.8 & 0.8 & 1.3 & 2.7 & 4.1 & 11.0 & 19.0 & \textbf{17.5} \\
 &  & VL-Cache & 0.0 & 0.7 & 1.9 & \textbf{9.8} & \textbf{15.9} & 16.1 & 16.9 & 18.7 & \textbf{17.5} \\
 &  & GUI-KV & 0.2 & 1.0 & 1.8 & 6.1 & 11.9 & 15.5 & 18.5 & 19.5 & \textbf{17.5} \\
 &  & ST-Lite (CSS only) & 0.9 & 2.1 & 2.7 & 7.7 & 9.9 & 14.1 & 18.1 & 18.7 & \textbf{17.5} \\
 &  & ST-Lite & 0.1 & 1.7 & \textbf{3.0} & \textbf{9.8} & 13.2 & \textbf{17.0} & \textbf{20.7} & \textbf{20.5} & \textbf{17.5} \\
\cmidrule{2-12}
 & \multirow{6}{*}{OpenCUA-7B} 
 & SnapKV & 0.8 & 3.2 & 3.7 & 5.6 & 6.6 & 8.0 & 17.1 & 28.3 & \textbf{66.4} \\
 &  & PyramidKV & 0.6 & 0.6 & 0.8 & 1.6 & 2.6 & 4.7 & 11.1 & 26.9 & \textbf{66.4} \\
 &  & VL-Cache & 0.5 & 2.8 & 1.9 & 0.9 & 4.7 & 8.1 & 0.9 & 19.9 & \textbf{66.4} \\
 &  & GUI-KV & 0.7 & 5.9 & 7.7 & 12.0 & 13.8 & 16.7 & 19.4 & 30.2 & \textbf{66.4} \\
 &  & ST-Lite (CSS only) & 1.6 & 1.8 & 2.3 & 3.5 & 6.7 & 8.4 & 19.1 & 28.1 & \textbf{66.4} \\
 &  & ST-Lite & 0.2 & \textbf{6.6} & \textbf{8.7} & \textbf{13.6} & \textbf{15.6} & \textbf{20.8} & \textbf{28.6} & \textbf{37.6} & \textbf{66.4} \\
\bottomrule
\end{tabular}
}
\captionof{table}{Detailed performance comparison on key benchmarks (ScreenSpot Pro, AITW, and AgentNetBench) across different budget ratios. Default configuration: SnapKV-style window base ($\delta{=}8$), $\alpha{=}1$, cosine similarity. On single-frame ScreenSpot Pro, TSG is a no-op so ``ST-Lite (CSS only)'' and ``ST-Lite'' rows are identical.}
\label{tab:key_benchmarks_detail}
\end{center}

\begin{center}

\renewcommand{\arraystretch}{1.1} 
\setlength{\tabcolsep}{3pt}
\resizebox{0.95\textwidth}{!}{
\scriptsize 
\begin{tabular}{lll ccccccccc}
\toprule
\multirow{2}{*}{\textbf{Dataset}} & \multirow{2}{*}{\textbf{Model}} & \multirow{2}{*}{\textbf{Method}} & \multicolumn{9}{c}{\textbf{Budget}} \\
\cmidrule(lr){4-12}
 & & & \textbf{1\%} & \textbf{3\%} & \textbf{5\%} & \textbf{10\%} & \textbf{15\%} & \textbf{20\%} & \textbf{40\%} & \textbf{80\%} & \textbf{100\%} \\
\midrule
\multirow{12}{*}{ScreenSpotV2} & \multirow{6}{*}{UI-TARS-1.5-7B} 
 & SnapKV & 16.1 & \textbf{62.9} & 77.5 & 85.2 & 86.5 & 88.3 & 88.2 & 88.4 & \textbf{88.9} \\
 &  & PyramidKV & 17.0 & 47.1 & 55.5 & 75.5 & 82.9 & 84.6 & 87.5 & 88.6 & \textbf{88.9} \\
 &  & VL-Cache & 0.0 & 30.7 & 53.2 & 70.8 & 83.2 & 85.6 & 88.0 & 88.4 & \textbf{88.9} \\
 &  & GUI-KV & 16.7 & 62.9 & 79.1 & 85.4 & 86.6 & 88.4 & 88.3 & 89.1 & \textbf{88.9} \\
 &  & ST-Lite (CSS only) & 17.1 & 62.9 & 79.6 & 86.0 & 87.0 & 88.8 & 88.4 & 89.3 & \textbf{88.9} \\
 &  & ST-Lite & \textbf{17.1} & \textbf{62.9} & \textbf{79.6} & \textbf{86.0} & \textbf{87.0} & \textbf{88.8} & \textbf{88.4} & \textbf{89.3} & \textbf{88.9} \\
\cmidrule{2-12}
 & \multirow{6}{*}{OpenCUA-7B} 
 & SnapKV & \textbf{42.1} & 71.5 & \textbf{73.8} & 70.8 & 64.0 & 66.3 & 60.2 & 74.3 & \textbf{90.7} \\
 &  & PyramidKV & 16.9 & 34.9 & 41.7 & 41.0 & 41.4 & 63.3 & 57.9 & 74.4 & \textbf{90.7} \\
 &  & VL-Cache & 0.9 & 52.3 & 59.0 & 67.5 & 57.8 & 61.1 & 55.8 & 74.0 & \textbf{90.7} \\
 &  & GUI-KV & 42.0 & 71.6 & 72.8 & 71.0 & 65.2 & 67.9 & 65.8 & 75.9 & \textbf{90.7} \\
 &  & ST-Lite (CSS only) & 41.9 & 71.7 & 69.9 & 71.1 & 65.5 & 70.4 & 67.2 & 77.0 & \textbf{90.7} \\
 &  & ST-Lite & 41.9 & \textbf{71.7} & 69.9 & \textbf{71.1} & \textbf{65.5} & \textbf{70.4} & \textbf{67.2} & \textbf{77.0} & \textbf{90.7} \\
\midrule
\multirow{12}{*}{AndroidControl} & \multirow{6}{*}{UI-TARS-1.5-7B} 
 & SnapKV & \textbf{18.9} & \textbf{26.1} & 28.9 & 39.1 & 41.7 & 44.4 & 47.7 & 46.6 & \textbf{49.6} \\
 &  & PyramidKV & 11.6 & 20.8 & 24.9 & 28.9 & 34.4 & 37.9 & 44.6 & 46.3 & \textbf{49.6} \\
 &  & VL-Cache & 1.9 & 6.6 & 10.6 & 34.9 & 39.9 & 42.9 & 46.3 & 45.8 & \textbf{49.6} \\
 &  & GUI-KV & 17.9 & 26.0 & 30.1 & 41.2 & 44.7 & 46.0 & 48.2 & 46.9 & \textbf{49.6} \\
 &  & ST-Lite (CSS only) & 15.9 & 26.0 & 30.1 & 39.7 & 42.0 & 43.6 & 48.7 & 47.1 & \textbf{49.6} \\
 &  & ST-Lite & 13.9 & 25.9 & \textbf{34.8} & \textbf{44.2} & \textbf{46.9} & \textbf{48.9} & \textbf{50.3} & \textbf{48.1} & \textbf{49.6} \\
\cmidrule{2-12}
 & \multirow{6}{*}{OpenCUA-7B} 
 & SnapKV & \textbf{3.6} & 15.3 & 18.6 & 21.6 & 27.9 & 25.9 & 30.9 & 26.1 & \textbf{40.7} \\
 &  & PyramidKV & 1.9 & \textbf{15.9} & 15.9 & 22.9 & 27.3 & 27.9 & 23.1 & 26.9 & \textbf{40.7} \\
 &  & VL-Cache & 0.0 & 2.9 & 2.5 & 11.9 & 15.7 & 19.5 & 27.9 & 6.5 & \textbf{40.7} \\
 &  & GUI-KV & 3.0 & 15.8 & 19.6 & 23.8 & 30.2 & 27.3 & 32.3 & 31.7 & \textbf{40.7} \\
 &  & ST-Lite (CSS only) & 2.3 & 14.2 & 18.8 & 23.3 & 29.8 & 28.7 & 34.4 & 29.1 & \textbf{40.7} \\
 &  & ST-Lite & 0.7 & \textbf{15.9} & \textbf{19.9} & \textbf{26.8} & \textbf{36.9} & \textbf{32.9} & \textbf{37.9} & \textbf{34.1} & \textbf{40.7} \\
\midrule
\multirow{12}{*}{\shortstack{Multimodal-\\Mind2Web}} & \multirow{6}{*}{UI-TARS-1.5-7B} 
 & SnapKV & 2.1 & 4.7 & 7.1 & 11.6 & 16.9 & 18.6 & 22.7 & 24.2 & \textbf{25.9} \\
 &  & PyramidKV & 1.4 & 3.9 & 4.7 & 5.9 & 9.1 & 11.9 & 21.9 & 24.2 & \textbf{25.9} \\
 &  & VL-Cache & 0.0 & 0.1 & 0.0 & 0.6 & 3.9 & 7.4 & 19.9 & 23.9 & \textbf{25.9} \\
 &  & GUI-KV & 2.0 & 5.2 & 8.3 & 13.6 & 18.7 & 20.3 & 24.2 & 25.6 & \textbf{25.9} \\
 &  & ST-Lite (CSS only) & 2.4 & 7.3 & 10.9 & 16.4 & 19.4 & 20.7 & 21.9 & 25.9 & \textbf{25.9} \\
 &  & ST-Lite & 1.9 & 7.1 & \textbf{12.9} & \textbf{21.4} & \textbf{25.7} & \textbf{27.2} & \textbf{30.2} & \textbf{31.1} & \textbf{25.9} \\
\cmidrule{2-12}
 & \multirow{6}{*}{OpenCUA-7B} 
 & SnapKV & \textbf{0.9} & \textbf{2.4} & 5.6 & 9.4 & 12.7 & 14.3 & 16.7 & 17.3 & \textbf{33.6} \\
 &  & PyramidKV & 0.0 & 0.9 & 1.8 & 5.6 & 6.7 & 10.9 & 12.7 & 14.1 & \textbf{33.6} \\
 &  & VL-Cache & 0.0 & 0.0 & 0.0 & 0.1 & 0.1 & 0.0 & 13.4 & 14.4 & \textbf{33.6} \\
 &  & GUI-KV & 0.4 & 2.3 & 6.0 & 10.2 & 13.7 & 14.8 & 17.0 & 17.7 & \textbf{33.6} \\
 &  & ST-Lite (CSS only) & 0.8 & 0.0 & 5.7 & 10.1 & 11.7 & 14.4 & 17.1 & 18.7 & \textbf{33.6} \\
 &  & ST-Lite & 0.2 & 1.8 & \textbf{7.8} & \textbf{13.4} & \textbf{15.2} & \textbf{16.9} & \textbf{18.3} & \textbf{19.4} & \textbf{33.6} \\
\midrule
\multirow{12}{*}{OSWorld-Verified} & \multirow{6}{*}{UI-TARS-1.5-7B} 
 & SnapKV & \textbf{2.2} & \textbf{3.2} & 6.9 & 16.6 & 16.7 & 18.6 & 19.2 & 23.2 & \textbf{26.3} \\
 &  & PyramidKV & 0.2 & 1.9 & 1.1 & 2.9 & 9.9 & 15.6 & 16.9 & 20.3 & \textbf{26.3} \\
 &  & VL-Cache & 0.0 & 1.9 & \textbf{10.6} & 7.6 & 8.9 & 9.7 & 10.1 & 19.9 & \textbf{26.3} \\
 &  & GUI-KV & 1.9 & 2.8 & 8.3 & 17.1 & 18.0 & 20.7 & 25.1 & 24.2 & \textbf{26.3} \\
 &  & ST-Lite (CSS only) & 1.9 & 3.0 & 7.1 & 17.3 & 18.2 & 19.6 & 23.1 & 25.3 & \textbf{26.3} \\
 &  & ST-Lite & 0.6 & 1.1 & 8.6 & \textbf{18.9} & \textbf{19.7} & \textbf{25.9} & \textbf{27.0} & \textbf{28.1} & \textbf{26.3} \\
\cmidrule{2-12}
 & \multirow{6}{*}{OpenCUA-7B} 
 & SnapKV & 0.7 & 0.9 & 0.5 & 6.9 & 10.1 & 14.3 & 16.2 & 16.8 & \textbf{23.8} \\
 &  & PyramidKV & 0.3 & 0.9 & 1.9 & 6.7 & 8.6 & 11.2 & 13.1 & 14.9 & \textbf{23.8} \\
 &  & VL-Cache & 0.0 & 0.1 & 0.6 & 3.1 & 6.9 & 9.9 & 9.7 & 9.1 & \textbf{23.8} \\
 &  & GUI-KV & 0.3 & 0.9 & 1.6 & 7.5 & 11.2 & 15.1 & 17.2 & 17.0 & \textbf{23.8} \\
 &  & ST-Lite (CSS only) & 1.1 & 1.4 & 3.3 & 6.9 & 10.9 & 14.7 & 17.4 & 17.9 & \textbf{23.8} \\
 &  & ST-Lite & 0.1 & 0.9 & \textbf{4.1} & \textbf{8.9} & \textbf{13.9} & \textbf{16.9} & \textbf{19.7} & 17.5 & \textbf{23.8} \\
\bottomrule
\end{tabular}
}
\captionof{table}{Performance evaluation on additional benchmarks (ScreenSpotV2, AndroidControl, Multimodal-Mind2Web, and OSWorld-Verified).}
\label{tab:additional_benchmarks_detail}
\end{center}

\end{document}